\documentclass[journal]{IEEEtran}
\usepackage{microtype}
\usepackage{graphicx}
\graphicspath{{../}{./}}
\usepackage{booktabs} 
\usepackage{makecell}
\usepackage{dsfont}
\usepackage{hhline}
\usepackage[utf8]{inputenc} 
\usepackage[T1]{fontenc}    
\usepackage{hyperref}       
\hypersetup{hidelinks}
\usepackage{url}            
\usepackage{amsfonts}       
\usepackage{nicefrac}       
\usepackage[table]{xcolor}  
\usepackage{cite}
\usepackage{titletoc}
\usepackage{tikz}
\usetikzlibrary{arrows.meta,positioning,calc}
\usepackage[caption=false,font=footnotesize]{subfig}
\usepackage{caption}
\DeclareCaptionLabelFormat{lc}{\MakeLowercase{#1}~#2}
\usepackage[linesnumbered,ruled,vlined]{algorithm2e}

\usepackage{multirow}    
\usepackage{array} 
\usepackage{tabularx}
\usepackage{enumitem}
\usepackage{needspace}
\usepackage{amsmath}
\usepackage{amssymb}
\usepackage{mathtools}
\usepackage{amsthm}
\allowdisplaybreaks

\usepackage[capitalize,noabbrev]{cleveref}

\theoremstyle{plain}
\newtheorem{theorem}{Theorem}[section]
\newtheorem{proposition}[theorem]{Proposition}

\newtheorem{lemma}[theorem]{Lemma}
\newtheorem{corollary}[theorem]{Corollary}
\theoremstyle{definition}

\newtheorem{assumption}[theorem]{Assumption}

\theoremstyle{remark}
\newtheorem{remark}[theorem]{Remark}
\crefname{assumption}{Assumption}{Assumptions}
\Crefname{assumption}{Assumption}{Assumptions}
\newcommand{\proofref}[1]{\textnormal{proof in \mbox{Appendix~\ref{#1}}}}
\newcommand{\refreshprob}{\widetilde{s}}
\newcommand{\refreshmask}{\widetilde{I}}

\usepackage[textsize=tiny]{todonotes}
\newcommand{\com}[1]{\tiny$\pm$#1}

\title{Correcting Before Adapting: Gradient Correction for Federated Adaptive Optimization}

\author{
Evan Chen,~\IEEEmembership{Student Member,~IEEE},
Shiqiang Wang,~\IEEEmembership{Fellow,~IEEE},
Dong-Jun Han,~\IEEEmembership{Member,~IEEE},
\\
Shahryar Zehtabi,~\IEEEmembership{Student Member,~IEEE},
Christopher G. Brinton,~\IEEEmembership{Senior Member,~IEEE}
\thanks{Evan Chen, Shahryar Zehtabi, and Christopher G. Brinton are with the
Elmore Family School of Electrical and Computer Engineering, Purdue University,
West Lafayette, IN 47907, USA
(e-mail: \{chen4388,szehtabi,cgb\}@purdue.edu).}
\thanks{Shiqiang Wang is with the Department of Computer Science,
University of Exeter, EX4 4RN, UK
(e-mail: s.wang9@exeter.ac.uk).}
\thanks{Dong-Jun Han is with the Department of Computer Science and Engineering,
Yonsei University, Seoul 03722, South Korea
(e-mail: djh@yonsei.ac.kr).}
}

\begin{document}

\maketitle
\begin{abstract}
Client-drift compensation methods mitigate the impact of data heterogeneity in federated learning (FL). While drift correction is well studied for stochastic gradient descent (SGD)-based FL, its interaction with adaptive client optimizers remains less understood. Adaptive optimizers such as Adam build each update from running averages of past gradients, so correcting only the final update direction leaves those averages, and the drift they contain, unchanged. We propose Federated Adaptive Gradient Correction ({\tt FA-GC}), which corrects each gradient before it enters these running averages, and compare it with Federated Adaptive Naive Tracking ({\tt FA-NT}), which adds the correction to the final direction. We show that, in an idealized setting without gradient noise, the direction FA-GC follows vanishes only at stationary points of the global objective, whereas the direction FA-NT follows can vanish elsewhere when the optimizer is nonlinear. We also prove convergence of FA-GC for Adam-family methods, gradient normalization, and Muon, which orthogonalizes matrix updates. Unlike our bound for FA-NT, these bounds contain no term for client heterogeneity. On seven image and language tasks, FA-GC with Adam needs the fewest rounds on average to reach the target accuracy, and on the three image tasks it also uses less total communication than uncorrected local Adam and FA-NT. On the QNLI question-entailment task with SOAP, which applies Adam in a rotated basis, and Muon under strong heterogeneity, FA-GC gives the highest accuracy, while FA-NT falls below the uncorrected optimizers.
\end{abstract}

\section{Introduction}

\label{sec:intro}
Federated learning (FL) allows clients to train a shared model without pooling their raw data~\cite{mcmahan2017communication}. However, the data available at each client often differ from the population represented by the global objective~\cite{li2020federated,kairouz2021advances}. Taking several local steps saves communication, but it also gives each client more opportunity to move toward its own objective. This client drift becomes especially consequential when only a fraction of clients can participate in each round, as is common under limited communication and computation at the network edge~\cite{bonawitz2019towards,lim2020federated}. For stochastic gradient descent (SGD), methods such as SCAFFOLD~\cite{karimireddy2020scaffold} and ProxSkip~\cite{mishchenko2022proxskip} address the discrepancy with control variates. Each client stores an estimate of its own gradient, the server stores the average of these estimates, and their difference is added to every local gradient. We call these stored estimates trackers, and replacing a client's tracker with a newly computed estimate a refresh. This allows clients to train locally while accounting for information from the rest of the system.

While SGD remains the default choice for local updates in FL, adaptive optimizers have demonstrated performance benefits in model training. Adam, for example, retains running averages, called moments, of gradients and their coordinatewise squares, then uses those averages to scale its updates~\cite{kingma2014adam}. AdaGrad also scales coordinates using gradient history, SOAP applies Adam in a rotated basis, and Muon constructs matrix directions through orthogonalization~\cite{duchi2011adaptive,vyas2024soap,jordan2024muon}. Their use in sequence and language-model training~\cite{graves2013generating,zhang2024towards} motivates bringing adaptive optimization into FL~\cite{cheng2023momentum,reddi2020adaptive,wang2022communication}. In these settings, clients learn from different data and may join only some training rounds. When clients keep their optimizer history across server updates, as in this work, drift compensation must also account for the gradients stored in that history.

The success of correction methods for SGD suggests adding a correction to the adaptive update as well. For an adaptive optimizer, the correction can be applied either to the gradients that enter the moments~\cite{yan2025problem} or to the final direction, and the two choices can produce different directions even at the same shared model.

Motivated by this, we investigate two questions:
\begin{enumerate}[leftmargin=*,itemsep=2pt,topsep=2pt]
\item \textit{Can client-drift compensation provide performance benefits for adaptive federated training, as it does for SGD-based methods?}
\item \textit{How does applying correction before or after moment estimation change the update dynamics and convergence guarantees?}
\end{enumerate}

\textbf{Key Challenges.} Adaptive optimizers turn their moments into a direction through a nonlinear map: Adam divides the first moment by the square root of the second, while gradient normalization and Muon rescale the first moment. For SGD, applying gradient correction as done in SCAFFOLD directly corrects the update direction, since the SGD update is linear in the gradient. However, this property does not carry over to adaptive optimizers such as Adam, where the average of the clients’ adaptive directions generally differs from the direction obtained by applying the same optimizer to the global gradient. Adding the correction to the final adaptive direction, as one would for SGD, leaves both moments built from each client's own gradients. In an idealized noiseless setting, this naive placement gives the same averaged direction as no correction at all, and that direction can vanish away from a stationary point (Section~\ref{sec:41}). On the QNLI task with SOAP and Muon under strong heterogeneity, it reaches lower accuracy than the uncorrected optimizers (Section~\ref{sec:exp}). Correcting each gradient before it enters the moments avoids this, but makes the analysis harder. The moments mix gradients from several earlier rounds, the trackers were computed at older models, and the same noisy gradients enter both the moments and the trackers, so their errors are correlated.

\textbf{Our Approach and Contributions.} We develop Federated Adaptive Gradient Correction ({\tt FA-GC}), which corrects each stochastic gradient before it enters the client moments. For comparison, we also study the naive alternative that applies the correction after the optimizer has computed its direction. We call this method Federated Adaptive Naive Tracking ({\tt FA-NT}), and its trackers store average adaptive directions. Both methods use the same rules for client participation and tracker refresh, allowing us to examine how the placement of the correction affects training. Refreshing a tracker costs one extra upload, so to reduce communication we refresh the trackers of only a subset of participating clients, a procedure we call selective tracking. Our contributions are as follows:

\begin{itemize}[leftmargin=*,itemsep=1pt,topsep=2pt]
\item We characterize the effect of correction placement when gradients are noiseless and the optimizer states have reached a fixed point. FA-GC recovers the direction of the optimizer applied to the global gradient. FA-NT, like uncorrected local training, averages the directions computed from individual clients' gradients, and this average can vanish away from a stationary point for nonlinear optimizers (Lemma~\ref{lem:fixed-point-bias}).
\item We develop a shared convergence analysis through bounds on the size of each update direction and its alignment with the global gradient (Lemmas~\ref{lem:C1-main} and~\ref{lem:C23-main}). It accounts for client moments that persist through server updates, partial client participation, and reuse of training gradients in the trackers across Adam, AMSGrad (Adam with a running maximum in the denominator), a SOAP variant that fixes its rotation within each round (together, the Adam family), gradient normalization, and Muon.
\item With the first-moment decay $\beta_1$ and the step sizes chosen from the number of rounds $T$, FA-GC attains an $O(\sqrt{\log(T+2)/T})$ bound on the average squared gradient norm for Adam-family methods, with the second-moment decay $\beta_2$ held fixed. We also analyze the different update rules used by gradient normalization and Muon, which attain $O(T^{-1/4})$ in their respective gradient norms without assuming bounded gradients. Since the Adam-family bound is on the squared norm, both rates mean that the gradient norm decays like $T^{-1/4}$ up to a logarithmic factor. Unlike our bound for FA-NT, these bounds contain no separate client-heterogeneity term.
\item We evaluate FA-GC primarily with Adam across seven image and language tasks. On average, FA-GC needs the fewest rounds to reach the target accuracy on all seven tasks, and it reduces image-task communication relative to local Adam without correction (LocalAdam) and FA-NT. In addition, on the QNLI question-entailment task with practical SOAP and Muon, FA-GC gives the highest accuracy under strong heterogeneity, with a clear margin for Muon and a smaller one for SOAP.
\end{itemize}

\section{Related Works}

\textbf{Drift correction and gradient tracking.}
Gradient tracking estimates global gradient information during distributed local training~\cite{di2016next,nedic2017achieving,tian2018asy,koloskova2021improved,takezawa2022momentum,wang2024momentum}. In client-server FL, SCAFFOLD and ProxSkip use control variates to compensate local gradients~\cite{karimireddy2020scaffold,mishchenko2022proxskip}. Other approaches constrain local movement (FedProx), normalize heterogeneous local work (FedNova), track parameter drift (FedDC), or combine correction with infrequent communication (FedLin)~\cite{li2020fedprox,wang2020fednova,gao2022feddc,mitra2021fedlin}. Local-step tracking and GTAdam extend related ideas to decentralized and distributed online settings~\cite{liu2023decentralized,carnevale2022gtadam}.

MIME transfers centralized optimizers to FL using server statistics and local control variates~\cite{karimireddy2020mime}. FedDyn and FedRed use dynamic or doubly regularized local objectives~\cite{acar2021federated,jiang2024federated}. Our question concerns the information entering a private, persistent client optimizer: whether its state should accumulate corrected gradients or whether correction should act only on the resulting direction.

\textbf{Adaptive client optimizers.}
FedAdam and FedAMS adapt the server update~\cite{reddi2020adaptive,wang2022communication}, while Local AdaAlter and FedLADA use adaptation on clients~\cite{xie2019local,sun2023efficient}. Local adaptation can also converge to points that are not stationary for the global objective~\cite{wang2021local}. Closest to our setting, PAdaMFed applies a control-variate correction before normalizing a direction built from a stochastic gradient and server-aggregated momentum~\cite{yan2025problem}. Remark~\ref{rem:padamfed} explains how its momentum updates differ from FA-GC.

Our optimizer bounds cover Adam, AMSGrad, and a SOAP variant whose rotated basis is fixed within each round~\cite{kingma2014adam,reddi2018convergence,vyas2024soap}. In this work we call this variant round-frozen SOAP. FedMuon uses persistent bias-corrected momentum in a linear minimization oracle (LMO) control construction~\cite{takezawa2026fedmuon}. Inexact-Muon analysis also treats numerical orthogonalization error as an additive oracle perturbation~\cite{shulgin2025beyond}. We instead directly analyze a fixed number of Newton--Schulz iterations, the matrix iteration that Muon uses to approximate orthogonalization, with a constant that depends on the matrix rank. The practical SOAP and Muon variants used in the experiments are described in Section~\ref{sec:exp}.

\textbf{Federated training systems.}
Oort and FedScale address participant selection and system heterogeneity~\cite{lai2021oort,lai2022fedscale}. Foundation-model tuning adds substantial optimizer and communication costs~\cite{woisetschlaeger2024survey}, even with parameter-efficient updates~\cite{malaviya2023reducing}. FederatedScope-LLM and FedLLM-Bench provide fine-tuning infrastructure and benchmarks, while FedPepTAO uses adaptive prompt tuning under heterogeneous data~\cite{kuang2024federatedscope,ye2024fedllmbench,che2023fedpeptao}. These settings motivate retaining useful client optimizer history across intermittent participation.

Figure~\ref{fig:related-taxonomy} locates these correction mechanisms in the update sequence.

\begin{figure*}[t]
\centering
\begingroup
\definecolor{gcBlue}{RGB}{0,114,178}
\definecolor{ntOrange}{RGB}{213,94,0}
\begin{tikzpicture}[x=1cm,y=1cm]
\tikzset{
  stage/.style={draw=gcBlue!55,rounded corners=2.5pt,fill=gcBlue!5,
    minimum width=3.00cm,minimum height=0.58cm,inner sep=2pt},
  prior/.style={draw=black!50,rounded corners=2pt,fill=white,
    minimum width=3.10cm,minimum height=0.54cm,text width=2.96cm,
    font=\scriptsize,align=center,inner sep=1.5pt},
  priorhalf/.style={draw=black!50,rounded corners=2pt,fill=white,
    minimum width=1.66cm,minimum height=0.54cm,text width=1.56cm,
    font=\scriptsize,align=center,inner sep=1pt},
  framework/.style={rounded corners=4pt,fill=white,
    minimum width=8.52cm,minimum height=1.30cm,inner sep=0pt},
  family/.style={draw=black!50,rounded corners=2pt,fill=white,
    minimum width=5.30cm,minimum height=1.10cm,text width=5.02cm,
    font=\scriptsize,align=center,inner sep=1.5pt},
  pipe/.style={draw=black!60,rounded corners=1.5pt,fill=white,
    minimum width=1.02cm,minimum height=0.40cm,font=\footnotesize,
    align=center,inner sep=1pt},
  statepipe/.style={minimum width=1.32cm},
  stagelabel/.style={font=\footnotesize\bfseries,align=left,anchor=west},
  flow/.style={-{Latex[length=3.2pt,width=2.2pt]},line width=0.75pt,black!70},
  connector/.style={line width=0.65pt,black!55}
}

\fill[black!1.5,rounded corners=4pt] (0.10,3.395) rectangle (17.60,5.250);
\draw[black!25,rounded corners=4pt] (0.10,3.395) rectangle (17.60,5.250);
\fill[gcBlue!9,rounded corners=4pt] (0.10,4.956) rectangle (17.60,5.250);
\node[font=\footnotesize\bfseries] at (8.85,5.103)
  {Related methods: where the correction enters};

\node[stage] (objective) at (1.75,4.501) {};
\draw[line width=0.65pt] (0.70,4.501) circle (0.23);
\draw[line width=0.65pt] (0.70,4.501) circle (0.12);
\fill (0.70,4.501) circle (0.035);
\node[stagelabel] at (1.13,4.501) {Objective};

\node[stage] (gradient) at (5.25,4.501) {};
\draw[flow] (4.20,4.354)--(4.70,4.627);
\draw[line width=0.6pt] (4.20,4.627)--(4.70,4.354);
\node[stagelabel] at (4.82,4.501) {Gradient\\input};

\node[stage] (state) at (8.75,4.501) {};
\node[draw,rounded corners=1pt,font=\scriptsize,minimum width=0.45cm,minimum height=0.18cm,inner sep=0pt] at (7.70,4.592) {$m$};
\node[draw,rounded corners=1pt,font=\scriptsize,minimum width=0.45cm,minimum height=0.18cm,inner sep=0pt] at (7.81,4.417) {$v$};
\node[stagelabel] at (8.18,4.501) {Persistent\\state};

\node[stage] (map) at (12.25,4.501) {};
\node[draw,rounded corners=1pt,font=\footnotesize,minimum width=0.60cm,minimum height=0.32cm,inner sep=0pt] at (11.20,4.501) {$\mathcal P$};
\node[stagelabel] at (11.68,4.501) {Adaptive\\map/LMO};

\node[stage] (output) at (15.75,4.501) {};
\draw[flow] (14.70,4.501)--(15.16,4.501);
\fill (15.24,4.501) circle (0.07);
\node[stagelabel] at (15.45,4.501) {Final\\direction};

\draw[flow] (objective.east)--(gradient.west);
\draw[flow] (gradient.east)--(state.west);
\draw[flow] (state.east)--(map.west);
\draw[flow] (map.east)--(output.west);

\node[prior] (feddyn) at (1.75,3.710)
  {\textbf{FedDyn / FedRed}\\objective regularization};
\node[priorhalf] (scaffold) at (4.40,3.710)
  {\textbf{SCAFFOLD}\\client controls};
\node[priorhalf] (mime) at (6.10,3.710)
  {\textbf{MIME}\\server statistics};
\node[prior] (padamfed) at (8.75,3.710)
  {\textbf{PAdaMFed}\\shared $m$, no learned $v$};
\node[prior] (fedmuon) at (12.25,3.710)
  {\textbf{FedMuon}\\control before matrix LMO};

\draw[connector] (objective.south)--(feddyn.north);
\draw[connector] (gradient.south)--(5.25,4.081);
\draw[connector] (5.25,4.081)-|(scaffold.north);
\draw[connector] (5.25,4.081)-|(mime.north);
\draw[connector] (state.south)--(padamfed.north);
\draw[connector] (map.south)--(fedmuon.north);

\fill[black!1.5,rounded corners=4pt] (0.10,-0.470) rectangle (17.60,3.290);
\draw[black!25,rounded corners=4pt] (0.10,-0.470) rectangle (17.60,3.290);
\fill[gcBlue!9,rounded corners=4pt] (0.10,3.010) rectangle (17.60,3.290);
\node[font=\footnotesize\bfseries] at (8.85,3.150)
  {FA-GC and FA-NT for a general optimizer};

\node[framework,draw=gcBlue,line width=1pt,fill=gcBlue!5] (fagcpanel) at (4.40,2.226) {};
\node[framework,draw=ntOrange,dashed,line width=0.85pt,fill=ntOrange!4] (fantpanel) at (13.30,2.226) {};

\node[font=\footnotesize\bfseries] at (4.40,2.716) {FA-GC: pre-estimation correction};
\node[font=\footnotesize\bfseries] at (13.30,2.716) {FA-NT: post-estimation correction};

\node[pipe] (gcg) at (1.28,2.275) {$g_i$};
\node[pipe,draw=gcBlue,fill=gcBlue!10] (gcc) at (2.83,2.275) {$+c_i$};
\node[pipe,statepipe] (gcstate) at (4.40,2.275) {$(m_i,v_i)$};
\node[pipe] (gcmap) at (5.97,2.275) {$\mathcal P_i$};
\node[pipe] (gcd) at (7.52,2.275) {$d_i$};
\draw[flow] (gcg.east)--(gcc.west);
\draw[flow] (gcc.east)--(gcstate.west);
\draw[flow] (gcstate.east)--(gcmap.west);
\draw[flow] (gcmap.east)--(gcd.west);

\node[pipe] (ntg) at (10.18,2.275) {$g_i$};
\node[pipe,statepipe] (ntstate) at (11.73,2.275) {$(m_i,v_i)$};
\node[pipe] (ntmap) at (13.30,2.275) {$\mathcal P_i$};
\node[pipe,draw=ntOrange,fill=ntOrange!10] (ntc) at (14.87,2.275) {$+c_i$};
\node[pipe] (ntd) at (16.42,2.275) {$d_i+c_i$};
\draw[flow] (ntg.east)--(ntstate.west);
\draw[flow] (ntstate.east)--(ntmap.west);
\draw[flow] (ntmap.east)--(ntc.west);
\draw[flow] (ntc.east)--(ntd.west);

\fill[black!4,rounded corners=2pt] (0.38,1.610) rectangle (8.42,1.834);
\fill[black!4,rounded corners=2pt] (9.28,1.610) rectangle (17.32,1.834);
\node[font=\scriptsize,align=center,text width=7.70cm] at (4.40,1.722)
  {Moments built from corrected gradients; trackers store gradients.};
\node[font=\scriptsize,align=center,text width=7.70cm] at (13.30,1.722)
  {Moments built from raw gradients; trackers store directions.};

\node[draw=black!55,rounded corners=2pt,fill=white,font=\scriptsize\bfseries,
  minimum width=6.80cm,minimum height=0.30cm,inner sep=1pt] (sharedbar) at (8.85,1.295)
  {Both methods are analyzed for};
\draw[connector] (fagcpanel.south)--(4.40,1.470)--(8.85,1.470)--(sharedbar.north);
\draw[connector] (fantpanel.south)--(13.30,1.470)--(8.85,1.470);

\node[family,fill=gcBlue!3] (diagonalfamily) at (3.15,0.240)
  {\textbf{Diagonal second moment}\\Adam, AMSGrad, round-frozen SOAP};
\node[family,fill=gcBlue!3] (vectorfamily) at (8.85,0.240)
  {\textbf{Vector normalization}\\Gradient normalization};
\node[family,fill=gcBlue!3] (matrixfamily) at (14.55,0.240)
  {\textbf{Matrix LMO}\\Exact and finite-step Muon};

\coordinate (familybranch) at ($(vectorfamily.north)+(0,0.20)$);
\draw[connector] (sharedbar.south)--(familybranch);
\draw[connector] (diagonalfamily.north)--(diagonalfamily.north|-familybranch)
  --(matrixfamily.north|-familybranch)--(matrixfamily.north);
\draw[connector] (familybranch)--(vectorfamily.north);
\end{tikzpicture}
\endgroup
\caption{Where the correction enters the update. Top: related methods grouped by correction placement. Bottom: FA-GC corrects the gradient before moment estimation, while FA-NT corrects the resulting direction. Here $g_i$ is the raw client gradient, $c_i=y-y_i$ the difference between server and client trackers, $(m_i,v_i)$ the first and second moments, and $d_i=\mathcal P_i(m_i;v_i)$ the optimizer direction. Both placements are analyzed for the three optimizer families shown; finite-step Muon approximates the polar map with a fixed number of Newton--Schulz iterations (Section~\ref{ssec:gc_algorithm}).}
\label{fig:related-taxonomy}
\end{figure*}

\section{Background and Preliminaries}
\label{sec:background}
\begin{figure*}[t]
    \centering
    \includegraphics[width=1.0\linewidth]{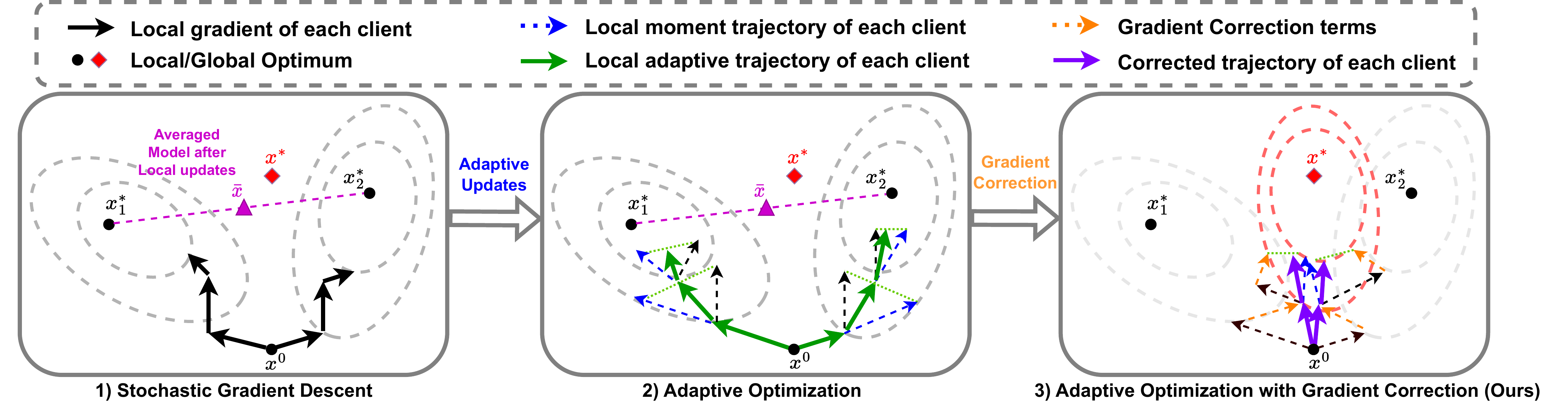}
    \caption{Client drift under local SGD, adaptive optimization, and adaptive optimization with gradient correction. Gray and red dashed contours represent client and global objectives. Black dots $x_i^*$ mark client minimizers, and the red diamond $x^*$ marks the global minimizer. Trajectories start from $x^0$, and the magenta point $\bar x$ is the average local model. Gradient correction changes the directions used by the adaptive optimizer.}
    \label{fig:GT_visualize}
\end{figure*}
\textbf{System model.}
The $n$ clients train a shared model $x$ in the parameter space $\mathcal X$ to minimize their average loss
\begin{equation}
\min_{x\in\mathcal X}f(x)=\frac1n\sum_{i=1}^n f_i(x),\qquad
f_i(x)=\mathbb E_{\xi_i\sim\mathcal D_i}\ell_i(x;\xi_i),
\end{equation}
where $\mathcal D_i$ is client $i$'s data distribution and $\ell_i$ is its sample loss. The model may be a matrix in $\mathbb R^{d_1\times d_2}$, with $d=d_1d_2$ scalar parameters. Vector methods use the flattened model, while Muon acts on parameter matrices.

In round $t=0,\ldots,T-1$, the server uniformly samples a training set $\mathcal S^t$ of $S$ clients. Each selected client starts from the server model, $x_i^{t,1}=x^t$, and takes $K$ local steps with learning rate $\eta_\ell$. At step $k$, it draws a sample $\xi_i^{t,k}$ and computes $g_i^{t,k}=\nabla\ell_i(x_i^{t,k};\xi_i^{t,k})$. It then returns its model change to the server, which adds the average change scaled by $\eta_g$. Thus $x_i^{t,K+1}$ is the model after exactly $K$ updates. Although the clients start from the same model, their local objectives may lead them in different directions, as Figure~\ref{fig:GT_visualize} illustrates.

\textbf{Adaptive client updates.}
Adam stores moving averages of gradients and their coordinatewise squares~\cite{kingma2014adam}. These first and second moments determine the direction and coordinatewise scaling of each update; the first moment is also called the momentum. For a gradient input $\widehat g_i^{t,k}$, the moments and update without initialization bias correction are
\begin{align}
m_i^{t,k}&=\beta_1m_i^{t,k-1}+(1-\beta_1)\widehat g_i^{t,k},\notag\\
v_i^{t,k}&=\beta_2v_i^{t,k-1}+(1-\beta_2)(\widehat g_i^{t,k})^{\odot2},\notag\\
d_i^{t,k}&=\frac{m_i^{t,k}}{\sqrt{v_i^{t,k}}+\epsilon},\qquad
x_i^{t,k+1}=x_i^{t,k}-\eta_\ell d_i^{t,k}.
\label{eq:Adam_update_formulation}
\end{align}
Here $\beta_1,\beta_2\in[0,1)$ are decay coefficients, $\epsilon>0$ stabilizes the denominator, and powers, division, and $\odot$ act coordinatewise. Our analysis omits moment bias correction, which divides the two moving averages by $1-\beta_1^k$ and $1-\beta_2^k$ after $k$ updates to offset their zero initialization; the experiments include it (Section~\ref{sec:exp}). This is separate from the correction for client drift studied here. Ordinary local Adam uses $\widehat g_i^{t,k}=g_i^{t,k}$. Using a corrected gradient as $\widehat g_i^{t,k}$ would apply the correction before either moving average is updated.

We denote the operation that converts moment estimates into an update direction by $\mathcal P_i^{t,k}(m;v)$. For Adam it is the fraction in~\eqref{eq:Adam_update_formulation}. Gradient normalization uses $N(m)=m/\|m\|_2$, with the Euclidean norm, while exact Muon uses the polar direction $P_{\rm M}(M)=UV^\top$ for the compact singular value decomposition $M=U\Sigma V^\top$. These two methods use only the first moment and omit $v$. At zero input we set $N(0)=0$ and $P_{\rm M}(0)=0$. Equivalently, a linear minimization oracle over matrices with largest singular value at most one returns $-P_{\rm M}(M)$, which explains the LMO description of Muon.

\textbf{Drift compensation for SGD.} Before considering adaptive updates, we recall how {\tt SCAFFOLD} corrects client drift in SGD. Each client stores a control variate $y_i^t$ that estimates its own gradient, and the server keeps $y^t$, which estimates the global gradient. These are the trackers of Section~\ref{sec:intro}. Their difference is the correction added to a local gradient. Setting $x_i^{t,1}=x^t$, the local update and tracker refresh can be written as
\begin{align}
x_i^{t,k+1} &= x_i^{t,k} - \eta_\ell ( g_i^{t,k} + y^t - y_i^t ),\notag\\
y_i^{t+1} &= y_i^t - y^t + \frac{1}{\eta_\ell K}(x^t - x_i^{t,K+1}). \label{eq:scaffold_update}
\end{align}
To see what information this refresh stores, sum the $K$ local steps to obtain
\begin{equation}
\frac{x^t-x_i^{t,K+1}}{\eta_\ell K}
=\frac1K\sum_{k=1}^K g_i^{t,k}+y^t-y_i^t,
\label{eq:scaffold-displacement}
\end{equation}
so the refresh in~\eqref{eq:scaffold_update} cancels the correction and stores the average raw stochastic gradient. Thus the client and server trackers estimate the client and global gradients in later rounds.

\section{Correcting Before Moment Estimation}

\subsection{Motivation and Challenges}\label{sec:41}
\textbf{Why is Drift Compensation Needed?}
To understand why local adaptive updates need correction, consider a shared model at which the global gradient is zero. Let $x^*$ satisfy
$\nabla f(x^*) = 0,$
where $f(x)=\frac{1}{n}\sum_{i=1}^n f_i(x)$.
We first set aside sampling noise and use population gradients, which average over each client's data distribution. Recall that a client's first moment is a moving average of its past gradients,
$ 
m_i^{t,k}
=
\beta_1 m_i^{t,k-1}
+
(1-\beta_1)\nabla f_i(x_i^{t,k}),
$
and that the client forms its local direction through
\begin{equation}
x_i^{t,k+1}=x_i^{t,k}-\eta_\ell
\mathcal{P}_i^{t,k}(m_i^{t,k};v_i^{t,k}),
\end{equation}
where $\mathcal P_i^{t,k}(m;v)$ turns the updated moment estimates into a direction. Methods without a second moment omit $v$. If the model is held at $x^*$ while the moments reach their fixed point, the first moment satisfies
\begin{equation}
m_i^*=\beta_1m_i^*+(1-\beta_1)\nabla f_i(x^*)
\quad\Longrightarrow\quad
m_i^*=\nabla f_i(x^*)
\label{eq:local-momentum-fixed-point}
\end{equation}
for every $\beta_1<1$. Although $\nabla f(x^*)=0$, different client objectives may give $\nabla f_i(x^*)\neq0$. For the optimizers studied here, a nonzero momentum gives a nonzero direction, so a client's local update can move away from $x^*$ even when the global gradient vanishes. Averaging across clients may cancel part of this motion. Lemma~\ref{lem:fixed-point-bias} below explains why this cancellation can fail for nonlinear updates.

\textbf{Why Naive Tracking Fails for Adaptive Methods.}
The naive way to extend SCAFFOLD is to apply its correction after the optimizer has computed the adaptive direction, which is the FA-NT method introduced earlier. Because the correction now acts on directions, each client stores in $y_i^t$ the average direction it took in its last refresh round. The server keeps their average $y^t$, and a client updates with
$x_i^{t,k+1}=x_i^{t,k}-\eta_\ell(d_i^{t,k}+y^t-y_i^t)$, where $d_i^{t,k}=\mathcal P_i^{t,k}(m_i^{t,k};v_i^{t,k})$. For SGD ($\beta_1=0$ and $\mathcal P_i^{t,k}(u)=u$), this is exactly SCAFFOLD, and~\eqref{eq:scaffold-displacement} shows that its refresh stores the average raw gradient. For an adaptive method, the same refresh gives
\begin{align}
\frac{x^t-x_i^{t,K+1}}{\eta_\ell K}
&=\frac1K\sum_{k=1}^K d_i^{t,k}+y^t-y_i^t,\notag\\
y_i^{t+1}
&=y_i^t-y^t+\frac{x^t-x_i^{t,K+1}}{\eta_\ell K}
=\frac1K\sum_{k=1}^K d_i^{t,k}.
\label{eq:adaptive-direction-refresh}
\end{align}
Thus an FA-NT tracker stores average adaptive directions, while the moment estimates continue to accumulate the client's own gradients. Correcting the direction therefore leaves this gradient history unchanged.

The alternative is to track raw gradients and add their correction before updating the moments. To compare the two placements, we assume a fixed server model $x$, remove gradient noise, and let all clients participate and refresh in every round, evaluating every gradient at $x$. Then, a tracker of raw gradients gives client $i$ the correction $\nabla f(x)-\nabla f_i(x)$, so each client feeds the same global gradient into its moments. The following lemma compares the resulting directions once the moments and trackers reach their fixed point.

\begin{lemma}[Where the fixed-point directions vanish, \proofref{app:proof-fixed-point}]\label{lem:fixed-point-bias}
In the setting above, let every client use the same map $\mathcal P(m;v)$ from Adam with fixed $\beta_2$, gradient normalization, or Muon; methods without a second moment omit $v$. Each of these maps returns zero exactly when its momentum input is zero. Hold $x$ fixed and let the moments and trackers reach their fixed point. Then:
\begin{enumerate}[label=(\roman*),leftmargin=*]
\item With gradient correction before moment estimation (FA-GC), every client's corrected gradient equals $\nabla f(x)$, so the direction the server would take is
$d_{\rm GC}(x)=\mathcal P(\nabla f(x);\nabla f(x)^{\odot2})$, the direction the same optimizer would take on the global objective. Hence $d_{\rm GC}(x)=0$ exactly when $\nabla f(x)=0$.
\item With correction applied to the final adaptive direction (FA-NT), the direction the server would take is the average of the directions that each client's optimizer computes from its own gradient,
$d_{\rm NT}(x)=\frac1n\sum_{i=1}^n\mathcal P(\nabla f_i(x);\nabla f_i(x)^{\odot2})$.
\item If $\mathcal P$ is linear, as in SGD with or without momentum, then $d_{\rm NT}=d_{\rm GC}$. Otherwise, the points where $d_{\rm NT}$ vanishes need not be stationary points of $f$.
\end{enumerate}
\end{lemma}

Part~(ii) has a simple reason. At the fixed point, client $i$'s tracker equals its own direction $p_i=\mathcal P(\nabla f_i(x);\nabla f_i(x)^{\odot2})$ and $y=n^{-1}\sum_jp_j$, so every client steps along $p_i+y-y_i=y$. The correction only replaces each client's direction by the average of all clients' directions, which is also the averaged direction of uncorrected local training. For a nonlinear $\mathcal P$, this average of outputs differs from the output at the average gradient, and a correction added afterwards cannot change the moments that produced them.

For a concrete example, take three clients with $f_i(x)=\frac12(x-x_i^*)^2$ and client minimizers $(x_1^*,x_2^*,x_3^*)=(0,1,5)$. The global minimizer is $x^*=2$, where the client gradients are $(2,1,-3)$. These gradients sum to zero, but in one dimension normalization keeps only their signs. Consequently, FA-NT's direction at $x^*$ is $\frac13(1+1-1)=\frac13\neq0$, and it vanishes only at the median $x=1$, where $\nabla f(1)=-1$. For Adam at its fixed point, a client with gradient $g$ moves along $g/(|g|+\epsilon)$. A small $\epsilon$ makes this close to the sign of $g$, and a large $\epsilon$ makes it close to the linear map $g/\epsilon$. Hence the point where $d_{\rm NT}$ vanishes moves from $1$ toward $2$ as $\epsilon$ grows (about $1.29$ at $\epsilon=1$), and the common choice $\epsilon=10^{-8}$ behaves like the sign. For FA-GC, both optimizer choices give $d_{\rm GC}(x)=0$ only at $x^*=2$ in this fixed-point setting. Running both algorithms on this problem shows that the iterates stop at these points: FA-GC reaches $x^*=2$, while FA-NT stops at $x=1$ for normalization and for Adam with $\epsilon=10^{-8}$ (Remark~\ref{rem:numerical-check}).

With only two clients, the gradients at $x^*$ are exact opposites. Normalization and Adam at its fixed point give opposite directions for opposite inputs, so $d_{\rm NT}(x^*)=0$. However, the average direction can also vanish elsewhere: for two quadratic clients as above, normalization makes $d_{\rm NT}$ vanish at every point strictly between the two client minimizers.

\subsection{Federated Adaptive Framework with Gradient Correction}
\label{ssec:gc_algorithm}

Lemma~\ref{lem:fixed-point-bias} motivates correcting gradients before moment estimation. In practice, clients must estimate this correction from gradients they have already computed. These estimates become out of date as clients take local steps and the server updates the shared model. For normalization and Muon, any remaining nonzero momentum still produces a direction of norm one, so an error in the estimated correction is not scaled down near a stationary point. Section~\ref{sec:convergence} therefore controls these errors through the step sizes.

\textbf{FA-GC Algorithm.} We implement this correction using the gradient trackers in Algorithm~\ref{alg:FAGC}. Each client stores a tracker $y_i^t$, the average of the raw gradients from its last refresh round, so that $y_i^t$ estimates the client's own gradient. The server keeps $y^t=n^{-1}\sum_i y_i^t$. Their difference $y^t-y_i^t$ estimates the ideal correction $\nabla f(x)-\nabla f_i(x)$ of Section~\ref{sec:41}. A client adds it to every stochastic gradient, $\widehat g_i^{t,k}=g_i^{t,k}+y^t-y_i^t$, and updates both moments with $\widehat g_i^{t,k}$ before computing its direction (Figure~\ref{fig:related-taxonomy}). When the server sends an updated model, clients keep the moments accumulated during their previous local training. Under the schedules of Theorem~\ref{thm:FA-GC}, this lets the gradient noise average out as training proceeds. If the moments were reset every round, the first local steps would follow only a few fresh gradients, and for normalization and Muon our bound would keep a noise term that does not shrink (Appendix~\ref{app:why-persist}).

To save communication, only a random subset $\widetilde{\mathcal S}^t\subseteq\mathcal S^t$ of $\widetilde S\le S$ clients refreshes its trackers in each round. This is selective tracking. Every participating client still performs the local training steps, even if its tracker is not refreshed.

\begingroup
\begin{algorithm}[t]

\caption{FA-GC: Federated Adaptive Gradient Correction}
\label{alg:FAGC}
\DontPrintSemicolon
\SetAlgoLined
\KwIn{Rounds $T$, local steps $K$, stepsizes $\eta_\ell,\eta_g$, decays $\beta_1,\beta_2$, initial model $x^0$, adaptive operator $\mathcal{P}_i^{t,k}(m;v)$}
\tcp{For normalization/Muon, omit $v$ and its update}
Initialize $m_i^0=0$, $v_i^0=0$, $y_i^0=\nabla\ell_i(x^0;\xi_i^0)$ for all $i$, and $y^0=n^{-1}\sum_i y_i^0$\;
\For{$t=0$ \KwTo $T-1$}{
    Uniformly sample $S$ training clients $\mathcal S^t$ and $\widetilde S$ refresh clients $\widetilde{\mathcal S}^t\subseteq\mathcal S^t$. Broadcast $(x^t,y^t)$\;
    \ForEach{$i\in\mathcal{S}^t$ \tcp*[f]{parallel}}{
        Set $x_i^{t,1}=x^t$, $m_i^{t,0}=m_i^t$, $v_i^{t,0}=v_i^t$\;
        \For{$k=1$ \KwTo $K$}{
            $g_i^{t,k}=\nabla\ell_i(x_i^{t,k};\xi_i^{t,k})$, \quad
            $\widehat g_i^{t,k}=g_i^{t,k}+y^t-y_i^t$\;
            $m_i^{t,k}=\beta_1m_i^{t,k-1}+(1-\beta_1)\widehat g_i^{t,k}$,\quad
            $v_i^{t,k}=\beta_2v_i^{t,k-1}+(1-\beta_2)(\widehat g_i^{t,k})^{\odot2}$\;
            $d_i^{t,k}=\mathcal{P}_i^{t,k}(m_i^{t,k};v_i^{t,k})$,\quad
            $x_i^{t,k+1}=x_i^{t,k}-\eta_\ell d_i^{t,k}$\;
        }
        Store $m_i^{t+1}=m_i^{t,K}$ and $v_i^{t+1}=v_i^{t,K}$\;
        \eIf{$i \in \widetilde{\mathcal{S}}^t$}{ $y_i^{t+1} \gets \frac{1}{K} \sum_{k=1}^K g_i^{t,k}$\; }{ $y_i^{t+1} \gets y_i^t$\; }
    }
    Training clients send $x_i^{t,K+1}-x^t$. Refreshed clients also send $y_i^{t+1}-y_i^t$\;
    Unselected clients retain $(m_i^t,v_i^t,y_i^t)$\;
    $x^{t+1}=x^t+\frac{\eta_g}{|\mathcal{S}^t|}\sum_{i\in\mathcal{S}^t}(x_i^{t,K+1}-x^t)$\;
    $y^{t+1}=y^t+\frac{1}{n}\sum_{i\in\widetilde{\mathcal{S}}^t}(y_i^{t+1}-y_i^t)$\;
}
\end{algorithm}
\endgroup

\begin{remark}[Relation to PAdaMFed]\label{rem:padamfed}
PAdaMFed~\cite{yan2025problem} also keeps client control variates across rounds and corrects gradients before normalization. It combines each corrected gradient with momentum aggregated by the server in the previous round. FA-GC maintains each client's moments across local steps and server updates, including a learned second moment for Adam-family methods.
\end{remark}

\textbf{Communication and stored state.}
The refresh $y_i^{t+1}=K^{-1}\sum_k g_i^{t,k}$ reuses gradients already computed during local training, so it requires no extra gradient evaluations. It does, however, require communication. Apart from a one-time initialization, each round sends the model and $y^t$ to every training client and receives one model update from each, plus one tracker update from each refreshed client. This gives $3S+\widetilde S$ model-sized vectors per round, compared with $4S$ when every training client refreshes. Refreshing fewer clients saves communication, but their stored gradient estimates remain unchanged for longer. Section~\ref{sec:exp} shows that refreshing a moderate subset reduces total communication while keeping most of the reduction in training rounds. Because the server updates $y^t$ with weight $1/n$, it always equals the average of all client trackers. The corrections $y^t-y_i^t$ therefore average to zero over all clients before entering the moment updates.

\textbf{The FA-NT Baseline.}
To compare the two ways of applying correction, FA-NT uses the same client sampling, aggregation, and communication as FA-GC. Its client moments also persist across server updates. However, those moments receive the raw gradients $g_i^{t,k}$, and the correction $y^t-y_i^t$ is added only after the optimizer computes the direction $\mathcal P_i^{t,k}(m_i^{t,k};v_i^{t,k})$. A refreshed tracker then stores the average direction, as in~\eqref{eq:adaptive-direction-refresh}. This also determines how we initialize the trackers. Each client draws one gradient $g_{i,\rm init}$, sets $m_i^0=(1-\beta_1)g_{i,\rm init}$ and $v_i^0=(1-\beta_2)(g_{i,\rm init})^{\odot2}$, and uses the resulting direction as its first tracker; we call this the FA-NT warm start. Algorithm~\ref{alg:FA-NT} in Appendix~\ref{appen:naive_tracking_alg_and_rate} gives the full procedure.

\textbf{Other Optimizers.}
The same correction can be used with other moment-based updates. AMSGrad keeps the running maximum $\widehat v_i^{t,k}=\max\{\widehat v_i^{t,k-1},v_i^{t,k}\}$ across local steps and server updates, and uses $\widehat v$ in the Adam denominator.

SOAP additionally computes the coordinatewise scaling in a rotated basis. For the round-frozen variant we analyze, let $Q_L^t,Q_R^t$ be orthogonal matrices chosen from past information and kept fixed during round $t$ (client indices suppressed). The momentum stays in the original coordinates, while the corrected gradient is rotated before entering the second moment. With the rotated gradient $\widehat g_{i,\rm rot}^{t,k}=(Q_L^t)^{\top}\widehat g_i^{t,k}Q_R^t$, the second moment and direction are
\begin{align}
v_i^{t,k}&=\beta_2v_i^{t,k-1}+(1-\beta_2)(\widehat g_{i,\rm rot}^{t,k})^{\odot2},\notag\\
d_i^{t,k}&=Q_L^t\frac{(Q_L^t)^{\top}m_i^{t,k}Q_R^t}
{\sqrt{v_i^{t,k}}+\epsilon}(Q_R^t)^{\top}.
\label{eq:soap-main}
\end{align}
At a round boundary, the analyzed SOAP variant keeps the entries of $v_i^t$ and interprets them in the new basis. Section~\ref{sec:exp} describes how the experimental implementation handles these basis changes.

Muon in practice approximates the polar map with a few Newton--Schulz iterations. The finite-step variant we analyze divides the momentum matrix $M$ by its Frobenius norm, $Z_0=M/\|M\|_F$, and returns $P_R(M)=Z_R$ after $R$ iterations of
\begin{equation}
\begin{aligned}
Z_{j+1}={}&\tfrac{15}{8}Z_j-\tfrac54Z_jZ_j^\top Z_j +\tfrac38Z_jZ_j^\top Z_jZ_j^\top Z_j,
\end{aligned}
\label{eq:ns-main}
\end{equation}
with $P_R(0)=0$. This polynomial, also analyzed for FedMuon~\cite{takezawa2026fedmuon}, keeps every singular value in $[0,1]$, so its spectral norm (largest singular value) satisfies $\|P_R(M)\|_{\rm sp}\le1$. The experiments use the coefficients of~\cite{jordan2024muon} instead (Appendix~\ref{appen:qnli-soap-muon}).

\section{Convergence Analysis}
\label{sec:convergence}
To show that the global objective decreases, we need two facts about each optimizer. Its update directions have bounded size, and they point roughly along the global gradient when the stored momentum is close to that gradient. Lemma~\ref{lem:C1-main} bounds the size of each direction. Lemma~\ref{lem:C23-main} gives two inequalities: the progress a direction would make if every momentum equaled the global gradient, and the progress lost because it does not. The constants $C_2$ and $C_3$ of Lemma~\ref{lem:C23-main} are listed in Table~\ref{tab:constants_for_methods}. Further lemmas in Appendix~\ref{app:core-lemmas} bound the momentum error caused by local movement, stale trackers, and reused gradient noise. All of these enter one descent argument for the persistent moments without bias correction used in Algorithm~\ref{alg:FAGC}.

\textbf{Norms.} We measure the size of an update $h$ by a norm $\|h\|_{\rm pr}$ suited to the optimizer, and the size of a gradient $g$ by the matching norm $\|g\|_*:=\max_{\|h\|_{\rm pr}\le1}\langle g,h\rangle$, so that $\langle g,h\rangle\le\|g\|_*\|h\|_{\rm pr}$. For Adam, AMSGrad, and normalization, both are the ordinary Euclidean norm, and SOAP uses its matrix version, the Frobenius norm $\|\cdot\|_F$. Only Muon uses two different matrix norms: the spectral norm $\|\cdot\|_{\rm sp}$, the largest singular value, measures its step, and the nuclear norm $\|\cdot\|_{\rm tr}$, the sum of the singular values, measures the gradient. This pairing suits Muon because its direction always has spectral norm at most one. We measure how close the iterates come to a stationary point by the average of $\|\nabla f(x^t)\|_*^q$ with $q\in\{1,2\}$. Throughout, $f_*:=\inf_x f(x)>-\infty$, $\mathcal F^t$ denotes the information available before round $t$, and $C$ denotes a finite constant that does not depend on the number of rounds $T$; it may depend on the problem, the dimension, the initialization, and the participation fractions.

\begin{assumption}[Smoothness, noise, and sampling]\label{assump:genLoss}
\emph{Smoothness.} The global objective is bounded below and
\[
f(x+h)\le f(x)+\langle\nabla f(x),h\rangle+\frac L2\|h\|_{\rm pr}^2.
\]
Each client gradient $\nabla f_i$ is $L_g$-Lipschitz, with distances between models measured in $\|\cdot\|_{\rm pr}$ and differences between gradients in the Euclidean (Frobenius) norm.

\emph{Gradient noise.} Let $\mathcal F^{t,k}$ record the sampled client sets and everything that happens before the current gradient is drawn. Given $\mathcal F^{t,k}$, the noise $\zeta_i^{t,k}=g_i^{t,k}-\nabla f_i(x_i^{t,k})$ has mean zero and
\[
\mathbb E[\|\zeta_i^{t,k}\|_2^2\mid\mathcal F^{t,k}]\le\sigma^2.
\]
Noises drawn at the same step are conditionally independent across clients.

\emph{Sampling and initialization.} The training and refresh sets are drawn uniformly at random before the round's gradients are computed, independently across rounds and of the gradient samples. The initial tracker and moment errors have finite mean squares.
\end{assumption}

\begin{assumption}[Adam family]\label{assump:BG}
For Adam, AMSGrad, and round-frozen SOAP, every client gradient is bounded, $\|\nabla f_i(x)\|_2\le G$, the denominator uses $\epsilon>0$, and the noise has sub-Gaussian tails with scale $\sigma$: for every unit vector $u$ and real $\lambda$,
\[
\mathbb E[e^{\lambda\langle u,\zeta_i^{t,k}\rangle}\mid\mathcal F^{t,k}]\le e^{\lambda^2\sigma^2/2}.
\]
The trackers start from a fresh gradient as in Algorithm~\ref{alg:FAGC}, from zero, or from any value whose error satisfies the same sub-Gaussian condition as the noise of a fresh gradient.
\end{assumption}

\begin{assumption}[FA-NT heterogeneity]\label{assump:BDH}
For FA-NT only, let $n^{-1}\sum_i\|\nabla f_i(x)-\nabla f(x)\|_2^2\le B^2$. Its Adam-family instances also use Assumption~\ref{assump:BG}, while its normalization and Muon instances need only this condition and Assumption~\ref{assump:genLoss}.
\end{assumption}

Assumption~\ref{assump:genLoss} is used throughout. Its two smoothness constants play different roles: $L$ controls how much one server step can decrease $f$, and $L_g$ controls how much a client's gradient changes as its model moves away from $x^t$. Sub-Gaussian tails mean that the noise along any direction exceeds $z$ in magnitude with probability at most $2e^{-z^2/(2\sigma^2)}$, as for Gaussian noise. We use them only for the Adam family, to bound the largest gradient coordinate seen during training and hence Adam's denominator. Normalization and Muon produce updates of bounded size by construction, so their FA-GC guarantees need neither bounded gradients, nor sub-Gaussian noise, nor bounded heterogeneity. For the Adam family, the bound $G$ on every client gradient also limits how different the clients can be, so the FA-GC bound needs no separate heterogeneity term.

\textbf{Largest Gradient Coordinate.} Fix a failure probability $p\in(0,1)$ and define
\begin{equation}
\begin{aligned}
H_{0,p}&=G+\sigma\sqrt{2\log\frac{2d(n+SKT)}p},\\
\lambda_p&=(3c_{\rm rot}H_{0,p}+\epsilon)^{-1},
\end{aligned}
\label{eq:main-preconditioner-envelope}
\end{equation}
where $d$ is the number of scalar parameters, $c_{\rm rot}=1$ for Adam and AMSGrad, and $c_{\rm rot}=\sqrt d$ for the SOAP variant in~\eqref{eq:soap-main}. Let $\mathcal H_p$ be the event that every coordinate of every stochastic gradient drawn during the run, including initialization, has magnitude at most $H_{0,p}$. A run draws $n+SKT$ gradients, one per client at initialization and $K$ per selected client and round, each with $d$ coordinates. The tail bound above and a union bound over these $d(n+SKT)$ coordinates give $\mathbb P(\mathcal H_p)\ge1-p$. On $\mathcal H_p$, the corrected gradient $\widehat g_i^{t,k}=g_i^{t,k}+y^t-y_i^t$ is a sum of three terms whose coordinates are at most $H_{0,p}$, because the trackers and their average $y^t$ are averages of raw gradients or zero. Hence every coordinate of $v$ is at most $(3H_{0,p})^2$, and every diagonal entry of Adam's scaling $(\sqrt v+\epsilon)^{-1}$ is at least $\lambda_p$; for SOAP, the rotation can enlarge a coordinate by the factor $c_{\rm rot}$. Theorem~\ref{thm:FA-GC} takes $p=(T+1)^{-2}$, so $H_{0,p}$ grows like $\sqrt{\log T}$, which is the source of its logarithmic factor. When $\mathcal H_p$ fails, the proofs use only $\|\nabla f\|_2\le G$, which adds the small term $\nu$ in Table~\ref{tab:constants_for_methods}.

We now state the two lemmas. Let $d_i^{t,k}$ be the direction that a selected client uses at local step $k$, and let $\bar d^t=(SK)^{-1}\sum_{i\in\mathcal S^t,k}d_i^{t,k}$ be their average. Step sizes and moment coefficients are chosen once for a run of $T$ rounds, so they carry the subscript $T$. With the server step scale $\alpha_T=K\eta_{g,T}\eta_{\ell,T}$, the server moves by $x^{t+1}-x^t=-\alpha_T\bar d^t$. We compare $\bar d^t$ with the simpler direction $d_\star^t=C_2\phi(\nabla f(x^t))$ listed in Table~\ref{tab:constants_for_methods}, which depends only on the global gradient. For normalization and Muon, $d_\star^t$ is the direction the optimizer would take if its momentum equaled $\nabla f(x^t)$, up to a constant factor. Unlike $\bar d^t$, it is fixed before the round's random samples are drawn. The momentum error
\begin{equation}
\textstyle\mathcal M_{\rm GC}^t:=\mathbb E\!\left[\frac1{SK}
\sum_{i\in\mathcal S^t,k}\|m_i^{t,k}-\nabla f(x^t)\|_2^2\right]
\label{eq:main-momentum-mismatch}
\end{equation}
measures how far the momenta of the selected clients are from the current global gradient, averaged over local steps; matrices use the Frobenius norm.

\begin{lemma}[Bounded adaptive direction, \proofref{app:proof-c1}]\label{lem:C1-main}
In FA-GC, for gradient normalization, exact Muon, and the finite-step map in~\eqref{eq:ns-main}, every local direction satisfies $\|d_i^{t,k}\|_{\rm pr}\le C_1=1$. For Adam, AMSGrad, and round-frozen SOAP, the directions have bounded mean squares. In all cases,
\begin{equation}
\mathcal U^t:=\mathbb E\|\bar d^t\|_{\rm pr}^2
\le U_{\rm loc}^t:=\mathbb E\!\left[\frac1{SK}\sum_{i\in\mathcal S^t,k}
\|d_i^{t,k}\|_{\rm pr}^2\right]\le U_1,
\label{eq:expected-movement}
\end{equation}
with $U_1=1$ for normalization and Muon and $U_1=C(G^2+\sigma^2)/\epsilon^2$ for the Adam family. The first inequality holds because the squared norm of an average is at most the average of the squared norms.
\end{lemma}

\begin{lemma}[Alignment with controllable mismatch, \proofref{app:proof-c23}]\label{lem:C23-main}
Take $d_\star^t$, $C_2$, $C_3$, $\chi_E\ge0$, and $\nu$ from Table~\ref{tab:constants_for_methods}, and let $E_{\rm mis}^t=(\mathcal M_{\rm GC}^t)^{q/2}$. Then
\begin{align}
\langle\nabla f(x^t),d_\star^t\rangle
&\ge C_2\|\nabla f(x^t)\|_*^q,\label{eq:main-progress}\\
-\mathbb E\langle\nabla f(x^t),\bar d^t-d_\star^t\rangle
&\le C_3 E_{\rm mis}^t-\chi_E U_{\rm loc}^t+\nu.\label{eq:main-mismatch}
\end{align}
\end{lemma}

\begin{table}[t]
\centering
\caption{Optimizer-specific quantities in Lemma~\ref{lem:C23-main}. The comparison direction is $d_\star^t=C_2\,\phi(\nabla f(x^t))$, and the rate is measured by $\|\nabla f\|_*^q$. Adam-family rows have the small term $\nu=\lambda_pG^2p/2$ from the event that $\mathcal H_p$ fails; the others have $\nu=0$. The last row compares Adam-family directions with the fixed scaling $\epsilon^{-1}I$; it is used for FA-NT and in Appendix~\ref{app:proof-fixed-gain}, and its bound~\eqref{eq:main-mismatch} also contains the residual $\epsilon\|(D-\epsilon^{-1}I)m\|_2^2$, where $D$ is the optimizer's scaling matrix. For a $d_1\times d_2$ Muon parameter, $r=\min(d_1,d_2)$ is its largest possible rank.}
\label{tab:constants_for_methods}
\small
\setlength{\tabcolsep}{3.5pt}
\renewcommand{\arraystretch}{1.15}
\begin{tabular}{@{}lccccc@{}}
\toprule
Method & $\phi(g)$ & $C_2$ & $\|\nabla f\|_*^q$ & $C_3$ & $\chi_E$\\
\midrule
Adam, AMSGrad, SOAP$^\dagger$ & $g$ & $\lambda_p/2$ & $\|\nabla f\|_2^2$ & $\frac1{2\epsilon}$ & $\frac\epsilon2$\\
Normalization & $N(g)$ & $1$ & $\|\nabla f\|_2$ & $2$ & $0$\\
Exact Muon & $P_{\rm M}(g)$ & $1$ & $\|\nabla f\|_{\rm tr}$ & $2\sqrt r$ & $0$\\
Finite-step Muon & $P_{\rm M}(g)$ & $r^{-1/2}$ & $\|\nabla f\|_{\rm tr}$ & $1+\sqrt r$ & $0$\\
\midrule
Adam family, reference & $g$ & $\frac1{2\epsilon}$ & $\|\nabla f\|_2^2$ & $\frac1\epsilon$ & $0$\\
\bottomrule
\multicolumn{6}{@{}l}{\footnotesize $^\dagger$SOAP uses the Frobenius norm.}
\end{tabular}
\end{table}

\textbf{What the Two Lemmas Say.} Inequality~\eqref{eq:main-progress} gives the progress a direction would make if every momentum equaled the global gradient. Inequality~\eqref{eq:main-mismatch} bounds the progress lost because the momenta differ from it, measured by the mismatch term $E_{\rm mis}^t$, a power of the momentum error. For the Adam family, \eqref{eq:main-mismatch} also contains the extra term $-\chi_EU_{\rm loc}^t$ with $\chi_E=\epsilon/2$: the larger the local directions, the more descent is guaranteed. Several errors in the analysis, such as local drift and stale trackers, grow with the same quantity $U_{\rm loc}^t$, times a factor that increases with the step size. If the step size is small enough, the extra descent pays for them, and we say that these errors are \emph{absorbed} (Step~4 in Section~\ref{sec:main-proof-sketch}). Normalization and Muon have $\chi_E=0$, but Lemma~\ref{lem:C1-main} limits every direction to norm one, which also limits these errors.

The errors depend on how quickly the moments and trackers take in new information. Write $a_T=1-\beta_{1,T}$ and $b_T=1-\beta_{2,T}$ for the weights that the two moments give to the newest gradient, and $s=S/n$ and $\refreshprob=\widetilde S/n$ for the fractions of clients that train and that refresh in a round. Over $K$ local steps, the first moment gives total weight $\bar a_T=1-(1-a_T)^K$ to the new gradients. The local step scale is $\delta_T=K\eta_{\ell,T}$. We keep it separate from $\alpha_T$ because client directions can cancel when the server averages them while each client still moves away from $x^t$.

Let $\Delta_0=f(x^0)-f_*$ be the initial objective gap. For the zero initial moments and fresh-gradient trackers in Algorithm~\ref{alg:FAGC}, the initial state scale is
\[
\textstyle\mathcal E_{\rm init}:=\|\nabla f(x^0)\|_2^2+\frac1n\sum_i\mathbb E\|y_i^0-\nabla f_i(x^0)\|_2^2.
\]
The second term is at most $\sigma^2$ for independent fresh-gradient initialization and vanishes for exact initial trackers. Zero tracker initialization is also allowed, with its corresponding error retained in $\mathcal E_{\rm init}$.

The next theorem measures stationarity by the squared gradient norm for the Adam family and by the gradient norm itself for normalization and Muon, so the two rates do not rank the optimizers. In both cases, the moment coefficients and step sizes are chosen from the total number of rounds $T$ and held constant during the run.

\begin{theorem}[\proofref{app:proof-fagc}]\label{thm:FA-GC}
Under Assumption~\ref{assump:genLoss}, fix $K\ge1$ and $0<\refreshprob\le s\le1$, and use the
persistent FA-GC moments without bias correction. Write
$\mathcal G_{q,T}=T^{-1}\sum_{t<T}\mathbb E\|\nabla f(x^t)\|_*^q$ and
$c_a=\min\{1/16,\refreshprob/(8s)\}$. In both branches set
$a_T=c_a/[K\sqrt{T+1}]$, $\eta_{g,T}=1$, and $\eta_{\ell,T}=\alpha_T/K$.
\begin{enumerate}[label=(\roman*)]
\item Adam, AMSGrad, and the round-frozen SOAP update in~\eqref{eq:soap-main}
additionally satisfy Assumption~\ref{assump:BG}, with fresh-gradient or zero
tracker initialization. Fix any $b=1-\beta_2\in(0,1)$, set $p_T=(T+1)^{-2}$,
and use~\eqref{eq:main-preconditioner-envelope} with $p=p_T$. Let
\begin{equation}
\begin{aligned}
C_{s,\refreshprob}&=280+336/\refreshprob^2+44/s^2,\\
c_\eta&=\min\left\{\frac{\epsilon}{2L},
\frac{\epsilon}{2L_g\sqrt{C_{s,\refreshprob}}}\right\},\qquad
\alpha_T=c_\eta\bar a_T.
\end{aligned}
\label{eq:main-absorption-schedule}
\end{equation}
Take $L,L_g>0$ to be valid smoothness upper bounds. For a finite constant $C$ independent of $T$,
\begin{equation}
\begin{aligned}
\mathcal G_{2,T}\le{}&\frac{2\Delta_0}{\lambda_{p_T}T\alpha_T}
+\frac C{\lambda_{p_T}}\left(\frac{\mathcal E_{\rm init}}{T\bar a_T}+a_T\sigma^2\right)
+G^2p_T\\
={}&O\!\left(\sqrt{\frac{\log(T+2)}T}\right).
\end{aligned}
\label{eq:main-explicit-residual}
\end{equation}
The norm is Euclidean for Adam/AMSGrad and Frobenius for SOAP. With uniformly
bounded stochastic gradients, the rate is $O(T^{-1/2})$.
\item Gradient normalization and exact or finite-step Muon take
$\alpha_T=(T+1)^{-3/4}$ and satisfy
\begin{equation}
\mathcal G_{1,T}=O(T^{-1/4}),
\label{eq:main-first-power-rate}
\end{equation}
in their gradient norms from Table~\ref{tab:constants_for_methods}, with the finite-step rank factor listed there.
\end{enumerate}
The constants hold dimension, $K,s,\refreshprob$, and the problem parameters fixed.
\end{theorem}

For the Adam family, Theorem~\ref{thm:FA-GC} gives a square-root rate up to a logarithmic factor with any fixed $\beta_2\in(0,1)$, so the second moment keeps adapting throughout training. The logarithmic factor comes from bounding the largest stochastic gradient and disappears when stochastic gradients are bounded. The rate holds with partial participation and with the training gradients reused to refresh the trackers. For normalization and Muon, the same argument gives $O(T^{-1/4})$.

The next corollary keeps $\beta_1$ fixed, as in practice, and lets the step sizes decrease.

\begin{corollary}[Fixed first-moment coefficient, \proofref{app:proof-fixed-gain}]\label{cor:fixed-gain-main}
Under Assumption~\ref{assump:genLoss}, fix $K\ge1$, $0<\refreshprob\le s\le1$, and $a=1-\beta_1\in(0,1)$, with persistent moments without bias correction. If $\alpha_T\to0$, $\delta_T\to0$, and $T\alpha_T\to\infty$, then gradient normalization and exact or finite-step Muon satisfy
\begin{equation}
\limsup_T\mathcal G_{1,T}\le C\sqrt{a+K^{-1}}\,\sigma.
\label{eq:main-fixed-q1}
\end{equation}
The constant $C$ collects the fixed smoothness, dimension, participation, and rank parameters. If either step scale does not vanish, the finite-horizon bound also retains terms of order $\alpha_T+\delta_T$.
\end{corollary}

With $\beta_1$ fixed, the noise no longer averages away. It settles at a level set by $a+K^{-1}$: momentum filters fresh gradient noise with weight $a$, while a refreshed tracker stores the noise averaged over $K$ local steps. In the image setting $a=0.1$ (that is, $\beta_1=0.9$), $K=60$, $s=0.1$, and $\refreshprob=0.05$, the contribution of the initial momentum and tracker noise is bounded by a constant times $0.95^t$, while fresh noise keeps entering at each round. Here $0.95=\max\{1-s\bar a,1-\refreshprob\}=\max\{0.90,0.95\}$: a momentum that trains forgets almost all of its old noise within one round ($\bar a\approx1$), so the slower tracker refresh sets the rate. Appendix~\ref{app:proof-fixed-gain} gives the Adam-family bound, which also contains a term for the change of the learned denominator.

We next apply the same argument to FA-NT. Relating each client's adapted direction to the global gradient now leaves a heterogeneity term.

\begin{theorem}[FA-NT comparison, \proofref{app:proof-fant}]\label{thm:FA-NT}
Under Assumptions~\ref{assump:genLoss} and~\ref{assump:BDH}, fix $K\ge1$ and $0<\refreshprob\le s\le1$. Use the FA-NT warm start and persistent updates of Section~\ref{ssec:gc_algorithm}, without moment bias correction. For $q\in\{1,2\}$ set $c_a=\min\{1/16,\refreshprob/(8s)\}$ and
\begin{align}
a_T&=\frac{c_a}{K(T+1)^{2/(q+3)}},\qquad
\alpha_T=\delta_T=(T+1)^{-3/(q+3)},\notag\\
\eta_{g,T}&=1,\qquad \eta_{\ell,T}=\alpha_T/K.
\label{eq:main-nt-schedule}
\end{align}
For Adam, AMSGrad, and the SOAP variant in~\eqref{eq:soap-main}, also assume~\ref{assump:BG} and $\epsilon>0$, and take $q=2$ and $b_T=c_b/[K(T+1)^2]$, where $c_b=\min\{1/16,c_a/4\}$. For normalization and exact/finite-step Muon, take $q=1$. Then
\begin{equation}
\mathcal G_{q,T}\le C T^{-q/(q+3)}+C B^q,
\label{eq:main-fant-explicit}
\end{equation}
where $C$ is independent of $T$ and $B$, and may depend on smoothness, noise, initialization, $G,\epsilon$ when used, dimension, $K,s,\refreshprob$, and the finite-step rank factor. The stationarity measures are the same as in Theorem~\ref{thm:FA-GC}.
\end{theorem}

For FA-NT, the bound is $O(T^{-2/5})+O(B^2)$ for Adam-family methods and $O(T^{-1/4})+O(B)$ for normalization and Muon. FA-GC avoids these explicit heterogeneity residuals by correcting the inputs to the moments. Lemma~\ref{lem:fixed-point-bias} shows that the placement also changes where the averaged direction vanishes in the noiseless fixed-point setting, except for linear maps such as SGD with momentum.

The Adam-family FA-NT proof compares each client's direction with the fixed scaling $\epsilon^{-1}I$ (last row of Table~\ref{tab:constants_for_methods}), and the difference enters its bound through the residual $\epsilon\|(D-\epsilon^{-1}I)m\|_2^2$. Taking $1-\beta_{2,T}=O(T^{-2})$ keeps the second moment small, so this residual shrinks as $T$ grows. With the image-task values below and $T=1000$, this gives $1-\beta_{2,T}\approx2.6\times10^{-10}$, so Adam's scaling barely changes during the run. Theorem~\ref{thm:FA-GC} allows any fixed $\beta_2$, so the two Adam-family results cover different regimes.

\textbf{Scope of the Guarantees.} The schedules in Theorems~\ref{thm:FA-GC} and~\ref{thm:FA-NT} are sufficient conditions chosen from the number of rounds $T$. For example, with $T=1000$ rounds and the image-task values $K=60$, $s=0.1$, and $\refreshprob=0.05$, Theorem~\ref{thm:FA-GC} takes $1-\beta_1\approx3\times10^{-5}$. Its Adam-family step size is proportional to $\epsilon$, so its bound grows like $1/\epsilon$ as $\epsilon$ becomes small. The experiments instead use fixed coefficients and step sizes common in practice. The rates are upper bounds, so the gap between the two theorems is not a statement about the best achievable rates. The constants hold $K$ and the participation fractions fixed, and the bounds do not quantify gains from more local steps, more participating clients, or more Newton--Schulz iterations. Section~\ref{sec:exp} lists how the practical SOAP and Muon implementation differ from the analyzed variants.

\subsection{Proof Sketch}\label{sec:main-proof-sketch}
We first outline the FA-GC proof in four steps, then explain the FA-NT comparison. Step~1 turns Lemmas~\ref{lem:C1-main} and~\ref{lem:C23-main} into a one-round descent inequality whose only unknown is the momentum error. Step~2 supplies the optimizer-specific coefficients. Step~3 bounds the momentum error and contains the parts specific to federated adaptive training: moments kept across rounds, selective tracking, and gradients reused in both the moments and the trackers. Step~4 chooses the schedules. Complete proofs are in Appendices~\ref{app:main-results}, \ref{app:core-proofs}, and~\ref{sec:optimizer-instantiations}.

\textbf{Step 1: One-round descent.} The server moves by $x^{t+1}-x^t=-\alpha_T\bar d^t$. Applying smoothness, inserting the comparison direction $d_\star^t$, and using Lemma~\ref{lem:C23-main} with $\mathcal U^t\le U_{\rm loc}^t$ gives
\begin{equation}
\begin{aligned}
\mathbb Ef(x^{t+1})\le{}&\mathbb Ef(x^t)-\alpha_TC_2\mathbb E\|\nabla f(x^t)\|_*^q\\
&+\alpha_T\{C_3E_{\rm mis}^t+\nu\}\\
&-\alpha_T(\chi_E-L\alpha_T/2)\,U_{\rm loc}^t.
\end{aligned}
\label{eq:main-one-round-descent}
\end{equation}
The first line is ordinary descent. The second line is its cost: the mismatch term $E_{\rm mis}^t$ and the small term $\nu$. The third line comes from the optimizer. When $\chi_E>0$, each round also gains a decrease proportional to the squared local direction $U_{\rm loc}^t$. The rest of the proof bounds $E_{\rm mis}^t$ and shows that its movement-driven part is paid for by this gain.

\textbf{Step 2: Optimizer inequalities.} For Adam-family methods, the same gradients enter both the momentum $m$ and the scaling $D=\operatorname{diag}((\sqrt v+\epsilon)^{-1})$, whose diagonal entries all lie in $(0,\epsilon^{-1}]$. With $g=\nabla f(x^t)$, expanding the last term gives the identity
\[
2\langle g,Dm\rangle=g^\top Dg+m^\top Dm-(m-g)^\top D(m-g).
\]
Since every entry of $D$ is at most $\epsilon^{-1}$, $m^\top Dm\ge\epsilon\|Dm\|_2^2$ and $(m-g)^\top D(m-g)\le\epsilon^{-1}\|m-g\|_2^2$, so~\cite{guo2021novel}
\begin{equation}
2\langle g,Dm\rangle\ge g^\top Dg+\epsilon\|Dm\|_2^2
-\epsilon^{-1}\|m-g\|_2^2.
\label{eq:main-adam-sketch}
\end{equation}
This is an algebraic identity, so no independence between $m$ and $v$ is needed. SOAP applies the same argument in its rotated coordinates. On the event $\mathcal H_p$, every entry of $D$ is at least $\lambda_p$, which bounds the first term below by $\lambda_p\|g\|_2^2$. After dividing by two, the last term gives $C_3=(2\epsilon)^{-1}$, and the middle term $(\epsilon/2)\|Dm\|_2^2$ is the movement gain $\chi_E=\epsilon/2$. Normalization $N(m)=m/\|m\|_2$ instead satisfies
\begin{equation}
\langle g,N(m)\rangle\ge\|g\|_2-2\|m-g\|_2,
\label{eq:main-normalization-sketch}
\end{equation}
because $\langle g,N(g)\rangle=\|g\|_2$ and $\|N(m)-N(g)\|_2\le2\|m-g\|_2/\|g\|_2$. It has $\chi_E=0$, but its directions have norm at most one. Muon satisfies the same inequality in the nuclear norm, with the rank factors in Table~\ref{tab:constants_for_methods}.

\textbf{Step 3: The momentum error.} Write each tracker as $y_i^t=y_{i,\rm pop}^t+\omega_i^t$. Here $y_{i,\rm pop}^t=K^{-1}\sum_k\nabla f_i(\gamma_i^{t,k})$ averages noise-free gradients at the points $\gamma_i^{t,k}$ visited in client $i$'s last refresh round, and $\omega_i^t$ is the stored sampling noise. Correspondingly, split $m_i^{t,k}=m_{i,\rm pop}^{t,k}+e_i^{t,k}$. The first part is driven by the noise-free inputs $g_{i,\rm pop}^{t,k}=\nabla f_i(x_i^{t,k})+y_{\rm pop}^t-y_{i,\rm pop}^t$, with $y_{\rm pop}^t=n^{-1}\sum_iy_{i,\rm pop}^t$, and the second by noise. The momentum error then comes from four sources, each with its own recursion:
\begin{center}
\small
\setlength{\tabcolsep}{3pt}
\begin{tabular}{@{}lll@{}}
\toprule
Error & Meaning & Bound\\
\midrule
$\mathcal E_{\rm dr}^t$ & drift of local models from $x^t$ & mean $\le\delta_T^2U_{\rm loc}^t$\\
$\Xi^t$ & tracker lag behind $x^t$ & shrinks by $1-\refreshprob/2$\\
$\mathcal E_{\rm pop}^t$ & noise-free momentum error & shrinks by $1-s\bar a_T/2$\\
$\mathcal N^t$ & noise left in the momenta & settles near $a_T\sigma^2$\\
\bottomrule
\end{tabular}
\end{center}
Three bounds control them.

\emph{(a) Placement.} Adding and subtracting each client's gradient at $x^t$ gives
\begin{equation}
\mathbb E\Big[\frac1{SK}\sum_{i\in\mathcal S^t,k}\|g_{i,\rm pop}^{t,k}-\nabla f(x^t)\|_2^2\Big]
\le6L_g^2(\mathbb E\mathcal E_{\rm dr}^t+\Xi^t),
\label{eq:main-placement-sketch}
\end{equation}
where $\mathcal E_{\rm dr}^t=(SK)^{-1}\sum_{i\in\mathcal S^t,k}\|x_i^{t,k}-x^t\|_{\rm pr}^2$ is the local drift and $\Xi^t=(nK)^{-1}\sum_{i,k}\mathbb E\|\gamma_i^{t,k}-x^t\|_{\rm pr}^2$ is the tracker lag. The correction removes each client's offset $\nabla f_i(x^t)-\nabla f(x^t)$, so no heterogeneity term appears. FA-NT's moments receive raw gradients, so each momentum follows its own client's gradient $\nabla f_i(x^t)$, while its sampling noise is still averaged by the momentum. Comparing it with the global gradient through $\|m_i-\nabla f(x^t)\|_2^2\le2\|m_i-\nabla f_i(x^t)\|_2^2+2\|\nabla f_i(x^t)-\nabla f(x^t)\|_2^2$ adds $2B^2$ after averaging over clients (Proposition~\ref{prop:fant-bridge}). This is where $B^q$ enters Theorem~\ref{thm:FA-NT}.

\emph{(b) Persistent moments and selective tracking.} The population error $\mathcal E_{\rm pop}^t=n^{-1}\sum_i\mathbb E\|m_{i,\rm pop}^t-\nabla f(x^t)\|_2^2$ and the tracker lag satisfy
\begin{align}
\mathcal E_{\rm pop}^{t+1}\le{}&(1-\tfrac12s\bar a_T)\mathcal E_{\rm pop}^t
+C_ps\bar a_TL_g^2(\mathbb E\mathcal E_{\rm dr}^t+\Xi^t)\notag\\
&+\frac{C_pL_g^2\alpha_T^2\,\mathcal U^t}{s\bar a_T},\notag\\
\Xi^{t+1}\le{}&(1-\tfrac12\refreshprob)\Xi^t+2\refreshprob\,\mathbb E\mathcal E_{\rm dr}^t
+\frac{3}{\refreshprob}\alpha_T^2\,\mathcal U^t,
\label{eq:main-state-sketch}
\end{align}
with $C_p<11$ under the theorem's schedule. A client's momentum receives a fraction $\bar a_T$ of new information when it participates, so its error shrinks at rate $s\bar a_T$. A tracker is replaced only when refreshed, so the lag shrinks at rate $\refreshprob$; training without refreshing does not reduce it. The local drift satisfies $\mathbb E\mathcal E_{\rm dr}^t\le\delta_T^2U_{\rm loc}^t$, while server motion contributes $\alpha_T^2\mathcal U^t$. This is why we keep $\alpha_T$ and $\delta_T$ separate.

\emph{(c) Reused gradients.} A refreshed tracker stores the average noise of the gradients that already entered the momentum:
\begin{equation}
\omega_i^{t+1}=(1-\refreshmask_i^t)\,\omega_i^t+\refreshmask_i^t\frac1K\sum_{k=1}^K\zeta_i^{t,k},
\label{eq:main-buffer-recursion}
\end{equation}
where $\refreshmask_i^t$ indicates a refresh. Each noise sample thus enters the momentum directly and later again through $u_i^t=\bar\omega^t-\omega_i^t$, the noise part of the correction, with $\bar\omega^t=n^{-1}\sum_i\omega_i^t$. We therefore follow the momentum noise and the stored tracker noise together through $z_i^t=e_i^t+\theta_Tu_i^t$, with $\theta_T=s\bar a_T/(\refreshprob-s\bar a_T)$. This weight makes the expected value of $z_i^{t+1}$, given the past, exactly $(1-s\bar a_T)z_i^t$. The remaining random part has three pieces: which clients train, the fresh noise, and which trackers are refreshed. The last two reuse the same samples and are correlated, so we bound their sum with $\|a+b+c\|^2\le3(\|a\|^2+\|b\|^2+\|c\|^2)$. When $\bar a_T\le1/8$ and $s\bar a_T\le\refreshprob/2$, this gives
\begin{equation}
\mathcal N^t:=\frac1n\sum_i\mathbb E\|e_i^t\|_2^2
\le Ce^{-cs\bar a_Tt}(\mathcal N^0+\Omega^0)+\Gamma_{s,\refreshprob}\,a_T\sigma^2,
\label{eq:main-noise-sketch}
\end{equation}
where $\Omega^0$ is the initial tracker-noise energy and $\Gamma_{s,\refreshprob}=C(1+s/\refreshprob)$. The noise floor $a_T\sigma^2$ matches that of ordinary momentum. Reusing the training gradients therefore costs the factor $\Gamma_{s,\refreshprob}$ but no extra gradient batch.

\textbf{Step 4: Closing the argument.} The actual momentum error satisfies $\mathcal M_{\rm GC}^t\le2\mathcal E_{\rm pop}^t+12L_g^2(\mathbb E\mathcal E_{\rm dr}^t+\Xi^t)+2\mathcal N_{\rm loc}^t$, where $\mathcal N_{\rm loc}^t$ is the noise energy at the selected local steps. For the Adam-family schedule, $\delta_T=\alpha_T$. Summing~\eqref{eq:main-state-sketch} and~\eqref{eq:main-noise-sketch} over rounds gives
\begin{align}
\textstyle\frac1T\sum_{t<T}\mathcal M_{\rm GC}^t
\le{}&C\left(\frac{\mathcal E_{\rm init}}{T\bar a_T}+a_T\sigma^2\right)\notag\\
&+L_g^2C_{s,\refreshprob}\left(\frac{\alpha_T}{\bar a_T}\right)^2
\textstyle\frac1T\sum_{t<T}U_{\rm loc}^t.
\label{eq:main-master-sketch}
\end{align}
The first line is the initialization error $\mathcal E_{\rm init}$, which decays with $T\bar a_T$, plus the filtered noise. The second line is the cost of movement. Here $U_{\rm loc}^t$ is the average squared local direction in~\eqref{eq:expected-movement}, and $C_{s,\refreshprob}=280+336/\refreshprob^2+44/s^2$ collects the contributions of local drift, stale trackers, and the moving global gradient. Insert~\eqref{eq:main-master-sketch} into~\eqref{eq:main-one-round-descent} with $E_{\rm mis}^t=\mathcal M_{\rm GC}^t$, $C_3=(2\epsilon)^{-1}$, and $\chi_E=\epsilon/2$. After averaging over rounds, the movement terms combine to
\begin{equation}
-\Big(\frac\epsilon2-\frac{L\alpha_T}2-\frac{L_g^2C_{s,\refreshprob}\alpha_T^2}{2\epsilon\,\bar a_T^2}\Big)\overline U_{\rm loc}
\le-\frac\epsilon8\,\overline U_{\rm loc},
\label{eq:main-absorption-sketch}
\end{equation}
where $\overline U_{\rm loc}=T^{-1}\sum_{t<T}U_{\rm loc}^t$. This holds because~\eqref{eq:main-absorption-schedule} gives $L\alpha_T\le\epsilon/2$ and $L_g^2C_{s,\refreshprob}\alpha_T^2/\bar a_T^2\le\epsilon^2/4$. The gain from Adam's scaling therefore pays for every movement-driven error, and the term can be dropped. Telescoping~\eqref{eq:main-one-round-descent} and dividing by $C_2=\lambda_{p_T}/2$ gives~\eqref{eq:main-explicit-residual}. With $a_T\propto T^{-1/2}$, each term is $O(T^{-1/2})$ before the factor $\lambda_{p_T}^{-1}=O(\sqrt{\log T})$. The same calculation explains why the step size is proportional to $\epsilon$: the gain $\epsilon/2$ must exceed a movement cost multiplied by $C_3=(2\epsilon)^{-1}$.

\textbf{Normalization, Muon, and FA-NT.} For normalization and Muon, $\chi_E=0$, so movement cannot be absorbed; instead, Lemma~\ref{lem:C1-main} bounds the norm of every direction by one. The same state bounds enter~\eqref{eq:main-one-round-descent} through $E_{\rm mis}^t=(\mathcal M_{\rm GC}^t)^{1/2}$. The schedule then balances the objective term $1/(T\alpha_T)$, the filtered noise $\sqrt{a_T}\sigma$, and the momentum lag $\alpha_T/\bar a_T$ at $T^{-1/4}$. For FA-NT, the same recursions bound the distance between each momentum and its client's gradient $\nabla f_i(x^t)$. The comparison in Step~3(a) then adds $2B^2$, or $B$ when $q=1$, which becomes the $B^q$ term in~\eqref{eq:main-fant-explicit}.

\begin{table*}[t]
\caption{Communication rounds to target, reported as mean $\pm$ standard deviation over four trials. Unless otherwise stated, each round samples $S=10$ of the $n=100$ clients. Targets are 75\% for CIFAR-10, 50\% for CIFAR-100, 30\% for TinyImageNet, 85\% for SST-2, and 75\% for the other language tasks.}
\label{table:baseline}
\begin{center}
\resizebox{0.98\textwidth}{!}{
\setlength{\extrarowheight}{3pt} 
\begin{small}
\begin{tabular}{c|ccccccc
cc|
>{\columncolor[rgb]{0.85,0.9,0.9}}c}
\Xhline{1.5pt}
Settings & Task Type & Dataset & \makecell{FedAvg-M \\~\cite{cheng2023momentum}} & \makecell{SCAFFOLD-M\\~\cite{cheng2023momentum}} & \makecell{FedAdam\\~\cite{reddi2020adaptive}} & \makecell{FedAMS\\~\cite{wang2022communication}} & \makecell{LocalAdam} & \makecell{PAdaMFed\\~\cite{yan2025problem}} & FA-NT (Adam) & FA-GC (Adam)\\
\Xhline{1.5pt}
 
\multirow{7}{*}{\makecell{Communication \\Rounds}}
  & \multirow{3}{*}{\makecell{Image \\
Tasks}} 
& CIFAR-10 & 1014.5\com{250.3} & 544.3 \com{59.9} & 2532.5 \com{343.3} & 2388.8 \com{286.6} & 589.5\com{74.0}&360.3\com{20.8} & 394.8\com{31.3} & \textbf{310.0}\com{16.8}\\
\hhline{~~--------}
& & CIFAR-100 & 998.5\com{88.5} & 621.8\com{55.4} & 1854.0\com{185.3} & 1654\com{156.4} & 678.3\com{40.6}& 511.0\com{15.2} &530.3\com{17.6} & \textbf{323.8}\com{16.3}\\
\hhline{~~--------}
& & TinyImageNet & 215.5\com{10.4} & 242.2\com{15.4} & 543.2\com{45.2} & 463.5\com{35.7} & 177.3\com{8.3} & 87.5\com{6.4} &157.0\com{6.4} & \textbf{66.3}\com{4.4} \\
\hhline{~---------}
\hhline{~---------}
\hhline{~---------}

& \multirow{4}{*}{\makecell{Language \\
Tasks}} & 20NewsGroups & 245.5\com{27.6}& 214.0\com{17.7} & 247.0\com{8.1}& 224.5\com{5.4} & 156.8\com{7.8}& 149.3\com{3.4} & 155.0\com{7.0}& \textbf{143.3}\com{4.1}\\
\hhline{~~--------}
& & QNLI & 171.3\com{41.2}& 162.5\com{37.4}& 137.5\com{10.2}& 145.0\com{9.8}& 117.0\com{10.4} & 87.3\com{13.2} & 99.8\com{11.2} & \textbf{55.5}\com{16.3}\\
\hhline{~~--------}
& & QQP & 316.5\com{64.3} & 299.0\com{70.3} & 245.3\com{20.5} & 267.8\com{22.1} & 213.0\com{53.8} & 84.7\com{4.3}  & 196.3\com{5.1}  & \textbf{63.0}\com{4.6}\\
\hhline{~~--------}
& & SST-2 & 150.5\com{36.1} & 129.8\com{32.5} & 84.3\com{4.2} & 78.8\com{4.4} & 47.0\com{5.8}& 42.3\com{7.9}  & 48.8\com{8.4} & \textbf{30.3}\com{7.3}\\
\Xhline{1.5pt}

\end{tabular}

\end{small}
} \\
\end{center}
\end{table*}

\begin{figure*}[!t]
\centering
\subfloat[Cost under varying Dirichlet parameters to attain 50\% accuracy.\label{fig:non-iid-panel}]{\includegraphics[width=0.485\textwidth]{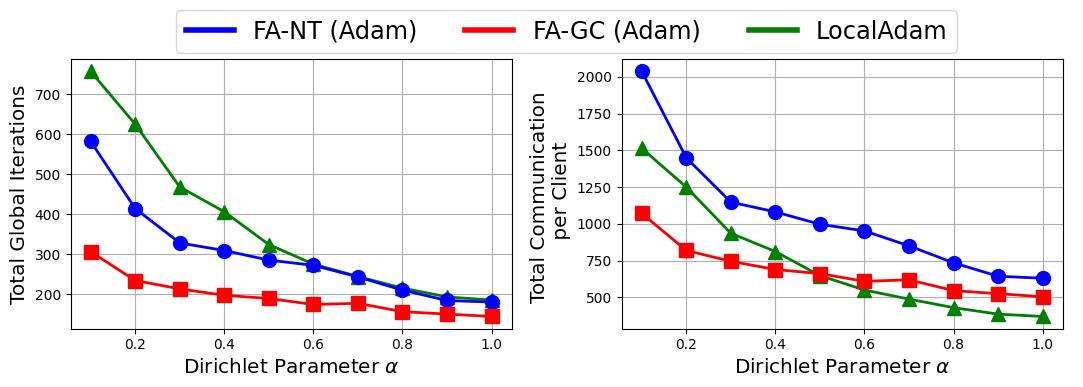}}\hfill
\subfloat[Rounds and communication per client as $\widetilde S$ varies, with $S=50$.\label{fig:tracking-rate-panel}]{\includegraphics[width=0.485\textwidth]{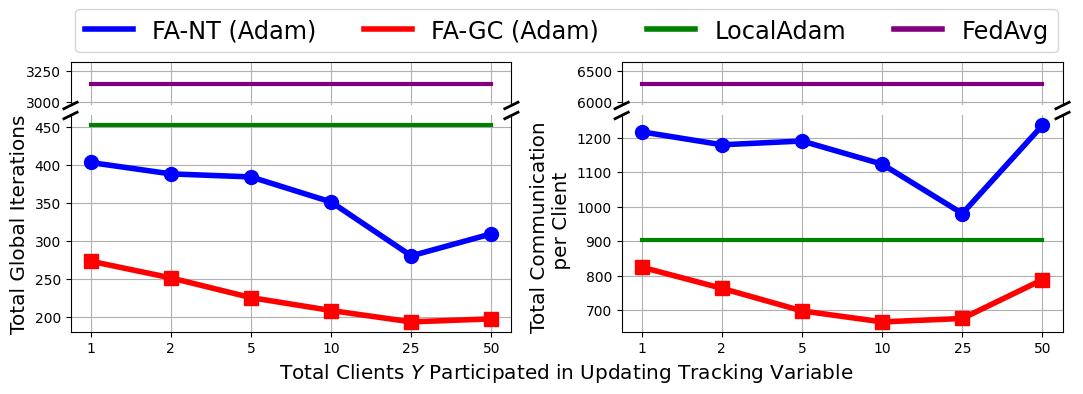}}
\caption{Adam sensitivity on CIFAR-100. In each pair, the left plot reports communication rounds and the right plot reports communication per client. (a) Cost to reach 50\% accuracy as the Dirichlet parameter varies. Smaller $\alpha$ means stronger heterogeneity. Each point in both panels is a single run, so the values at $\alpha=0.1$ differ from the four-trial means in Table~\ref{table:baseline}. (b) With $n=100$ and $S=50$, the horizontal axis $Y$ denotes the refresh count $\widetilde S$. Communication counts total transmitted model-sized vectors divided by $S$. Global iterations are communication rounds. The horizontal lines show LocalAdam and FedAvg (local SGD with server averaging), which use no trackers and so do not depend on $\widetilde S$.}
\label{fig:main_sensitivity}
\end{figure*}

\begin{figure*}[!t]
\centering
\includegraphics[width=0.99\textwidth]{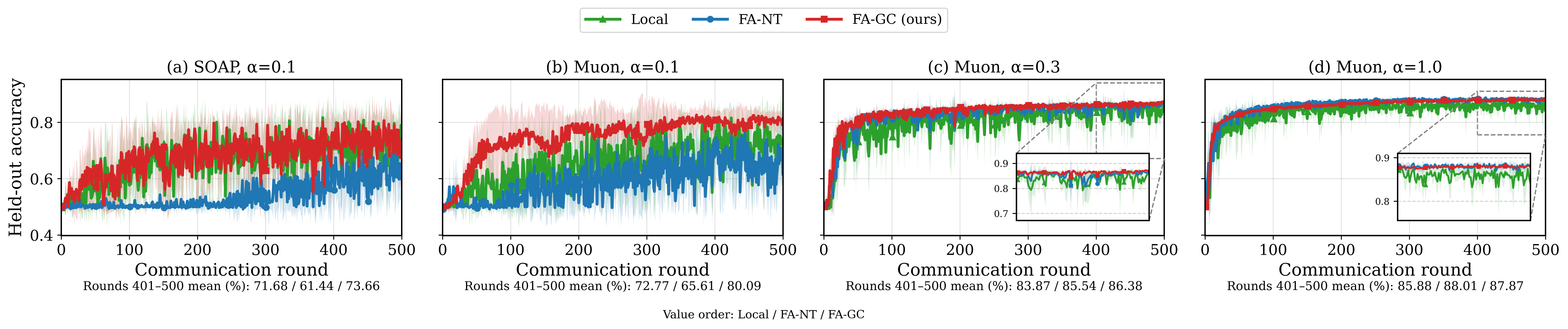}
\caption{QNLI with GPT-2 and fixed $\beta_1,\beta_2$: practical SOAP at $\alpha=0.1$ and Muon with five Newton--Schulz iterations at $\alpha\in\{0.1,0.3,1.0\}$ (Section~\ref{sec:exp}). Local is the same optimizer without tracking. All runs use $n=100$, $S=10$, one local epoch, 500 rounds, and four shared seeds. Curves and bands show the mean and one standard deviation; printed values are rounds-401--500 means in Local/FA-NT/FA-GC order.}
\label{fig:qnli-soap-muon}
\end{figure*}
 
\section{Experiments}
\label{sec:exp}

\textbf{Setup.}
The primary benchmark compares the Adam instances of LocalAdam, FA-NT, and FA-GC, together with FedAvg-M, SCAFFOLD-M, FedAdam, FedAMS, and PAdaMFed. FedAvg-M and SCAFFOLD-M add a server momentum, which clients use in their local steps, to FedAvg (local SGD followed by server averaging~\cite{mcmahan2017communication}) and to SCAFFOLD~\cite{cheng2023momentum}. LocalAdam applies Adam on each client without a tracker and keeps its moments across rounds, while FedAdam and FedAMS apply adaptation at the server. We use ResNet-18~\cite{he2016deep} on CIFAR-10, CIFAR-100~\cite{krizhevsky2009learning}, and TinyImageNet~\cite{le2015tiny}. Language experiments fine-tune GPT-2~\cite{radford2019language} with low-rank adaptation (LoRA)~\cite{hu2022lora}, which trains and communicates only small low-rank adapter weights, on topic classification (20NewsGroups), question entailment (QNLI), paraphrase detection (QQP), and sentiment classification (SST-2)~\cite{lang1995newsweeder,wang2018glue}.

Unless stated otherwise, $n=100$, $S=10$, and client data follow a Dirichlet partition with concentration parameter $\alpha=0.1$~\cite{hsu2019measuring}, which is unrelated to the step size $\alpha_T$ of Section~\ref{sec:convergence}. Smaller $\alpha$ gives each client fewer classes, which means stronger heterogeneity. Image training uses $K=60$ steps and $\widetilde S=5$. Language training uses one local epoch and $\widetilde S=10$, so its local step count can vary with client data size.

The primary Adam setting is $(\beta_1,\beta_2)=(0.9,0.99)$ with $\epsilon=10^{-8}$. All Adam optimizers in the benchmark use PyTorch's Adam, which includes bias correction. Theorem~\ref{thm:FA-GC} extends unchanged to its second-moment part, and the first-moment part matters only until each client has taken enough updates (Appendix~\ref{app:extensions}). Tables~\ref{table:lr} and~\ref{table:lr_peft} in Appendix~\ref{appen:lr_and_target_acc} record the selected local and server step sizes. For server-adaptive methods, $\eta_g$ is their server optimizer step. All benchmark entries report mean and standard deviation over four trials. The target accuracies appear in Table~\ref{table:baseline}.

\textbf{Communication metric.}
We report rounds to reach the target accuracy and the aggregate training-round payload. With $V$ model-sized vectors transmitted per round and $T_{\rm target}$ rounds to reach the target accuracy,
\[
\mathrm{CommVol}=T_{\rm target}\cdot V\cdot(\text{bytes per vector})/10^9\quad\text{GB}.
\]
A broadcast counts once per receiving client. LocalAdam, FedAdam, and FedAMS send $2S$ vectors per round. FedAvg-M sends $3S$, including server momentum. SCAFFOLD-M sends $4S$, counting its combined control/momentum downlink and separate model/control uploads~\cite{cheng2023momentum}. PAdaMFed sends $6S$, since each training client downloads and uploads the model, the control variate, and the momentum. FA-NT and FA-GC each send $3S+\widetilde S$. The metric excludes one-time initialization and network headers.
For the image comparison, each vector is the ResNet-18 model in 32-bit floating point (FP32), $46.8\times10^6$ bytes, and Table~\ref{tab:comm-main} applies this formula to Table~\ref{table:baseline}.

\begin{table}[t]
\centering
\caption{Aggregate image-training payload to target (decimal GB), derived from the round means and standard deviations.}
\label{tab:comm-main}
\setlength{\tabcolsep}{3pt}
\begin{tabular}{lccc}
\toprule
Dataset & LocalAdam & FA-NT & FA-GC\\
\midrule
CIFAR-10 & $551.77\pm69.26$ & $646.68\pm51.27$ & $\mathbf{507.78}\pm27.52$\\
CIFAR-100 & $634.89\pm38.00$ & $868.63\pm28.83$ & $\mathbf{530.38}\pm26.70$\\
TinyImageNet & $165.95\pm7.77$ & $257.17\pm10.48$ & $\mathbf{108.60}\pm7.21$\\
\bottomrule
\end{tabular}
\end{table}

\textbf{Baseline comparison.}
Table~\ref{table:baseline} compares rounds to the same task-specific targets. FA-GC has the lowest reported mean on all seven tasks, including against PAdaMFed. In the Adam comparison, FA-GC also uses less aggregate image-training communication than LocalAdam and FA-NT (Table~\ref{tab:comm-main}), and less than every other baseline, including PAdaMFed (Appendix~\ref{appen:storage}). Thus its reduced round count offsets the added tracker payload. FA-NT and FA-GC have identical per-round communication when their training and refresh sizes match, so their volumes differ only through the number of rounds. On the language tasks every training client also refreshes ($\widetilde S=S$), so FA-GC sends twice as many vectors per round as LocalAdam. It therefore uses less total communication than LocalAdam only where it needs fewer than half the rounds, which holds on QNLI and QQP but not on 20NewsGroups and SST-2. With fewer local steps ($K=10$), the two placements are closer, and FA-NT is slightly better at $10\%$ sampling (Appendix~\ref{appen:additional_experiments}), consistent with less client drift per round.

\textbf{Impact of Data Heterogeneity.}
Fig.~\ref{fig:main_sensitivity}(a) varies the Dirichlet parameter $\alpha$ from $0.1$ to $1$ on CIFAR-100. Smaller $\alpha$ means stronger heterogeneity. We compare the Adam versions of FA-GC, FA-NT, and LocalAdam. As heterogeneity increases, the gap between FA-GC and LocalAdam grows in both communication rounds and total communication per client, consistent with the role of gradient correction in mitigating client drift. With milder heterogeneity ($\alpha\ge0.5$), FA-GC still needs fewer rounds, but LocalAdam uses less total communication, since FA-GC sends $3.5$ instead of $2$ model-sized vectors per client and round. FA-GC costs less than FA-NT at every $\alpha$ tested, in line with Lemma~\ref{lem:fixed-point-bias}.

\textbf{Communication Efficiency under Different $\widetilde{S}$.}
Refreshing fewer trackers saves messages but leaves older gradient information in the trackers. Figure~\ref{fig:main_sensitivity}(b) fixes $n=100$ and $S=50$, and varies only $\widetilde S$ from 1 to 50. FA-GC generally needs more rounds when fewer trackers are refreshed. The FA-NT curve is nonmonotone. The right-hand plot shows that a moderate refresh subset reduces communication while retaining much of the convergence benefit.

\textbf{SOAP and Muon setup.}
The QNLI study fine-tunes GPT-2 with LoRA for 500 rounds. Local (the same optimizer without tracking), FA-NT, and FA-GC share data partitions, initialization, and client sampling, and use step sizes chosen by the same search. All three retain their moments across rounds, while FA-NT and FA-GC also retain their trackers. The practical optimizers differ from the analyzed variants: SOAP re-expresses its stored second moment in the new basis when the basis changes, whereas round-frozen SOAP keeps the stored entries unchanged; Muon uses five Newton--Schulz iterations with the coefficients of~\cite{jordan2024muon} instead of~\eqref{eq:ns-main}; and both apply weight decay. Appendix~\ref{appen:qnli-soap-muon} gives the full configuration.

\textbf{SOAP and Muon on QNLI.}
Figure~\ref{fig:qnli-soap-muon} tests whether the placement benefit extends beyond Adam. Under strong heterogeneity ($\alpha=0.1$), FA-GC gives the highest mean test accuracy for both SOAP and Muon. Averaged over rounds 401--500, the Local/FA-NT/FA-GC accuracies are 71.68/61.44/73.66\% for SOAP and 72.77/65.61/80.09\% for Muon. The Muon gap narrows as client data become more similar, and FA-NT is marginally higher at $\alpha=1.0$. At $\alpha=0.1$, FA-NT also falls below its Local counterpart for both optimizers. The SOAP gain over Local is modest relative to the plotted variability, while the Muon difference is larger. Muon on SST-2 gives the same ordering at $\alpha=0.1$ (Appendix~\ref{appen:sst2-muon}).

\textbf{Practical second-moment behavior.} A larger $\beta_2$ gives the second moment a longer memory of past squared gradients. Table~\ref{table:beta2_main} tests this smoothing effect: increasing fixed $\beta_2$ reduces mean rounds to target on all three image datasets. Here $\beta_1$ and the learning rates are held fixed.

\begin{table}[t]
\caption{Rounds to target for FA-GC (Adam) under fixed $\beta_2$: mean $\pm$ standard deviation over four trials; $\beta_2=0.99$ is the default.}
\label{table:beta2_main}
\centering
\footnotesize
\setlength{\tabcolsep}{2.2pt}
\begin{tabular}{@{}lcccc@{}}
\toprule
Dataset & \multicolumn{4}{c}{$\beta_2$}\\
\cmidrule(lr){2-5}
& $0$ & $0.3$ & $0.7$ & $0.99$\\
\midrule
CIFAR-10 & $358.8\!\pm\!21.4$ & $334.3\!\pm\!20.2$ & $318.8\!\pm\!15.3$ & $\mathbf{310.0\!\pm\!16.8}$\\
CIFAR-100 & $527.0\!\pm\!16.5$ & $413.3\!\pm\!15.5$ & $370.5\!\pm\!14.7$ & $\mathbf{323.8\!\pm\!16.3}$\\
TinyImageNet & $85.5\!\pm\!5.7$ & $75.3\!\pm\!5.7$ & $74.0\!\pm\!6.6$ & $\mathbf{66.3\!\pm\!4.4}$\\
\bottomrule
\end{tabular}
\end{table}

\textbf{Client storage.} FA-GC stores one additional trainable-parameter-sized tracker $y_i$ beyond LocalAdam: about 47~MB for ResNet-18 and 3~MB for the GPT-2 LoRA adapters (Table~\ref{tab:client_storage_overhead} in Appendix~\ref{appen:storage}). Selective tracking saves messages, not storage.

\section{Conclusion and Future Works}

\label{sec:conclusion}
We studied how drift compensation interacts with the optimizer state that adaptive clients keep between rounds. Correcting gradients before they enter the moments makes FA-GC follow the global optimizer's direction in the noiseless fixed-point setting, whereas the averaged direction of FA-NT can vanish at non-stationary points. Our analysis covers reused training gradients, Adam-family methods with fixed $\beta_2$, gradient normalization, and Muon, and the experiments support correcting before adapting, most clearly under strong heterogeneity. Future work includes weight decay~\cite{loshchilov2017decoupled}, SOAP's basis changes, the practical Newton--Schulz polynomial, asynchronous participation, and sharper bounds.

\bibliographystyle{IEEEtran}
\bibliography{example_paper}

\newpage
\appendices
\onecolumn

\begin{center}
    {\bf\Large Supplemental Appendix}
\end{center}

\noindent Appendix~\ref{app:roadmap} introduces the proof roadmap and setup.
Appendix~\ref{app:core-lemmas} states the core bounds, and
Appendix~\ref{app:main-results} proves the main results. The core, optimizer,
and supporting proofs follow in Appendices~\ref{app:core-proofs},
\ref{sec:optimizer-instantiations}, and~\ref{app:supporting}.
Appendix~\ref{app:extensions} records extensions and scope;
Appendix~\ref{appen:experimental_part} collects experimental details.

\startcontents[sections]
\printcontents[sections]{l}{1}{\setcounter{tocdepth}{2}}
\clearpage

\providecommand{\R}{\mathbb{R}}
\providecommand{\E}{\mathbb{E}}
\providecommand{\Pp}{\mathbb{P}}
\providecommand{\F}{\mathcal{F}}
\providecommand{\cS}{\mathcal{S}}
\providecommand{\cG}{\mathcal{G}}
\providecommand{\cD}{\mathcal{E}_{\mathrm{dr}}}
\providecommand{\cX}{\mathcal{X}}
\providecommand{\cV}{\mathcal{V}}
\providecommand{\cM}{\mathcal{M}}
\providecommand{\cP}{\mathcal{E}_{\mathrm{pop}}}
\providecommand{\cN}{\mathcal{N}}
\providecommand{\cR}{\mathcal{R}}
\providecommand{\cU}{\mathcal{U}}
\providecommand{\ip}[2]{\left\langle #1,#2\right\rangle}
\providecommand{\norm}[1]{\left\lVert #1\right\rVert}
\providecommand{\abs}[1]{\left\lvert #1\right\rvert}
\providecommand{\diag}{\operatorname{diag}}
\providecommand{\one}{\mathbf{1}}
\providecommand{\od}{\mathbin{\odot}}
\providecommand{\FA}{\textnormal{FA-GC}}
\providecommand{\NT}{\textnormal{FA-NT}}

\section{Proof Roadmap and Setup}
\label{appen:proofs}
\label{app:roadmap}
The main results keep the numbering assigned in Section~\ref{sec:convergence}.
Their headings include links to the corresponding proofs. This
appendix first collects the core bounds, then proves the main results before
turning to the derivations of those bounds and the optimizer-specific details.
For brevity, the appendix calls the direction bound of
Lemma~\ref{lem:C1-main} property C1, and the two inequalities
\eqref{eq:main-progress} and~\eqref{eq:main-mismatch} of
Lemma~\ref{lem:C23-main} properties C2 and C3.

\begin{center}
\begin{tabularx}{\linewidth}{p{0.24\linewidth}X}
\toprule
Destination & Contents and reading purpose\\
\midrule
Appendix~\ref{app:core-lemmas} & Core state and descent bounds, with a separate proof link for each result.\\
Appendix~\ref{app:main-results} & Proofs of Theorem~\ref{thm:FA-GC}, Corollary~\ref{cor:fixed-gain-main}, and Theorem~\ref{thm:FA-NT}, each written out from start to finish.\\
Appendix~\ref{app:core-proofs} & Proofs of the core bounds, in the order of their statements.\\
Appendix~\ref{sec:optimizer-instantiations} & Proofs of Lemmas~\ref{lem:C1-main} and~\ref{lem:C23-main}, organized by optimizer.\\
Appendix~\ref{app:supporting} & Sampling identities, contraction and convolution tools, fourth moments, and schedule feasibility.\\
Appendix~\ref{app:extensions} & Bias correction, additional parameter regimes, and limitations.\\
Appendix~\ref{appen:experimental_part} & Experimental and implementation details.\\
\bottomrule
\end{tabularx}
\end{center}

\paragraph{Reading the theorem proof.}
After consulting the definitions below, go to Appendix~\ref{app:main-results}.
Each proof quotes the core bounds it needs and carries the argument through to
the claimed residual or rate without inserting their supporting derivations.
\paragraph{Reading the full argument.}
Read the statements in Appendix~\ref{app:core-lemmas}, the main proofs in
Appendix~\ref{app:main-results}, and then the core and optimizer proofs in
Appendices~\ref{app:core-proofs} and~\ref{sec:optimizer-instantiations}.
Appendix~\ref{app:supporting} supplies the auxiliary calculations used along
the way. Table~\ref{tab:notation} is a reference for returning to a proof.

\subsection{Notation and probability conventions}
There are $n$ clients. Communication rounds are indexed by
$t\in\{0,\ldots,T-1\}$ and local steps by $k\in[K]:=\{1,\ldots,K\}$.
At round $t$, a uniform subset $\cS^t\subset[n]$ of size $S$ trains. Conditional
on $\cS^t$, a uniform subset $\widetilde\cS^t\subseteq\cS^t$ of size
$\widetilde S$ refreshes trackers.  Write
\[
 I_i^t:=\mathbf 1_{\{i\in\cS^t\}},\qquad
 \refreshmask_i^t:=\mathbf 1_{\{i\in\widetilde\cS^t\}},\qquad \refreshmask_i^t\le I_i^t.
\]
The set pairs are drawn independently across rounds and independently of all
gradient samples, before any gradient of the current round is computed. Thus
\[
\Pp(i\in\cS^t)=\frac Sn=:s,\qquad
\Pp(i\in\widetilde\cS^t)=\frac{\widetilde S}{n}=:\refreshprob,
\qquad 0<\refreshprob\le s\le1.
\]
Here $\refreshprob$ is the unconditional refresh probability, so a selected
training client refreshes with probability
$\refreshprob/s=\widetilde S/S$.
The parameter space is a finite-dimensional real vector space $\cX$ equipped
with an update norm $\norm{\cdot}_{\mathrm{pr}}$ and the matching gradient norm
$\norm g_*:=\max_{\norm \Delta_{\rm pr}\le1}\ip g\Delta$, so that
$\ip g\Delta\le\norm g_*\norm \Delta_{\rm pr}$. The inner product is Euclidean for vectors and Frobenius for
matrices. The global and local objectives are
\[
f(x)=\frac1n\sum_{i=1}^n f_i(x),\qquad f_*:=\inf_x f(x)>-\infty.
\]
The server model at the beginning of a round is $x^t$ and the local model is
$x_i^{t,k}$, with $x_i^{t,1}=x^t$.
Let $\F^t$ contain all models, client trackers, first/second-moment states, and
randomness revealed strictly before sampling $\cS^t$. For the analysis, imagine
that every client draws its minibatches in every round, whether or not it is
selected, and computes the local path it would follow if selected. We call this
its \emph{counterfactual path}. Selection only decides which of these paths
enter the update. This picture is useful because a client's counterfactual path
does not depend on whether the client is selected, so averages over the random
selection can be taken separately from averages over the gradient noise. Let
$\F^{t,k}$ additionally contain the sampled sets and all minibatches drawn
through local step $k-1$. The mini-batch stochastic gradient is
\[
g_i^{t,k}=\nabla f_i(x_i^{t,k})+\zeta_i^{t,k},\qquad
\E[\zeta_i^{t,k}\mid\F^{t,k}]=0.
\]
Here $\sigma^2$ below is the variance scale of this already-formed stochastic
gradient; there is no batch-size control variable.

\paragraph{Probability conventions.} The proofs use only three facts about the
randomness.
\begin{enumerate}[label=(P\arabic*)]
\item Given $\F^{t,k}$, a newly drawn noise $\zeta_i^{t,k}$ has mean zero and
$\E[\norm{\zeta_i^{t,k}}_2^2\mid\F^{t,k}]\le\sigma^2$.
\item If $X$ is determined before $\zeta_i^{t,k}$ is drawn, then
$\E\ip X{\zeta_i^{t,k}}=0$: average over the new draw first and use (P1).
In particular, noises drawn at different steps are uncorrelated.
\item Selection is independent of the minibatches. Hence, for any $X_i$ computed
from client $i$'s round-start state and its own minibatches,
$\E[I_i^tX_i]=s\,\E[X_i]$ and $\E[\refreshmask_i^tX_i]=\refreshprob\,\E[X_i]$
(\cref{lem:mask-average}).
\end{enumerate}
A quantity is \emph{predictable} if it is determined by $\F^t$, that is, before
the round's samples are drawn. We call a newly drawn noise term an
\emph{innovation}, to distinguish it from the noise stored in a tracker across
rounds.

\paragraph{Words used in the proofs.} We call Adam's scaling matrix
$D=\diag((\sqrt v+\epsilon)^{-1})$ the \emph{preconditioner}, and a bound that
holds uniformly over a whole run, such as the coordinate bound $H_{0,p}$ that
defines the event $\mathcal H_p$ in~\eqref{eq:main-preconditioner-envelope},
an \emph{envelope}. A random quantity is \emph{centered} if its mean given the
past is zero; two quantities are \emph{orthogonal} if their inner product has
mean zero; an identity holds \emph{pathwise} if it holds on every outcome. A
\emph{finite-population} average is taken over clients sampled without
replacement. A \emph{transient} is a term that decays geometrically in $t$,
and the level that remains is the \emph{steady-state} level, or floor. The
\emph{small-gain conditions} are $\bar a_T\le1/8$ and
$s\bar a_T\le\refreshprob/2$. The \emph{reference comparison} is the argument
with $\chi_E=0$ used for Corollary~\ref{cor:fixed-gain-main} and
Theorem~\ref{thm:FA-NT}; for the Adam family it compares $D$ with the fixed
scaling $D_{\rm ref}=\epsilon^{-1}I$ (\cref{lem:adam-reference}). An
exponential moving average is abbreviated EMA. Young's inequality is
$\norm{a+b}^2\le(1+\vartheta)\norm a^2+(1+\vartheta^{-1})\norm b^2$ for
$\vartheta>0$. In Appendices~\ref{app:core-lemmas}--\ref{app:extensions},
$h$ is shorthand for the one-round gain $\bar a_T$ (for a fixed gain $a$,
$h=1-(1-a)^K$), and $g^t:=\nabla f(x^t)$ is the exact global gradient. A bar
over a state without a time index, as in $\overline{\cM}_{\rm GC}$, denotes
its average over $t<T$; a bar over a quantity with index $t$, as in $\bar d^t$
or $\bar\omega^t$, denotes an average over clients. Write
\[
a_T:=1-\beta_{1,T},\qquad \bar a_T:=1-(1-a_T)^K,
\qquad b_T:=1-\beta_{2,T}.
\]
Thus $a_T$ and $b_T$ are the first- and second-moment gains.
The local learning rate $\eta_{\ell,T}$ and server aggregation step
$\eta_{g,T}$ determine the effective global step and local displacement:
\begin{equation}
    \alpha_T:=K\eta_{g,T}\eta_{\ell,T},\qquad
    \delta_T:=K\eta_{\ell,T}
\label{eq:gain-map}
\end{equation}

Steps and gains may depend on the horizon but stay fixed during a $T$-round
run, which is why they carry the subscript $T$. Since
$\alpha_T=\eta_{g,T}\delta_T$, server aggregation can attenuate global
motion without shortening local paths.
Let $\ell_i^t$ be the number of local moment updates completed by client $i$
strictly before round $t$. The optional client-update-count bias correction is
\begin{equation}
\psi_i^{t,k}=1-(1-a_T)^{\ell_i^t+k}.
\label{eq:client-age-correction}
\end{equation}
Theorems~\ref{thm:FA-GC} and~\ref{thm:FA-NT} use $\psi_i^{t,k}=1$. The
symbol $p\in(0,1)$ denotes the failure probability of the event $\mathcal H_p$,
while $q\in\{1,2\}$ denotes the stationarity power.
\begin{table}[!t]
\caption{Consolidated notation used in the main text and proofs.}
\label{tab:notation}
\centering
\footnotesize
\setlength{\tabcolsep}{4pt}
\renewcommand{\arraystretch}{1.08}
\begin{tabularx}{0.98\linewidth}{p{0.12\linewidth}p{0.25\linewidth}X}
\toprule
Group & Notation & Description \\
\midrule
System
& $n,S,\widetilde S;s,\refreshprob$ & Client count, training and refresh counts, and participation fractions $S/n$ and $\widetilde S/n$.\\
& $T,K;t,k$ & Communication rounds, local steps, and their indices.\\
& $\mathcal S^t,\widetilde{\mathcal S}^t$ & Training and tracker-refresh client sets in round $t$.\\
& $f,f_i,f_*$ & Global objective, client objective, and finite lower bound.\\
& $\mathcal D_i,\xi_i$ & Client data distribution and a stochastic sample.\\
& $d;L,L_g$ & Number of scalar parameters; smoothness constants of $f$ and of each $\nabla f_i$ (Assumption~\ref{assump:genLoss}).\\
& $\Delta_0,\mathcal E_{\rm init}$ & Initial objective gap $f(x^0)-f_*$ and initial state scale (Section~\ref{sec:convergence}).\\
& $x^t,x_i^{t,k}$ & Global model and client $i$'s local model.\\
\midrule
Algorithm
& $g_i^{t,k},\widehat g_i^{t,k}$ & Raw stochastic gradient and its FA-GC-corrected version.\\
& $\zeta_i^{t,k},\sigma$ & Newly sampled gradient noise and its scale (standard deviation; sub-Gaussian scale for the Adam family).\\
& $m_i^{t,k},v_i^{t,k}$ & Persistent first- and second-moment states without bias correction.\\
& $y_i^t,y^t$ & Client tracker and its exact server average, $y^t=n^{-1}\sum_i y_i^t$.\\
& $\mathcal P_i^{t,k},d_i^{t,k}$ & Map from moments to direction, and the realized direction.\\
& $\eta_\ell,\eta_g;\alpha_T,\delta_T$ & Local and server steps (written $\eta_{\ell,T},\eta_{g,T}$ for a $T$-round run), with effective global step $K\eta_g\eta_\ell$ and local displacement scale $K\eta_\ell$.\\
& $\beta_1,\beta_2;a_T,b_T$ & Moment decays and gains $1-\beta_{1,T}$ and $1-\beta_{2,T}$.\\
& $\bar a_T,h;a,b$ & One-round momentum gain and its shorthand $h$; fixed first- and second-moment gains.\\
& $c_\star,R_{K,T}$ & Slower contraction rate $\min\{s\bar a_T,\refreshprob\}$ and ratio $Ka_T/\bar a_T$.\\
& $c_a,c_b,c_\eta,C_{s,\refreshprob}$ & Schedule constants of Theorems~\ref{thm:FA-GC} and~\ref{thm:FA-NT}.\\
& $\psi_i^{t,k},\ell_i^t$ & Bias-correction factor and client update count, \eqref{eq:client-age-correction}.\\
& $\epsilon,C_{2,\min}$ & Adam stabilizer and smallest $C_2$ in the reference comparison: $1/(2\epsilon)$ (Adam family), $1$ (normalization, exact Muon), $r^{-1/2}$ (finite-step Muon).\\
& $D,D_{\rm ref};\widehat v$ & Preconditioner $\diag((\sqrt v+\epsilon)^{-1})$ and fixed scaling $\epsilon^{-1}I$; AMSGrad running maximum.\\
& $p,p_T;H_{0,p},\lambda_p,c_{\rm rot}$ & Failure probability and its scheduled value $(T+1)^{-2}$; coordinate bound, lower bound on $D$, and SOAP rotation factor \eqref{eq:main-preconditioner-envelope}.\\
& C1--C3; $C_1,C_2,C_3,\chi_E,\nu$ & The direction properties of Lemmas~\ref{lem:C1-main}--\ref{lem:C23-main}, and their constants (Table~\ref{tab:constants_for_methods}).\\
& $\|\cdot\|_{\rm pr},\|\cdot\|_*,q$ & Update norm, matching gradient norm, and stationarity power.\\
& $G,B$ & Population-gradient bound when used and FA-NT heterogeneity bound.\\
& $N,P_{\rm M},P_R;r$ & Normalization, exact and finite-step Muon maps, with $r=\min(d_1,d_2)$.\\
& $p_{\rm NS};Q_L^t,Q_R^t$ & Newton--Schulz polynomial \eqref{eq:ns-iteration}; SOAP rotations.\\
\midrule
Proof states
& $\mathcal F^t,\mathcal F^{t,k};I_i^t,\refreshmask_i^t$ & Information before round $t$ and before step $k$; training and refresh indicators.\\
& $\gamma_i^{t,k}$ & Local point stored from client $i$'s latest tracker refresh.\\
& $\mathcal E_{\rm dr}^t,\Xi^t$ & Mean client drift and tracker lag.\\
& $g^t$ & Exact global gradient $\nabla f(x^t)$.\\
& $y_{i,\rm pop}^t,\omega_i^t,\Omega^t$ & Population tracker, tracker noise, and average tracker-noise energy.\\
& $u_i^t,\bar\omega^t$ & Tracker-noise contrast $\bar\omega^t-\omega_i^t$ (the noise part of the correction) and average tracker noise.\\
& $g_{i,\rm pop}^{t,k},y_{\rm pop}^t$ & Noise-free corrected input and average population tracker.\\
& $m_{i,\rm pop}^{t,k},e_i^{t,k}$ & Population momentum and its noise component.\\
& $Z_i^t,W_2$ & Fresh noise entering the momentum in one round, and its weight bound.\\
& $\theta_T,z_i^t,\mathfrak Z^t$ & Joint-state weight, joint state $e_i^t+\theta_Tu_i^t$, and its energy.\\
& $\Gamma_{s,\refreshprob},C_p$ & Noise factor $C_{\rm n}(1+s/\refreshprob)$ and population constant $9R_{K,T}$.\\
& $\mathcal E_{\rm pop}^t,\mathcal N^t,\mathcal N_{\rm loc}^t$ & Population-momentum error and round-boundary/local momentum-noise energies.\\
& $\mathcal R_v^t,\mathcal R_{v,T},\mathcal R_{v,\infty}$ & Denominator residual in the reference comparison, its average over $t<T$, and its limsup.\\
& $\bar d^t,d_\star^t$ & Average selected direction and comparison direction.\\
& $\mathcal M_{\rm GC}^t,E_{\rm mis}^t$ & Momentum error and the mismatch term it gives in C3.\\
& $\mathcal U^t,U_{\rm loc}^t;U_*,U_{{\rm loc},*}$ & Mean squared aggregated and local directions, and their suprema over $t<T$.\\
& $H_4$ & Fourth-moment scale, $H_4^4=G^4+d^2\sigma^4$ (\cref{lem:derived-fourth}).\\
& $\mathcal Q_T,\mathcal I_T,J_T$ & Movement coefficient, initial-and-noise term, and noise average of \cref{thm:lyapunov}.\\
& $\mathcal G_{q,T}$ & Averaged stationarity measure.\\
& $A_{\rm movement},A_{\rm local},A_v;K_q$ & Movement, drift, and second-moment terms of the FA-NT bound and their coefficient (explicit in the proof of Theorem~\ref{thm:FA-NT}).\\
& $p_i^{t,k},c_i^t,P_*$ & FA-NT direction before correction, correction $y^t-y_i^t$, and $P_*^2=\sup\E\norm{p_i^{t,k}}_{\rm pr}^2$.\\
& $A_{\rm init,NT},A_{\rm noise,NT},$ $A_{\rm participation,NT}$ & Initialization, noise, and participation terms of the FA-NT bound.\\
& $\mathcal E_{\rm pop,NT}^t,\mathcal N_{\rm NT}^t,\mathcal M_{\rm NT}^t$ & FA-NT population error, noise energy, and momentum error, measured from $\nabla f_i(x^t)$.\\
& $t_{\rm b}$ & Burn-in time of Corollary~\ref{cor:bias-corrected-momentum}.\\
& $H_{0,p},H_p$ & Bounds on raw and corrected gradient coordinates on the event $\mathcal H_p$.\\
\bottomrule
\end{tabularx}
\end{table}

\subsection{Exact persistent-state conventions}
\subsubsection{Pre-estimation correction (\FA)}
Initialize $m_i^0=0$ and $v_i^0=0$. For
normalization and Muon, choose any trackers $y_i^0$ whose errors have finite mean square.
For Adam, AMSGrad, and round-frozen SOAP, use the displayed fresh-gradient
initialization, deterministic zero, or an initialization discrepancy satisfying
the same sub-Gaussian condition as the gradient noise (which also gives the
fourth-moment bound of \cref{lem:derived-fourth}).
Define the initial discrepancies by
\begin{equation}
y_i^0=\nabla f_i(x^0)+\omega_i^0,\qquad
\Omega^0:=\frac1n\sum_i\E\norm{\omega_i^0}_2^2<\infty,\qquad
\ell_i^0=0,
\label{eq:buffer-initialization}
\end{equation}
and set $y^0=n^{-1}\sum_i y_i^0$. Full-gradient initialization is the
special case $\omega_i^0=0$; the displayed stochastic-gradient initialization
has centered $\omega_i^0$ with second moment at most $\sigma^2$; and
deterministic zero corresponds to $\omega_i^0=-\nabla f_i(x^0)$. No centering
of the initial discrepancy is used by the refresh or innovation recursions.
When stochastic initialization is used, its random draws are independent of
later client selections and minibatch draws.
For the Adam family, the sub-Gaussian condition on the initialization gives
the coordinate bound on $\mathcal H_p$ (\cref{lem:good-event}) and the fourth
moments of \cref{lem:derived-fourth}. The explicit rate in
Theorem~\ref{thm:FA-GC} uses fresh-gradient or deterministic zero trackers. Set
$\gamma_i^{0,k}=x^0$ for all $i,k$, so $\Xi^0=0$.
At round $t$:
\begin{enumerate}
\item Sample $\cS^t$ and $\widetilde\cS^t$ and broadcast $(x^t,y^t)$.
\item For each $i\in\cS^t$, set
\[
x_i^{t,1}=x^t,\qquad m_i^{t,0}=m_i^t,\qquad v_i^{t,0}=v_i^t.
\]
For $k=1,\ldots,K$, compute
\begin{align}
g_i^{t,k}&=\nabla f_i(x_i^{t,k})+\zeta_i^{t,k},\\
\widehat g_i^{t,k}&=g_i^{t,k}+y^t-y_i^t,\label{eq:fagc-corrected}\\
m_i^{t,k}&=(1-a_T)m_i^{t,k-1}+a_T\widehat g_i^{t,k},\\
v_i^{t,k}&=(1-b_T)v_i^{t,k-1}+b_T
               (\widehat g_i^{t,k})^{\odot 2},\label{eq:correct-v}\\
d_i^{t,k}&=\mathcal P_i^{t,k}
 \bigl(m_i^{t,k}/\psi_i^{t,k};v_i^{t,k}\bigr),\\
x_i^{t,k+1}&=x_i^{t,k}-\eta_{\ell,T}d_i^{t,k}.
\end{align}
For an optimizer without a second moment, omit \cref{eq:correct-v}.
For AMSGrad additionally initialize $\widehat v_i^0=v_i^0$ and update
coordinatewise
\begin{equation}
 \widehat v_i^{t,k}=\max\{\widehat v_i^{t,k-1},v_i^{t,k}\},
 \qquad \widehat v_i^{t,0}=\widehat v_i^t,
 \label{eq:amsgrad-max}
\end{equation}
using $\widehat v_i^{t,k}$ in the adaptive map.
\item A selected client stores
\begin{equation}
(m_i^{t+1},v_i^{t+1},\ell_i^{t+1})
=(m_i^{t,K},v_i^{t,K},\ell_i^t+K).
\label{eq:selected-persistence}
\end{equation}
 An unselected client retains $(m_i^t,v_i^t,\ell_i^t)$ and, for AMSGrad,
$\widehat v_i^t$. A selected AMSGrad client stores
$\widehat v_i^{t+1}=\widehat v_i^{t,K}$. Only
$i\in\widetilde\cS^t$ refreshes
$y_i^{t+1}=K^{-1}\sum_{k=1}^K g_i^{t,k}$; every other tracker is unchanged.
\item The server updates
\begin{align}
x^{t+1}&=x^t+\frac{\eta_{g,T}}{S}\sum_{i\in\cS^t}
             (x_i^{t,K+1}-x^t),\label{eq:server-model}\\
y^{t+1}&=y^t+\frac1n\sum_{i\in\widetilde\cS^t}
             (y_i^{t+1}-y_i^t).\label{eq:server-tracker}
\end{align}
\end{enumerate}

\subsubsection{Post-estimation tracking (\NT)}
For \NT, perform one adaptive warm-start pass on every client. Draw
$g_{i,\rm init}=\nabla f_i(x^0)+\zeta_{i,\rm init}$, start from
$m_{i,\rm pre}=0$, $v_{i,\rm pre}=0$, and make one ordinary moment
update
\[
 m_i^0=a_Tg_{i,\rm init},\qquad
 v_i^0=b_T(g_{i,\rm init})^{\odot 2},\qquad
 \ell_i^0=1.
\]
For a method without a second moment omit the $v$ update. For AMSGrad set
$\widehat v_i^0=v_i^0$.  For round-frozen
SOAP choose an orthogonal basis that depends only on $x^0$, rotate $g_{i,\rm init}$ before
the displayed $v$ update, and store $m_i^0$ in ambient coordinates. Define
$\psi_i^0=1$ for the primary variant without moment bias correction and $\psi_i^0=a_T$ for the
optional bias correction after one update. Store
\[
 y_i^0=\mathcal P_i^0(m_i^0/\psi_i^0;v_i^0),\qquad
 y^0=\frac1n\sum_i y_i^0.
\]
For AMSGrad, the second-state argument in this display is
$\widehat v_i^0$ rather than $v_i^0$.
The warm-start noise is retained in the initial moment state below. After this
pass, the moment recursions use $g_i^{t,k}$ rather than
$\widehat g_i^{t,k}$, with the same persistent-state rule
\eqref{eq:selected-persistence}. Let
\[
p_i^{t,k}=\mathcal P_i^{t,k}
 (m_i^{t,k}/\psi_i^{t,k};v_i^{t,k}),\qquad
d_i^{t,k}=p_i^{t,k}+y^t-y_i^t.
\]
The local tracker refresh is $y_i^{t+1}=K^{-1}\sum_k p_i^{t,k}$ for
$i\in\widetilde\cS^t$. Equations \eqref{eq:server-model} and
\eqref{eq:server-tracker} are unchanged. If
$\norm{p_i^{t,k}}_{\mathrm{pr}}\le C_1$, including the warm-start
direction, then convexity of a norm and the invariant
$y^t=n^{-1}\sum_i y_i^t$ give
\begin{equation}
\norm{d_i^{t,k}}_{\mathrm{pr}}\le 3C_1.
\label{eq:nt-three-c1}
\end{equation}
The same recursions permit zero initialization, or initializing a tracker
only when its client is first selected, whenever the
initial moment and direction-tracker energies are finite; these choices change
only the explicitly retained initial-state terms.  The all-client warm start is
one concrete way to guarantee bounded initial direction trackers for every
optimizer instance.
\begin{remark}[First moments with and without bias correction]
The main results set $\psi_i^{t,k}=1$.  Equation
\eqref{eq:client-age-correction} gives the optional bias-corrected version.
Because $\ell_i^t$ counts client updates rather than server rounds, it remains
correct under partial participation and needs no extra communication.  The
proof without moment bias correction retains initialization as an exponentially decaying
participation transient.
\end{remark}
\begin{remark}[Communication cost]
Persistence adds client-side state but no synchronized tensor: $m_i^t$ and
$v_i^t$ are read and written only by client $i$. The communicated objects
remain the model update and the already-present tracker update. This
differs from a server-momentum construction that would broadcast or aggregate
an additional state.
\end{remark}
\begin{remark}[The \NT\ tracker]
For \NT, $y_i^t$ stores old adaptive directions rather than old gradients, so
the point-valued lag $\Xi^t$ is not inserted into its raw-moment recursion.
By the invariant $y^t=n^{-1}\sum_i y_i^t$, the sampled correction has zero
conditional mean over uniform training selection; its finite-population mean square
enters the bound $\E\norm{\bar d^t}^2\le10P_*^2$ of \cref{lem:nt-tracker-energy},
where $P_*^2:=\sup_{t,i,k}\E\norm{p_i^{t,k}}_{\rm pr}^2$,
and is zero at full participation. This
is separate from the $B$ term retained by the comparison bound when
correction follows the nonlinear map.
\end{remark}

\subsection{Detailed forms of the main assumptions}
The following paragraphs record the precise conventions used by the proofs
under Assumptions~\ref{assump:genLoss}--\ref{assump:BDH}. They retain those
main-text identities rather than introducing a second set of assumptions.

\paragraph{Smoothness in the update norm (Assumption~\ref{assump:genLoss}).}
For every $x,\Delta\in\cX$,
\[
f(x+\Delta)\le f(x)+\ip{\nabla f(x)}\Delta+\frac L2\norm \Delta_{\mathrm{pr}}^2.
\]
For Euclidean/Frobenius geometry this follows from the usual $L$-Lipschitz
gradient assumption. For Muon, spectral-norm smoothness is paired with the
nuclear norm as the matching gradient norm. If only Frobenius smoothness $L_F$ is known and
$r=\min(d_1,d_2)$, then $\norm{H}_F^2\le r\norm{H}_{\mathrm{sp}}^2$, so one may
take $L=rL_F$.
In the moment recursions we also use
$\norm{\nabla f_i(x)-\nabla f_i(z)}_2\le L_g\norm{x-z}_{\rm pr}$.

\paragraph{Bounded population gradients (Assumption~\ref{assump:BG}).}
For every client and point, $\norm{\nabla f_i(x)}_2\le G$. For matrices the
Frobenius norm is used. This does not bound stochastic gradients.  This
assumption is invoked only for the Adam/AMSGrad/SOAP fourth-moment and
denominator arguments; the normalization and Muon bounds and their
$q=1$ convergence proof do not use it.

\paragraph{Gradient noise (Assumptions~\ref{assump:genLoss} and~\ref{assump:BG}).}
For every method, $\E[\zeta_i^{t,k}\mid\F^{t,k}]=0$ and
$\E[\norm{\zeta_i^{t,k}}_2^2\mid\F^{t,k}]\le\sigma^2$. For Adam, AMSGrad, and
round-frozen SOAP, Assumption~\ref{assump:BG} adds sub-Gaussian tails: after
vectorizing a matrix when necessary, for every unit vector $u$ and all
$\lambda\in\R$,
\[
\E\!\left[\exp(\lambda\ip{u}{\zeta_i^{t,k}})\mid\F^{t,k}\right]
\le \exp\!\left(\frac{\lambda^2\sigma^2}{2}\right).
\]
These tails are used only for the event $\mathcal H_p$ of
\cref{lem:good-event} and the fourth moments of \cref{lem:derived-fourth}.
Noises drawn at the same step are conditionally independent across clients,
and each new noise has mean zero given all earlier draws (P1--P2). The same bounds hold
for fresh tracker-initialization and warm-start noise, which are independent
across clients. Selection indicators are sampled before, and are conditionally
independent of, newly sampled current-round noises. The noise stored in a
tracker may depend on $x^t$, since earlier gradients shaped both; the proofs
never treat them as independent. Fourth moments needed by the denominator estimate are
derived below rather than imposed as bounded stochastic gradients.

\paragraph{Positive Adam-family stabilizer (Assumption~\ref{assump:BG}).}
Adam, AMSGrad, and round-frozen SOAP use a fixed
$\epsilon>0$.  The convergence and movement bounds for these methods are not
claimed at $\epsilon=0$.

\paragraph{FA-NT heterogeneity (Assumption~\ref{assump:BDH}).}
For \NT, $n^{-1}\sum_i\norm{\nabla f_i(x)-\nabla f(x)}_2^2\le B^2$.
No $B$ is used in the pre-estimation mean-error bound in \cref{lem:fagc-mean}.

\section{Core Lemmas and Descent Bounds}
\label{app:core-lemmas}
The optimizer inequalities in Lemmas~\ref{lem:C1-main} and~\ref{lem:C23-main}
reduce the convergence analysis to controlling movement and the mismatch
between the adaptive direction and the current global gradient. That mismatch
persists across rounds: local models move, stored gradients become stale,
and momentum reuses their noise. The core results below bound these sources
of error and show how they enter the loss decrease.

The argument has three stages. Appendices~\ref{app:core-drift}--\ref{app:core-momentum}
control drift, tracking, and momentum for FA-GC.
Appendices~\ref{app:core-denominator} and~\ref{app:core-nt} supply the additional
bounds for adaptive denominators and the FA-NT comparison, respectively.
Appendix~\ref{app:core-descent} combines these controls with descent, providing
the inequalities used in the main proofs in Appendix~\ref{app:main-results}.
For each deferred proof, a short proof idea accompanies the statement and the
result heading links to the full derivation in Appendix~\ref{app:core-proofs}.
Auxiliary calculations are
collected in Appendix~\ref{app:supporting}.

\subsection{Client drift and selective tracking}
\label{app:core-drift}
A corrected local gradient still differs from the current global gradient
because local models have moved and stored gradients were evaluated at older
points. We therefore bound both client drift and tracking lag. Local and
server movement enter separately, while tracker refresh contracts the lag.
Together, these bounds supply the input errors in Appendix~\ref{app:core-momentum}.
Define
\begin{align}
\cD^t&:=\frac1{SK}\sum_{i\in\cS^t}\sum_{k=1}^K
        \norm{x_i^{t,k}-x^t}_{\mathrm{pr}}^2,\label{eq:client-drift}\\
\Xi^t&:=\frac1{nK}\sum_{i=1}^n\sum_{k=1}^K
        \E\norm{\gamma_i^{t,k}-x^t}_{\mathrm{pr}}^2.\label{eq:tracking-lag}
\end{align}
Here $\gamma_i^{t,k}$ is replaced by $x_i^{t,k}$ when $i$ refreshes at round
$t$ and otherwise is retained.

\begin{lemma}[Client drift for each realization, \proofref{app:proof-client-drift}]\label{lem:client-drift}
If $\norm{d_i^{t,k}}_{\mathrm{pr}}\le D_1$, then
\begin{equation}
\cD^t\le\frac{(K-1)(2K-1)}{6}\eta_{\ell,T}^2D_1^2
\le\frac{K^2}{3}\eta_{\ell,T}^2D_1^2,\qquad
\norm{x^{t+1}-x^t}_{\mathrm{pr}}\le\alpha_TD_1.
\label{eq:drift-bound}
\end{equation}
\end{lemma}

\noindent\textit{Proof idea.} Summing the bounded local steps gives the
client-displacement estimate. The server update averages these displacements,
so the same estimate also bounds global movement.\par

\begin{lemma}[Unconditional client drift, \proofref{app:proof-expected-client-drift}]\label{lem:expected-client-drift}
Let
$U_{\rm loc}^t:=\E[(SK)^{-1}\sum_{i\in\cS^t,k}
\norm{d_i^{t,k}}_{\rm pr}^2]$. Then
\begin{equation}
\E\cD^t\le \delta_T^2U_{\rm loc}^t,\qquad
\E\norm{x^{t+1}-x^t}_{\rm pr}^2=\alpha_T^2\cU^t
\le\alpha_T^2U_{\rm loc}^t.
\label{eq:expected-drift}
\end{equation}
\end{lemma}

\noindent\textit{Proof idea.} A local displacement is a partial sum of
directions. We first apply the squared-norm inequality for a sum and then
average over the selected clients and steps to obtain the unconditional
drift estimate. Jensen's inequality then transfers this estimate to the
server update.\par

\begin{remark}[Local and aggregated mean squared directions]
The two quantities cannot be interchanged. For two clients with directions $u$
and $-u$ at every local step, $\cU^t=0$ while
$U_{\rm loc}^t=\norm{u}_{\rm pr}^2$. Both clients can travel
away from $x^t$ even though their averaged server direction cancels. This is
why client drift is controlled by $\delta_T^2U_{\rm loc}^t$ rather than by
$\alpha_T^2\cU^t$; the distinction remains essential when $\eta_{g,T}\ne1$.
\end{remark}

\begin{lemma}[Selective-tracking contraction, \proofref{app:proof-tracking-contraction}]\label{lem:tracking-contraction}
For $0<\refreshprob=\widetilde S/n<1$,
\begin{equation}
\Xi^{t+1}\le\left(1-\frac\refreshprob2\right)\Xi^t
+2\refreshprob\,\E\cD^t+\frac3\refreshprob\E\norm{x^{t+1}-x^t}_{\mathrm{pr}}^2.
\label{eq:tracking-contraction}
\end{equation}
If $\refreshprob=1$, the sharper bound
$\Xi^{t+1}\le2\E\cD^t+2\E\norm{x^{t+1}-x^t}_{\mathrm{pr}}^2$ holds.
\end{lemma}

\noindent\textit{Proof idea.} Only a refresh replaces the point at which a
client's tracker was evaluated. Averaging the refreshed and stale cases
therefore contracts the old lag. However, local drift and server movement
introduce new error, which the remaining terms in the recursion control.\par

\begin{remark}[Why $S$ cannot replace $\widetilde S$]
The contraction is generated by replacement of $\gamma_i^{t,k}$, so its
probability is $\widetilde S/n$. Training without refreshing does not contract
the lag.
\end{remark}

\subsection{Tracker noise and reuse of sampled gradients}
\label{app:core-noise}
The same gradient sample can enter momentum immediately and return through a
stored tracker in later rounds. To control this correlated noise, we separate
the population and stochastic parts of the tracker and momentum. The resulting
momentum-noise bound supplies the stochastic term in the main rate.
A separate comparison handles fixed first-moment gains.
Write $g^t:=\nabla f(x^t)$ and decompose every stored \FA\ tracker as
\begin{equation}
y_i^t=y_{i,\rm pop}^t+\omega_i^t,\qquad
y_{i,\rm pop}^t:=\frac1K\sum_{k=1}^K\nabla f_i(\gamma_i^{t,k}),\qquad
u_i^t:=\bar\omega^t-\omega_i^t\ \text{(the tracker-noise contrast)},\quad
\bar\omega^t:=\frac1n\sum_i\omega_i^t .
\label{eq:stored-noise}
\end{equation}
Let $\bar\zeta_i^t:=K^{-1}\sum_{k=1}^K\zeta_i^{t,k}$ on a refresh
round, with counterfactual draws used only for notation.  The implemented
reuse of the local gradients is exactly
\begin{equation}
\omega_i^{t+1}=(1-\refreshmask_i^t)\omega_i^t+\refreshmask_i^t\bar\zeta_i^t.
\label{eq:buffer-noise-recursion}
\end{equation}
Define
\begin{equation}
\Omega^t:=\frac1n\sum_i\E\norm{\omega_i^t}_2^2.
\label{eq:buffer-energy}
\end{equation}

\begin{lemma}[Tracker energy and refresh fluctuation, \proofref{app:proof-buffer-energy}]
\label{lem:buffer-energy}
For every $t$,
\begin{align}
\Omega^{t+1}&\le(1-\refreshprob)\Omega^t+\frac{\refreshprob\sigma^2}{K},
\label{eq:buffer-energy-recursion}\\
\frac1n\sum_i\E\norm{u_i^t}_2^2&\le\Omega^t.
\label{eq:contrast-energy}
\end{align}
Moreover
\[
\rho_i^{t+1}:=u_i^{t+1}-(1-\refreshprob)u_i^t
\]
is centered conditional on $\F^t$ and satisfies
\begin{equation}
\frac1n\sum_i\E[\norm{\rho_i^{t+1}}_2^2\mid\F^t]
\le \refreshprob\left\{\frac1n\sum_j\norm{\omega_j^t}_2^2+
\frac{\sigma^2}{K}\right\}.
\label{eq:refresh-martingale-variance}
\end{equation}
\end{lemma}

\noindent\textit{Proof idea.} On a refresh round, the tracker contains the
average of the newly sampled local gradient noises. However, when the device is not selected
for refresh, the tracker retains its previous noise. This replacement rule
thus determines the energy recursion, while centering over clients yields
the contrast and refresh-variance bounds.\par

For the stochastic split, define the population-corrected input
\[
g_{i,\rm pop}^{t,k}:=\nabla f_i(x_i^{t,k})+y_{\rm pop}^t-y_{i,\rm pop}^t,\qquad
y_{\rm pop}^t:=\frac1n\sum_jy_{j,\rm pop}^t,
\]
and the population momentum, on the same realized local path, by
\[
m_{i,\rm pop}^{t,k}=(1-a_T)m_{i,\rm pop}^{t,k-1}
+a_Tg_{i,\rm pop}^{t,k},\qquad
m_{i,\rm pop}^{t,0}=m_{i,\rm pop}^t.
\]
Initialize $m_{i,\rm pop}^0=m_i^0=0$. It has the same selected/unselected
persistence rule as $m_i$.  Set
$e_i^{t,k}:=m_i^{t,k}-m_{i,\rm pop}^{t,k}$.  Then, exactly,
\begin{equation}
e_i^{t,k}=(1-a_T)e_i^{t,k-1}+a_T(\zeta_i^{t,k}+u_i^t).
\label{eq:exact-noise-filter}
\end{equation}
Let
\[
w_k:=a_T(1-a_T)^{K-k},\qquad
Z_i^t:=\sum_{k=1}^Kw_k\zeta_i^{t,k},\qquad
W_2:=\sum_kw_k^2.
\]
Thus $\sum_kw_k=\bar a_T$ and
\begin{equation}
W_2=\frac{a_T\{1-(1-a_T)^{2K}\}}{2-a_T}
\le 2a_T\bar a_T.
\label{eq:local-ema-square-weight}
\end{equation}
At round boundaries,
\begin{equation}
e_i^{t+1}=(1-\bar a_TI_i^t)e_i^t
+I_i^t(Z_i^t+\bar a_Tu_i^t).
\label{eq:round-noise-filter}
\end{equation}

\begin{lemma}[Momentum noise under gradient reuse, \proofref{app:proof-innovation-influence}]
\label{lem:innovation-influence}
Suppose
\[
0<\bar a_T\le\frac18,\qquad s\bar a_T\le\frac\refreshprob2,
\]
and define
\[
 \cN^t:=\frac1n\sum_i\E\norm{e_i^t}_2^2,\qquad
 \cN_{\rm loc}^t:=\E\!\left[\frac1{SK}\sum_{i\in\cS^t,k}
 \norm{e_i^{t,k}}_2^2\right],\qquad
\theta_T:=\frac{s\bar a_T}{\refreshprob-s\bar a_T}.
\]
There are numerical $c_{\rm n},C_{\rm n},c_\omega,C_\omega>0$ such that
\begin{align}
\cN^t\le{}&
 C_{\rm n} e^{-c_{\rm n}s\bar a_Tt}\{\cN^0+\theta_T^2\Omega^0\}
 +C_\omega e^{-c_\omega\min\{s\bar a_T,\refreshprob\}t}\Omega^0
+\Gamma_{s,\refreshprob}a_T\sigma^2,
\label{eq:innovation-influence}\\
\Gamma_{s,\refreshprob}&:=C_{\rm n}\left(1+\frac{s}{\refreshprob}\right).
\label{eq:innovation-factor}
\end{align}
Moreover, for every $a_T\in(0,1)$, without the two conditions above, the
momentum noise at the intermediate local steps satisfies the explicit comparison
\begin{equation}
 \cN_{\rm loc}^t
 \le3\cN^t+3\bar a_T^2\Omega^t+3a_T\sigma^2.
 \label{eq:local-innovation-comparison}
\end{equation}
The $\Omega^t$ term is the noise stored in the trackers;
\cref{lem:buffer-energy} bounds it separately.
\end{lemma}

\noindent\textit{Proof idea.} Gradient reuse makes the direct and later tracker
contributions of one sample correlated. We therefore follow the joint state
$z_i^t=e_i^t+\theta_Tu_i^t$, whose conditional mean contracts by exactly
$1-s\bar a_T$ per round. Its centered increment has three correlated pieces,
which we bound with $\norm{x_1+x_2+x_3}^2\le3\sum_j\norm{x_j}^2$ instead of
assuming orthogonality. Summing the resulting geometric recursion yields both
the transient and the steady-state terms.\par

\paragraph{Fixed first-moment gain.}
The fixed-gain corollary uses the following comparison in place of
Lemma~\ref{lem:innovation-influence}, which needs the small-gain conditions.
\begin{lemma}[Two-state contraction for arbitrary fixed first-moment gain, \proofref{app:proof-fixed-gain-two-state}]
\label{lem:fixed-gain-two-state}
Let $a\in(0,1)$ be fixed and put $h=1-(1-a)^K\in(0,1)$. At round
boundaries the momentum-noise and tracker-noise energies obey
\begin{equation}
\begin{aligned}
\cN^{t+1}&\le(1-sh)\cN^t+sh\,\Omega^t+2sah\sigma^2,\\
\Omega^{t+1}&\le(1-\refreshprob)\Omega^t+\refreshprob\sigma^2/K.
\end{aligned}
\label{eq:fixed-two-state}
\end{equation}
The initial values $\cN^0$ and $\Omega^0$ therefore contribute at most a
constant times $\max\{1-sh,1-\refreshprob\}^t$ when $sh\ne\refreshprob$.
Consequently
\begin{equation}
\limsup_t\Omega^t\le\frac{\sigma^2}{K},\qquad
\limsup_t\cN^t\le\left(2a+\frac1K\right)\sigma^2,
\label{eq:fixed-two-state-floor}
\end{equation}
and, by~\eqref{eq:local-innovation-comparison},
$\limsup_t\cN_{\rm loc}^t\le(9a+6/K)\sigma^2$.
\end{lemma}

\noindent\textit{Proof idea.} Convexity bounds the momentum energy by a convex
combination of the old momentum and tracker energies, plus fresh noise.
Coupling this estimate with the tracker recursion therefore gives a triangular
comparison, from which we obtain the limiting noise floor.\par

\subsection{Correction placement and population momentum}
\label{app:core-momentum}
We now connect the drift, tracking, and noise bounds to the momentum mismatch
in Lemma~\ref{lem:C23-main}. The first lemma identifies the cancellation from
placing correction before adaptation; the second propagates the remaining
population error through persistent momentum and adds back the stochastic
part controlled in Appendix~\ref{app:core-noise}.
\begin{lemma}[Averaged correction-placement decomposition, \proofref{app:proof-fagc-mean}]
\label{lem:fagc-mean}
For \FA,
\begin{equation}
\E\left[\frac1{SK}\sum_{i\in\cS^t,k}
\norm{g_{i,\rm pop}^{t,k}-g^t}_2^2\right]
\le 6L_g^2(\E\cD^t+\Xi^t).
\label{eq:fagc-decomp}
\end{equation}
No heterogeneity or stochastic term occurs. For \NT, the moments receive raw gradients, so the analysis instead compares each momentum with its own client's gradient $\nabla f_i(x^t)$. Heterogeneity enters only when that comparison is converted to the global gradient in \cref{prop:fant-bridge}.
\end{lemma}

\noindent\textit{Proof idea.} The cancellation appears when we add and subtract
each client's gradient at the current server model. FA-GC's correction then
removes the client-specific offset, leaving drift and tracking errors.\par

Define
\begin{equation}
\cP^t:=\frac1n\sum_i\E\norm{m_{i,\rm pop}^t-g^t}_2^2.
\label{eq:population-moment-state}
\end{equation}
Let $\cM_{\rm GC}^t$ denote the selected client-step average of
$\E\norm{m_i^{t,k}-g^t}_2^2$.

\begin{lemma}[Population momentum and actual local mismatch, \proofref{app:proof-population-momentum}]
\label{lem:population-momentum}
If $0<a_T<1$ (equivalently $0<\bar a_T<1$), define
\begin{equation}
 R_{K,T}:=\frac{Ka_T}{\bar a_T}.
 \label{eq:round-ema-ratio}
\end{equation}
Then one may take $c_p=1/2$ and $C_p=9R_{K,T}$ in
\begin{align}
\cP^{t+1}\le{}&
(1-c_ps\bar a_T)\cP^t
+C_ps\bar a_TL_g^2(\E\cD^t+\Xi^t)
+\frac{C_pL_g^2\alpha_T^2\cU^t}{s\bar a_T},
\label{eq:population-momentum-recursion}\\
\cM_{\rm GC}^t\le{}&
2\cP^t+12L_g^2(\E\cD^t+\Xi^t)+2\cN_{\rm loc}^t.
\label{eq:local-moment-error}
\end{align}
\end{lemma}

\noindent\textit{Proof idea.} First write population momentum as a weighted
average of its old state and the corrected inputs. Since the global gradient
also changes between rounds, we bound this movement before adding
the noise component back. The resulting estimate controls the actual local
momentum mismatch.\par

\subsection{Persistent second-moment variation}
\label{app:core-denominator}
The FA-NT comparison and the conservative fixed-gain bound use a reference
denominator with $\chi_E=0$. This subsection bounds its departure from
initialization. The proof of Theorem~\ref{thm:FA-GC}(i) instead uses the positive
movement term and does not require this residual. Its fourth-moment input is proved in
Appendix~\ref{app:fourth-moments}.
Here $H_4^4:=G^4+d^2\sigma^4$ is the fourth-moment scale of
\cref{lem:derived-fourth}.

Let
$D(v)=\diag((\sqrt v+\epsilon)^{-1})$, $D_{\rm ref}=\epsilon^{-1}I$, and define the denominator residual
\begin{equation}
\cR_{v,T}:=\frac1T\sum_{t<T}\E\left[
 \frac1{SK}\sum_{i\in\cS^t,k}
 \norm{[D(v_i^{t,k})-D_{\rm ref}]m_i^{t,k}}_2^2\right].
\label{eq:denominator-residual}
\end{equation}

\begin{lemma}[Cumulative second-moment variation, \proofref{app:proof-v-variation}]\label{lem:v-variation}
A client has completed $\ell_i^t+k\le KT$ second-moment updates through
local step $(t,k)$, including that step. Let $\tau_T:=KTb_T$.
Under \cref{lem:derived-fourth}, there is a
numerical constant $C_v>0$ such that
\begin{equation}
\cR_{v,T}\le
 \frac{C_v\tau_T H_4^4}{\epsilon^4},\label{eq:v-cumulative-residual}
\end{equation}
with $H_4$ as in \cref{lem:derived-fourth}.  The same statements hold
with $v$ replaced by the persistent AMSGrad maximum $\widehat v$.
For \NT's one-update warm start replace $KT$ by $KT+1$ and
$\tau_T$ by $(KT+1)b_T$; this does not change any scheduled exponent.
\end{lemma}

\noindent\textit{Proof idea.} The denominator changes throughout the run, so
we need to control its cumulative departure from the reference value.
Unrolling the second-moment recursion and applying the fourth-moment bound
bounds the second moment. The difference between the stabilized denominator
and $\epsilon^{-1}$ is at most $\sqrt v/\epsilon^2$, which transfers this
estimate to the direction residual. For AMSGrad, the same bound also holds for
the running maximum.\par

\begin{corollary}[Scheduled and fixed second moments]
\label{cor:scheduled-beta2}
For $b_T=\Theta((KT^2)^{-1})$, $\tau_T=\Theta(T^{-1})$ and hence
$\cR_{v,T}=O(T^{-1})$.  For fixed $b>0$, the reference comparison retains the
persistent residual $\cR_{v,\infty}:=\limsup_{T\to\infty}\cR_{v,T}$; this is the denominator/state term in
the Adam-family fixed-gain neighborhood~\eqref{eq:fixed-q2}.
\end{corollary}

\subsection{FA-NT direction and correction energy}
\label{app:core-nt}
FA-NT adds its correction after the adaptive map, so a bound on the local
adaptive direction alone does not yet bound the server's movement. The next
lemma controls the added correction and hence the complete direction before
the FA-NT proof applies smoothness.
\begin{lemma}[\NT\ tracker and correction energy, \proofref{app:proof-nt-tracker-energy}]\label{lem:nt-tracker-energy}
Suppose
\[
 P_*^2:=\sup_{t,i,k}\E\norm{p_i^{t,k}}_{\rm pr}^2<\infty,
\]
including the warm-start direction. Because every $y_i^t$ is either the
warm-start direction or an average of $K$ historical $p_i$ directions,
\begin{align}
 \sup_t\frac1n\sum_i\E\norm{y_i^t-y^t}_{\rm pr}^2&\le4P_*^2,
 \label{eq:nt-tracker-energy}\\
 \sup_t\E\left[\frac1{SK}\sum_{i\in\cS^t,k}
 \norm{p_i^{t,k}+y^t-y_i^t}_{\rm pr}^2\right]&\le10P_*^2.
 \label{eq:nt-direction-energy}
\end{align}
For matrix optimizers these bounds may first be taken in Frobenius norm and
then passed to spectral norm.
\end{lemma}

\noindent\textit{Proof idea.} Since each stored tracker averages historical
directions, Jensen's inequality bounds its second moment. We can therefore
use squared-norm estimates to control both the correction and the complete
corrected direction.\par

\subsection{Descent and sums of state errors}
\label{app:core-descent}
The master inequality retains the movement term from C3. The next result
then sums the common federated state recursions without replacing movement
by a constant. For Theorem~\ref{thm:FA-GC}(i) ($q=2$), the favorable movement term controls the
movement-induced errors. The proof of Theorem~\ref{thm:FA-GC}(ii) ($q=1$) instead combines the same
recursions with bounded movement.
\begin{theorem}[Master inequality, \proofref{app:proof-master}]\label{thm:master}
Suppose Assumption~\ref{assump:genLoss} holds and the directions satisfy
C1--C3 (Lemmas~\ref{lem:C1-main} and~\ref{lem:C23-main}), with constant
server step $\alpha_T$ and constant $C_2>0$. Then
\begin{align}
C_2\,\mathcal G_{q,T}
\le{}&\frac{f(x^0)-f_*}{T\alpha_T}
+\frac1T\sum_{t<T}C_3\E E_{\rm mis}^t+\nu\notag\\
&+\frac1T\sum_{t<T}\left\{
\frac L2\alpha_T\cU^t-\chi_EU_{\rm loc}^t\right\}.
\label{eq:master}
\end{align}
\end{theorem}

\noindent\textit{Proof idea.} Insert the comparison direction $d_\star^t$ into the
smoothness inequality and apply C2--C3. Summing over rounds telescopes the
objective. The local movement term remains available to control the state
errors, while Jensen bounds the aggregated movement by local movement.\par

\Needspace{24\baselineskip}
\begin{theorem}[Summed state errors and descent for $q=2$, \proofref{app:proof-lyapunov}]
\label{thm:lyapunov}
Under Assumption~\ref{assump:genLoss} and the FA-GC state convention, fix
horizon parameters with $0<a_T<1$, write $h=\bar a_T$, and let a bar denote
the average over $t<T$. Set $C_p=9Ka_T/h$ and
\begin{equation}
\mathcal Q_T=L_g^2\left[(4C_p+12)
\left(5\delta_T^2+\frac{6\alpha_T^2}{\refreshprob^2}\right)
+\frac{4C_p\alpha_T^2}{s^2h^2}\right].
\label{eq:state-movement-coefficient}
\end{equation}
\begin{enumerate}[label=(\alph*),leftmargin=*]
\item For every $a_T\in(0,1)$,
\begin{equation}
\overline{\cM}_{\rm GC}
\le\frac{4\cP^0}{Tsh}
+\frac{2(4C_p+12)L_g^2\Xi^0}{T\refreshprob}
+2\overline{\cN}_{\rm loc}+\mathcal Q_T\overline U_{\rm loc}.
\label{eq:state-ledger-general}
\end{equation}
\item Suppose in addition that $h\le1/8$ and $sh\le\refreshprob/2$. Set
\begin{align}
\theta_T&=\frac{sh}{\refreshprob-sh},\qquad
\mathfrak Z^0=\frac1n\sum_i\E\norm{e_i^0+\theta_T u_i^0}^2,\notag\\
C_{\Omega,T}&=144h+72sh/\refreshprob+6\theta_T^2+3h^2,\notag\\
J_T&=\frac{6\mathfrak Z^0}{Tsh}
+\frac{C_{\Omega,T}\Omega^0}{T\refreshprob}
+\left(\frac{C_{\Omega,T}}K+36a_T+
\frac{72sh}{\refreshprob K}+3a_Th\right)\sigma^2.
\label{eq:explicit-innovation-average}
\end{align}
Then $\overline{\cN}_{\rm loc}\le J_T$, and
\begin{equation}
\overline{\cM}_{\rm GC}
\le\mathcal I_T+\mathcal Q_T\overline U_{\rm loc},
\label{eq:state-movement-ledger}
\end{equation}
where
\begin{equation}
\mathcal I_T=\frac{4\cP^0}{Tsh}
+\frac{2(4C_p+12)L_g^2\Xi^0}{T\refreshprob}+2J_T.
\label{eq:state-initial-noise}
\end{equation}
For a $q=2$ instance of C1--C3 with $C_2=c>0$
and $E_{\rm mis}^t=\cM_{\rm GC}^t$,
\begin{equation}
c\mathcal G_{2,T}
+(\chi_E-L\alpha_T/2-C_3\mathcal Q_T)\overline U_{\rm loc}
\le\frac{\Delta_0}{T\alpha_T}+C_3\mathcal I_T+\nu.
\label{eq:lyap-final}
\end{equation}
If the movement coefficient is nonnegative, it may be dropped.
\end{enumerate}
In the noiseless case with exact initial trackers, $\overline{\cN}_{\rm loc}=0$
in~\eqref{eq:state-ledger-general}.
\end{theorem}

\noindent\textit{Proof idea.} Sum the tracking and population-momentum
recursions while keeping local and aggregated movement separate; this needs
only $0<a_T<1$. The reused innovation bound, which uses the two conditions in
part~(b), supplies $J_T$. Substitution into the master inequality
then exposes the coefficient that must remain nonnegative for absorption.\par

\subsection{Residual terms for the reference comparison}
\label{app:reference-residuals}
The FA-NT theorem uses the following error terms for movement and
second-moment evolution; its proof adds initialization, noise, and
participation terms of the same form. Theorem~\ref{thm:FA-GC}(i) instead uses~\eqref{eq:lyap-final}.
Finite Young and norm-conversion factors are collected in a horizon-independent $K_q$, whose values are given in the proof of Theorem~\ref{thm:FA-NT}. For these reference comparisons, $C_{2,\min}=1/(2\epsilon)$ for the Adam family, $C_{2,\min}=1$ for normalization and exact Muon, and $C_{2,\min}=r^{-1/2}$ for finite-step Muon. Set $c_\star=\min\{s\bar a_T,\refreshprob\}$. The reference comparisons have no tail term, $\nu=0$.
Let $U_*:=\sup_{t<T}\mathcal U^t$ and $U_{{\rm loc},*}:=\sup_{t<T}U_{\rm loc}^t$, and define
\begin{equation}
A_{\rm movement}:=K_q\alpha_TU_*,\qquad
A_{\rm local}:=K_q(\delta_T^2U_{{\rm loc},*})^{q/2},\qquad
A_v:=K_q\mathcal R_{v,T}.
\label{eq:main-grouped-factors}
\end{equation}

The normalization and Newton--Schulz bounds in Lemmas~\ref{lem:normalization} and~\ref{lem:muon-inexact} also have $\nu=0$, with the finite-step rank factor kept in $C_2$. Methods without a second-moment state have no denominator residual, so $A_v=0$ by the definition in~\eqref{eq:denominator-residual}. For the reference-denominator comparison used in the FA-NT and fixed-gain bounds, Lemma~\ref{lem:derived-fourth} bounds the fourth moments by a numerical multiple of $H_4^4=G^4+d^2\sigma^4$, using bounded population gradients and sub-Gaussian noise. Combining this envelope with Lemma~\ref{lem:v-variation} gives
\begin{equation}
\mathcal R_{v,T}\le
\frac{C_vH_4^4}{\epsilon^4}KTb_T.
\label{eq:main-v-explicit}
\end{equation}

\section{Proofs of the Main Convergence Results}
\label{app:main-results}
Each subsection proves one result from Section~\ref{sec:convergence}.
The core bounds are stated in Appendix~\ref{app:core-lemmas}; their derivations
and the optimizer calculations follow in
Appendices~\ref{app:core-proofs} and~\ref{sec:optimizer-instantiations}.

\subsection{Proof of Theorem~\ref{thm:FA-GC}: FA-GC convergence}
\label{app:proof-fagc}
The proof uses one C1--C3 master inequality and the common state recursions.
Part~(i) ($q=2$) retains movement until it is absorbed. Part~(ii) ($q=1$)
uses C1 to bound movement and then sums the square-root state errors.
\begin{proof}[Proof of Theorem~\ref{thm:FA-GC}]
\par\smallskip\noindent\textbf{Step 1: Insert the optimizer inequalities.}\quad
Proposition~\ref{prop:optimizer-bridge} gives all instances of C1--C3.
For Adam/AMSGrad and the stated SOAP convention,
\[
C_2=\lambda_{p_T}/2,\quad C_3=(2\epsilon)^{-1},\quad
\chi_E=\epsilon/2,\quad \nu=\lambda_{p_T}G^2p_T/2,
\]
with $p_T=(T+1)^{-2}$ and $\lambda_p=(3c_{\rm rot}H_{0,p}+\epsilon)^{-1}$
from~\eqref{eq:main-preconditioner-envelope}.
The comparison direction is fixed before the round's samples are drawn, although the actual preconditioner depends on
current gradients. The tail inequality~\eqref{eq:positive-preconditioner-tail} holds on every outcome before expectations
are taken. For normalization and Muon, $\chi_E=0$, $\nu=0$, and the
mismatch is bounded by a geometry-dependent multiple of
$\sqrt{\cM_{\rm GC}^t}$.

\par\smallskip\noindent\textbf{Step 2: Sum the common states.}\quad
Use the drift, tracking, population-momentum, and reused-noise bounds
in~\cref{thm:lyapunov}. Part~(b) applies because $h\le Ka_T\le c_a\le1/16$
and $sh\le sc_a\le\refreshprob/8$. With $h=\bar a_T$ they give
\[
\overline{\cM}_{\rm GC}
\le\mathcal I_T+\mathcal Q_T\overline U_{\rm loc},
\qquad
J_T=O\left(\frac{\cN^0+\Omega^0}{Th}+a_T\sigma^2\right).
\]
The explicit coefficients retain $K,s,\refreshprob$ and both movement scales.
For part~(ii), we also need averages of square roots of the
states:
\begin{align}
\frac1T\sum_{t<T}\sqrt{\cN_{\rm loc}^t}
&\le\sqrt{\frac{6\mathfrak Z^0}{Tsh}}+\sqrt{\frac{61\,\Omega^0}{T\refreshprob}}
+\sqrt{181\Big(1+\frac s\refreshprob\Big)a_T}\,\sigma,
\label{eq:q1-n-ledger}\\
\frac1T\sum_{t<T}\sqrt{\Xi^t}
&\le C\left\{\frac{\sqrt{\Xi^0}}{T\refreshprob}
+\delta_T\sqrt{U_{{\rm loc},*}}+\frac{\alpha_T\sqrt{U_*}}\refreshprob\right\},
\label{eq:q1-xi-ledger}\\
\frac1T\sum_{t<T}\sqrt{\cP^t}
&\le C\left\{\frac{\sqrt{\cP^0}}{Tsh}
+\sqrt{\frac{\Xi^0}{T\refreshprob}}
+\frac{\alpha_T\sqrt{U_*}}{sh}
+\frac{\alpha_T\sqrt{U_*}}\refreshprob
+\delta_T\sqrt{U_{{\rm loc},*}}\right\}.
\label{eq:q1-p-ledger}
\end{align}
Here $U_*:=\sup_t\cU^t$ and $U_{{\rm loc},*}:=\sup_tU_{\rm loc}^t$. The last
two bounds come from a recursion of the form $A^{t+1}\le(1-c)A^t+F^t$ and the
square-root average \eqref{eq:sqrt-convolution}, together with
$\sqrt{x+y+z}\le\sqrt x+\sqrt y+\sqrt z$.
\begin{itemize}
\item For~\eqref{eq:q1-xi-ledger}, \cref{lem:tracking-contraction,lem:expected-client-drift}
give $c=\refreshprob/2$ and
$F^t\le2\refreshprob\delta_T^2U_{{\rm loc},*}+3\alpha_T^2U_*/\refreshprob$, so
\[
\frac1T\sum_{t<T}\sqrt{\Xi^t}\le\frac{4\sqrt{\Xi^0}}{T\refreshprob}
+2\delta_T\sqrt{U_{{\rm loc},*}}+\frac{\sqrt6\,\alpha_T\sqrt{U_*}}{\refreshprob}.
\]
\item For~\eqref{eq:q1-p-ledger}, \cref{lem:population-momentum} gives $c=sh/2$
and $F^t=C_pshL_g^2(\E\cD^t+\Xi^t)+C_pL_g^2\alpha_T^2U_*/(sh)$. Summing the
tracking recursion gives
$\sum_{t<T}\Xi^t\le2\Xi^0/\refreshprob+T(4\delta_T^2U_{{\rm loc},*}+6\alpha_T^2U_*/\refreshprob^2)$,
so
\[
\frac2{Tsh}\sum_{t<T}F^t\le2C_pL_g^2\Big[5\delta_T^2U_{{\rm loc},*}
+\frac{2\Xi^0}{T\refreshprob}+\frac{6\alpha_T^2U_*}{\refreshprob^2}
+\frac{\alpha_T^2U_*}{s^2h^2}\Big].
\]
Its square root, plus $4\sqrt{\cP^0}/(Tsh)$, gives the five terms
of~\eqref{eq:q1-p-ledger}.
\item For~\eqref{eq:q1-n-ledger}, the Cauchy--Schwarz inequality gives
$T^{-1}\sum_{t<T}\sqrt{\cN_{\rm loc}^t}\le\sqrt{\overline{\cN}_{\rm loc}}\le\sqrt{J_T}$
by \cref{thm:lyapunov}(b). Step~3 below shows, using only the schedule of
$a_T$ common to both branches, that $C_{\Omega,T}<61$ and that the $\sigma^2$
coefficient of $J_T$ is at most $181(1+s/\refreshprob)a_T$. Then use
$\sqrt{x+y+z}\le\sqrt x+\sqrt y+\sqrt z$.
\end{itemize}

\par\smallskip\noindent\textbf{Step 3: Absorb movement for positive preconditioners.}\quad
Apply~\eqref{eq:lyap-final} with the coefficients from Step 1. The result is
\begin{align*}
\frac{\lambda_{p_T}}2\mathcal G_{2,T}
+\left(\frac\epsilon2-\frac{L\alpha_T}2
-\frac{\mathcal Q_T}{2\epsilon}\right)\overline U_{\rm loc}
\le\frac{\Delta_0}{T\alpha_T}
+\frac{\mathcal I_T}{2\epsilon}+\frac{\lambda_{p_T}G^2p_T}2.
\end{align*}
Lemma~\ref{lem:absorption-schedule} verifies
$L\alpha_T\le\epsilon/2$ and $\mathcal Q_T\le\epsilon^2/4$.
The movement coefficient is therefore at least $\epsilon/8$.
Dropping it gives
\[
\mathcal G_{2,T}\le\frac{2\Delta_0}{\lambda_{p_T}T\alpha_T}
+\frac{\mathcal I_T}{\lambda_{p_T}\epsilon}+G^2p_T.
\]
For Algorithm~\ref{alg:FAGC}, $\cN^0=0$, $\Xi^0=0$, and $\cP^0=\|\nabla f(x^0)\|_2^2$, while $\mathcal E_{\rm init}=\cP^0+\Omega^0$. To reduce $\mathcal I_T$ in~\eqref{eq:state-initial-noise}, use four facts:
$\mathfrak Z^0=\theta_T^2n^{-1}\sum_i\E\norm{u_i^0}^2\le\Omega^0$, because $e_i^0=0$ and
$\theta_T\le2sh/\refreshprob\le1$; $1/(T\refreshprob)\le1/(2Tsh)$, because
$sh\le\refreshprob/2$; $C_{\Omega,T}\le18+36+6+3/64<61$; and $h\le Ka_T$ together
with $\theta_T^2\le2sh/\refreshprob$, which bound the $\sigma^2$ coefficient
in~\eqref{eq:explicit-innovation-average} by $181(1+s/\refreshprob)a_T$. Hence
\[
\mathcal I_T\le\frac{4\cP^0+73\,\Omega^0}{Tsh}+362\Big(1+\frac s\refreshprob\Big)a_T\sigma^2
\le C\Big\{\frac{\mathcal E_{\rm init}}{T\bar a_T}+a_T\sigma^2\Big\},
\]
with $C=\max\{73/s,\,362(1+s/\refreshprob)\}$. Absorbing $\epsilon^{-1}$ into $C$ proves the self-contained bound~\eqref{eq:main-explicit-residual}.

Under the schedule, $h,\alpha_T=\Theta(T^{-1/2})$, $C_p<11$, and
$\mathcal I_T=O(T^{-1/2})$ for fixed problem and participation parameters.
Also $\lambda_{p_T}^{-1}=O(\sqrt{\log(T+2)})$ and
$G^2p_T=O(T^{-2})$. This proves the stated rate for any fixed
$b\in(0,1)$. The proof does not use the denominator-variation bound of
\cref{lem:v-variation}. If stochastic gradients are uniformly
bounded, $\lambda$ is bounded below deterministically, and the tail term
is zero. The same calculation gives $O(T^{-1/2})$.

\par\smallskip\noindent\textbf{Step 4: Use bounded movement for part~(ii).}\quad
For normalization and Muon, $U_{{\rm loc},*}\le1$ in their update
norms. Insert~\eqref{eq:q1-n-ledger}--\eqref{eq:q1-p-ledger} and
\eqref{eq:local-moment-error} into~\eqref{eq:master}. The geometry and norm
conversion constants are those of~\cref{prop:optimizer-bridge}. More explicitly, writing $c=C_{2,\min}$, the master inequality and $\sqrt{x+y}\le\sqrt x+\sqrt y$ give
\begin{align*}
c\mathcal G_{1,T}\le{}&\frac{\Delta_0}{T\alpha_T}
+\frac{L\alpha_TU_*}{2}+C_3\frac1T\sum_{t<T}\sqrt{\cM_{\rm GC}^t}\\
\le{}&\frac{\Delta_0}{T\alpha_T}+\frac{L\alpha_TU_*}{2}
+C_3\biggl[\frac{\sqrt2}{T}\sum_{t<T}\sqrt{\cP^t}
+\sqrt{12}L_g\delta_T\sqrt{U_{{\rm loc},*}}\\
&\hspace{20mm}+\frac{\sqrt{12}L_g}{T}\sum_{t<T}\sqrt{\Xi^t}
+\frac{\sqrt2}{T}\sum_{t<T}\sqrt{\cN_{\rm loc}^t}\biggr].
\end{align*}
With $a_T=\Theta(K^{-1}T^{-1/2})$ and $\alpha_T=\delta_T=T^{-3/4}$,
the objective, filtered noise, and momentum-lag terms have order
$T^{-1/4}$. The initial noise term $\sqrt{6\mathfrak Z^0/(Tsh)}$ also has order
$T^{-1/4}$, the other initial momentum and tracker terms have order
$T^{-1/2}$, and the remaining movement
terms have order $T^{-3/4}$. To see the substitution explicitly, the three leading terms are
\[
\frac1{T\alpha_T}=O(T^{-1/4}),\qquad
\sqrt{a_T}\,\sigma=O(T^{-1/4}),\qquad
\frac{\alpha_T}{s\bar a_T}=O(T^{-1/4}).
\]
The square-root tracker transient is $O(T^{-1/2})$, and both $\delta_T$ and $\alpha_T/\refreshprob$ are $O(T^{-3/4})$. All hidden constants hold $K,s,\refreshprob$ fixed. This proves~\eqref{eq:main-first-power-rate}.

\par\smallskip\noindent\textbf{A reference inequality for the Adam family.}\quad
Corollary~\ref{cor:fixed-gain-main} uses the following inequality, which
compares Adam's scaling with $D_{\rm ref}=\epsilon^{-1}I$ instead of using its
movement gain. Average the reference
inequality~\eqref{eq:adam-reference-c3} over the selected client steps (with
$\psi=1$). With $g=g^t$ and $\cR_v^t$ the selected-step average of
$\E\norm{[D(v)-D_{\rm ref}]m}_2^2$, smoothness gives the one-round inequality
\[
\E f(x^{t+1})\le\E f(x^t)-\frac{\alpha_T}{2\epsilon}\E\norm{g^t}_2^2
+\frac{\alpha_T}\epsilon\cM_{\rm GC}^t+\alpha_T\epsilon\,\cR_v^t
+\frac L2\alpha_T^2U_*.
\]
Summing over $t<T$ and dividing by $T\alpha_T/(2\epsilon)$ gives
\begin{equation}
\mathcal G_{2,T}\le\frac{2\epsilon\Delta_0}{T\alpha_T}+2\overline{\cM}_{\rm GC}
+2\epsilon^2\cR_{v,T}+\epsilon L\alpha_TU_*,
\label{eq:fagc-reference-sum}
\end{equation}
where $\cR_{v,T}=T^{-1}\sum_{t<T}\cR_v^t$. Unlike the bound of Step~3, this
one has no movement gain.
\end{proof}

\subsection{Proof of Corollary~\ref{cor:fixed-gain-main}: fixed-gain neighborhoods}
\label{app:proof-fixed-gain}
The proof replaces Lemma~\ref{lem:innovation-influence}, which needs the
small-gain conditions, by Lemma~\ref{lem:fixed-gain-two-state} and retains the terms that do not vanish.
\begin{proof}[Proof of Corollary~\ref{cor:fixed-gain-main}]
\par\smallskip\noindent\textbf{Step 1: Fix the parameter regime and retain its residuals.}\quad
Let $a_T\equiv a\in(0,1)$ and, for Adam/AMSGrad/SOAP, assume
\cref{assump:BG} and let $b_T\equiv b\in(0,1)$. Choose steps with
$\alpha_T\to0$, $\delta_T\to0$, and $T\alpha_T\to\infty$ so the optimization
gap, global smoothness movement, local drift, and initialization transients
vanish. Then \FA\ converges
to the neighborhoods
\begin{align}
\limsup_T\frac1T\sum_{t<T}\E\norm{\nabla f(x^t)}_2^2
&\le C\{\Gamma_{s,\refreshprob}(a+K^{-1})\sigma^2+\cR_{v,\infty}\},
&&\text{Adam/AMSGrad/SOAP},\label{eq:fixed-q2}\\
\limsup_T\frac1T\sum_{t<T}\E\norm{\nabla f(x^t)}_*
&\le C\sqrt{\Gamma_{s,\refreshprob}(a+K^{-1})}\,\sigma,
&&\text{normalization/Muon}.
\label{eq:fixed-q1}
\end{align}
For \NT, add $CB^2$ to \eqref{eq:fixed-q2} and $CB$ to
\eqref{eq:fixed-q1}. The Adam/SOAP quantity $\cR_{v,\infty}$ is the
persistent denominator-state residual
$\limsup_{T\to\infty}\cR_{v,T}$ defined in
\eqref{eq:denominator-residual}; fixed $b$ is therefore a valid but weaker
neighborhood result, not a convergence claim. The main-text
Corollary~\ref{cor:fixed-gain-main} states \eqref{eq:fixed-q1};
\eqref{eq:fixed-q2} is its Adam-family counterpart.

The size of $\cR_{v,\infty}$ matters. In the reference comparison
(\cref{lem:adam-reference}) the residual enters C3 as $\epsilon\norm{r_v}^2$,
where $r_v=[D(v)-D_{\rm ref}]m$, and
dividing by $C_{2,\min}=1/(2\epsilon)$ turns it into $2\epsilon^2\cR_{v,\infty}$.
For each coordinate,
$0\le\epsilon^{-1}-(\sqrt{v_j}+\epsilon)^{-1}\le\min\{\epsilon^{-1},\sqrt{v_j}/\epsilon^2\}$,
so
\[
2\epsilon^2\cR_{v,\infty}\le2\sup_{t,k}\E\norm{m_i^{t,k}}_2^2
\quad\text{and}\quad
2\epsilon^2\cR_{v,\infty}\le\frac2{\epsilon^2}\sup_{t,k}\E\big[\norm{v_i^{t,k}}_\infty\norm{m_i^{t,k}}_2^2\big].
\]
By the proof of \cref{lem:adam-expected-movement}, the first bound is a constant
multiple of $G^2+\sigma^2$, comparable to the trivial bound
$\norm{\nabla f}_2^2\le G^2$. The
second is small only when the learned denominator stays small compared with
$\epsilon$. The fixed-$\beta_2$ neighborhood for the Adam family is therefore
informative mainly for large $\epsilon$; normalization and Muon have no such
residual.

Under the additional small-gain conditions $\bar a\le1/8$ and
$s\bar a\le\refreshprob/2$, \cref{lem:innovation-influence} sharpens
$a+K^{-1}$ to $a$ in both displays.
Without both movement limits, the finite-horizon residual retains
\begin{equation}
 C\alpha_TU_*+C(\delta_T^2U_{{\rm loc},*})^{q/2}.
 \label{eq:fixed-movement-floor}
\end{equation}
Thus a fixed nonzero global step leaves the smoothness floor, while a fixed
nonzero local displacement leaves a separate drift floor even if server
aggregation makes $\alpha_T$ small.

\par\smallskip\noindent\textbf{Step 2: Bound the momentum error with a fixed gain.}\quad
Theorem~\ref{thm:lyapunov}(a) holds for every $a\in(0,1)$. With $a$ fixed,
$h=1-(1-a)^K$ and $C_p=9Ka/h$ are fixed, so $\mathcal Q_T\to0$ when
$\alpha_T,\delta_T\to0$, and the $\cP^0$ and $\Xi^0$ terms
of~\eqref{eq:state-ledger-general} vanish as $T\to\infty$.
\Cref{lem:fixed-gain-two-state} gives
$\limsup_t\cN_{\rm loc}^t\le(9a+6/K)\sigma^2$, and a time average has the same
limsup bound. Since $U_{{\rm loc},*}$ is finite (at most one for
normalization and Muon, and \cref{lem:adam-expected-movement} for the Adam
family), \eqref{eq:state-ledger-general} gives
\[
\limsup_T\overline{\cM}_{\rm GC}\le2(9a+6/K)\sigma^2\le18(a+K^{-1})\sigma^2.
\]

\par\smallskip\noindent\textbf{Step 3: Insert the bound into the descent inequality.}\quad
For normalization and Muon, apply~\eqref{eq:master} with $\chi_E=\nu=0$ and
$\cU^t\le1$, and use the Cauchy--Schwarz inequality
$T^{-1}\sum_{t<T}\sqrt{\cM_{\rm GC}^t}\le\sqrt{\overline{\cM}_{\rm GC}}$:
\[
C_2\mathcal G_{1,T}\le\frac{\Delta_0}{T\alpha_T}+\frac{L\alpha_T}2
+C_3\sqrt{\overline{\cM}_{\rm GC}}.
\]
Since $T\alpha_T\to\infty$ and $\alpha_T\to0$, taking the limsup gives
$\limsup_T\mathcal G_{1,T}\le\sqrt{18}\,(C_3/C_2)\sqrt{a+K^{-1}}\,\sigma$,
with $C_2$ and $C_3$ from Table~\ref{tab:constants_for_methods}. This
is~\eqref{eq:fixed-q1}. For the Adam family,~\eqref{eq:fagc-reference-sum}
gives $\limsup_T\mathcal G_{2,T}\le36(a+K^{-1})\sigma^2+2\epsilon^2\cR_{v,\infty}$,
which is~\eqref{eq:fixed-q2}. This limsup is finite because
$\norm{D(v)-D_{\rm ref}}_\infty\le\epsilon^{-1}$ (every entry of $D(v)$ lies in $(0,\epsilon^{-1}]$) and
\cref{lem:derived-fourth} gives a uniform mean-square bound for $m$; no scheduled
second-moment estimate is used for fixed $b$. The \NT\ terms come from
\cref{thm:FA-NT}.
The two-state recursion in \cref{lem:fixed-gain-two-state} controls
only the coupled stochastic momentum/tracker subsystem. It does not absorb
\eqref{eq:fixed-movement-floor}.

\end{proof}

\subsection{Proof of Theorem~\ref{thm:FA-NT}: the heterogeneous neighborhood}
\label{app:proof-fant}
The argument separates the sampled post-map correction, controls the local
momentum states, and then retains the heterogeneity contribution in the
optimizer comparison.
\begin{proof}[Proof of Theorem~\ref{thm:FA-NT}]
\par\smallskip\noindent\textbf{Step 1: Separate the correction and record the state bounds.}\quad
Let $c_i^t:=y^t-y_i^t$ and let $p_i^{t,k}$ be the counterfactual adaptive
direction generated along the selected local path. Conditional on $\F^t$,
with $\E_{\cS^t}$ denoting the average over the draw of $\cS^t$,
\begin{align}
\E_{\cS^t}\left[\frac1S\sum_{i\in\cS^t}c_i^t\right]&=0,
\label{eq:nt-zero-mean}\\
\E_{\cS^t}\norm{\frac1S\sum_{i\in\cS^t}c_i^t}_2^2
&=\frac{n-S}{S(n-1)}\frac1n\sum_i\norm{c_i^t}_2^2
\label{eq:nt-finite-pop}
\end{align}
for $n\ge2$, with value zero for $n=1$.  The second display vanishes at
$S=n$. The correction's first-order server contribution is therefore zero;
its mean square enters smoothness, while its effect on local paths is retained
in $\cD^t$.

Split each raw-gradient persistent momentum into a population momentum
$m_{i,\rm pop,NT}$ along the same realized path and
$e_{i,\rm NT}=m_i-m_{i,\rm pop,NT}$. Define
\[
 \mathcal E_{\rm pop,NT}^t:=\frac1n\sum_i\E\norm{m_{i,\rm pop,NT}^t
 -\nabla f_i(x^t)}_2^2,\qquad
 \cN_{\rm NT}^t:=\frac1n\sum_i\E\norm{e_{i,\rm NT}^t}_2^2.
\]
The warm start is split as
$m_{i,\rm pop,NT}^0=a_T\nabla f_i(x^0)$ and
$e_{i,\rm NT}^0=a_T\zeta_{i,\rm init}$. For normalization and Muon, where $G$
is not assumed,
$\mathcal E_{\rm pop,NT}^0=(1-a_T)^2n^{-1}\sum_i\norm{\nabla f_i(x^0)}_2^2
\le2\norm{\nabla f(x^0)}_2^2+2B^2$ depends on $B$. Its $B$-part contributes at
most a constant times $B^q/(Ts\bar a_T)^{q/2}$ to the bound below, and
$Ts\bar a_T$ is bounded below by a positive constant under the schedule, so
this part is absorbed into the $K_qB^q$ term. The constant $C$ in
Theorem~\ref{thm:FA-NT} is therefore independent of $B$.
The exact round recursions imply
\begin{align}
 \cN_{\rm NT}^{t+1}
 &\le(1-s\bar a_T)\cN_{\rm NT}^t
 +2sa_T\bar a_T\sigma^2,
 \label{eq:nt-noise-recursion}\\
 \mathcal E_{\rm pop,NT}^{t+1}
 &\le(1-c_ps\bar a_T)\mathcal E_{\rm pop,NT}^t
 +C_ps\bar a_TL_g^2\E\cD^t
 +\frac{C_pL_g^2\alpha_T^2\cU^t}{s\bar a_T},
 \label{eq:nt-population-recursion}\\
 \cM_{\rm NT}^t
 &\le C\{\mathcal E_{\rm pop,NT}^t+\cN_{\rm NT}^t
 +L_g^2\E\cD^t+a_T\sigma^2\},
 \label{eq:nt-local-error}
\end{align}
where $\cM_{\rm NT}^t$ is the selected client--step mean squared
distance between $m_i^{t,k}$ and $\nabla f_i(x^t)$. Consequently, for the
corresponding optimizer instance, define
\begin{align*}
A_{\rm init,NT}:={}&K_q\frac{\Delta_0}{T\alpha_T}+K_q
\left(\frac{\mathcal E_{\rm pop,NT}^0+\mathcal N_{\rm NT}^0}
{Ts\bar a_T}\right)^{q/2},\\
A_{\rm noise,NT}:={}&K_q(a_T\sigma^2)^{q/2},\\
A_{\rm participation,NT}:={}&K_q
\left(\frac{\alpha_T\sqrt{U_*}}{s\bar a_T}\right)^q.
\end{align*}

The terms $A_{\rm movement}$, $A_{\rm local}$, and $A_v$ are defined in~\eqref{eq:main-grouped-factors}. The finite-horizon comparison is
\begin{align}
 \frac1T\sum_{t<T}\E\norm{\nabla f(x^t)}_*^q
 \le C_{2,\min}^{-1}\bigl\{&A_{\rm init,NT}
 +A_{\rm noise,NT}+A_{\rm participation,NT}
 +A_{\rm movement}+A_{\rm local}\notag\\
 &+A_v+K_qB^q\bigr\}.
 \label{eq:fant-residual-bound}
\end{align}
For matrix directions,
\eqref{eq:nt-finite-pop} is applied in Frobenius norm and passed to spectral
norm via $\norm A_{\rm sp}\le\norm A_F$.

\par\smallskip\noindent\textbf{Step 2: Derive the bounds and compare optimizer directions.}\quad
The invariant $n^{-1}\sum_i c_i^t=0$ and \cref{lem:finite-pop} prove the first
two displays. At a selected round,
$e_{i,\rm NT}^{t+1}=(1-\bar a_T)e_{i,\rm NT}^t+Z_i^t$; otherwise it is
unchanged. The fresh-noise term has conditional mean zero, so its cross term with the old state vanishes. Hence
\[
\cN_{\rm NT}^{t+1}\le\{1-s+s(1-\bar a_T)^2\}\cN_{\rm NT}^t
+2sa_T\bar a_T\sigma^2
\le(1-s\bar a_T)\cN_{\rm NT}^t+2sa_T\bar a_T\sigma^2.
\]
This uses~\eqref{eq:local-ema-square-weight} and proves~\eqref{eq:nt-noise-recursion}. The
proof of \eqref{eq:nt-population-recursion} is the unrolling in
\cref{lem:population-momentum} with $\nabla f_i(x^t)$ in place of $g^t$ and no
$\Xi^t$ term. At a local step, weighted Jensen bounds the population error by $\mathcal E_{\rm pop,NT}^t+L_g^2\E\cD^t$. The noise error is at most $\cN_{\rm NT}^t+2a_T\sigma^2$. Squaring their sum gives
\[
\cM_{\rm NT}^t\le2\mathcal E_{\rm pop,NT}^t
+2\cN_{\rm NT}^t+2L_g^2\E\cD^t+4a_T\sigma^2,
\]
so $C=4$ suffices in~\eqref{eq:nt-local-error}.

The correction term is separated before applying smoothness. Its conditional
first-order mean is \eqref{eq:nt-zero-mean}; its second moment is bounded by
\cref{lem:nt-tracker-energy}, which gives $\E\norm{\bar d^t}^2\le10P_*^2$, and
its effect on local paths is retained in $\cD^t$. For the Adam family, averaging
\cref{lem:adam-reference} over the selected client steps and
applying~\eqref{eq:nt-adam-bridge} gives the one-round inequality
\[
\E f(x^{t+1})\le\E f(x^t)-\frac{\alpha_T}{2\epsilon}\E\norm{g^t}_2^2
+\frac{\alpha_T}\epsilon\big(2\cM_{\rm NT}^t+2B^2\big)
+\alpha_T\epsilon\,\cR_v^t+5L\alpha_T^2P_*^2.
\]
For normalization, \cref{lem:normalization} and~\eqref{eq:nt-norm-bridge} give
\[
\E f(x^{t+1})\le\E f(x^t)-\alpha_T\E\norm{g^t}_2
+2\alpha_T\Big(\sqrt{\cM_{\rm NT}^t}+B\Big)+5L\alpha_T^2,
\]
and Muon is the same with the nuclear norm, the rank factors
of~\eqref{eq:muon-averaged-c3}, and~\eqref{eq:nt-muon-bridge}. Summing over
rounds, the states enter only through
$\overline{\cM}_{\rm NT}=T^{-1}\sum_t\cM_{\rm NT}^t$. Applying
\cref{lem:convolution} to~\eqref{eq:nt-noise-recursion}
and~\eqref{eq:nt-population-recursion}, with $c=s\bar a_T$ and $c=s\bar a_T/2$, gives
\begin{align*}
\frac1T\sum_{t<T}\cN_{\rm NT}^t&\le\frac{\cN_{\rm NT}^0}{Ts\bar a_T}+2a_T\sigma^2,\\
\frac1T\sum_{t<T}\mathcal E_{\rm pop,NT}^t&\le
\frac{2\mathcal E_{\rm pop,NT}^0}{Ts\bar a_T}
+2C_pL_g^2\delta_T^2U_{{\rm loc},*}+\frac{2C_pL_g^2\alpha_T^2U_*}{s^2\bar a_T^2},
\end{align*}
and~\eqref{eq:nt-local-error}, with $\overline{\cD}\le\delta_T^2U_{{\rm loc},*}$,
combines them into
\[
\overline{\cM}_{\rm NT}\le\frac{4\mathcal E_{\rm pop,NT}^0+2\cN_{\rm NT}^0}{Ts\bar a_T}
+8a_T\sigma^2+(4C_p+2)L_g^2\delta_T^2U_{{\rm loc},*}
+\frac{4C_pL_g^2\alpha_T^2U_*}{s^2\bar a_T^2}.
\]
For the Adam family, summing the one-round inequality over $t<T$ and dividing
by $T\alpha_T/(2\epsilon)$ gives
\[
\mathcal G_{2,T}\le\frac{2\epsilon\Delta_0}{T\alpha_T}+4\overline{\cM}_{\rm NT}
+4B^2+2\epsilon^2\cR_{v,T}+10\epsilon L\alpha_TP_*^2.
\]
For normalization, summing and dividing by $T\alpha_T$, then using the
Cauchy--Schwarz inequality
$T^{-1}\sum_t\sqrt{\cM_{\rm NT}^t}\le\sqrt{\overline{\cM}_{\rm NT}}$, gives
\[
\mathcal G_{1,T}\le\frac{\Delta_0}{T\alpha_T}+2\sqrt{\overline{\cM}_{\rm NT}}
+2B+5L\alpha_T,
\]
and $\sqrt{x+y+z+w}\le\sqrt x+\sqrt y+\sqrt z+\sqrt w$ splits the square root
into one term per group. Substituting the bound on $\overline{\cM}_{\rm NT}$,
and reading $A_{\rm movement}$ and $A_{\rm participation,NT}$ with $10P_*^2$ in
place of $U_*$ (\cref{lem:nt-tracker-energy}), shows
that~\eqref{eq:fant-residual-bound} holds with
\[
K_2=\max\Big\{1,\ \epsilon,\ \frac L2,\ \frac{16}\epsilon,\
\frac{(8C_p+4)L_g^2}\epsilon\Big\},\qquad
K_1=\max\Big\{\frac L2,\ 4\sqrt2,\ 2L_g\sqrt{4C_p+2}\Big\}
\]
for the Adam family and for normalization; for Muon, $K_1$ is multiplied by
the rank factors of~\eqref{eq:muon-averaged-c3}. The heterogeneity enters only
through the $4B^2$ and $2B$ terms, which give $K_qB^q$. The warm start makes every initial
adaptive tracker finite and includes its stochastic energies in
$\mathcal E_{\rm pop,NT}^0,\cN_{\rm NT}^0$.

\par\smallskip\noindent\textbf{Step 3: Substitute the schedules.}\quad
Under the corresponding \NT\ conditions and \cref{assump:BDH}, the
schedules~\eqref{eq:constructive-schedule}--\eqref{eq:constructive-beta2} yield
\begin{equation}
\frac1T\sum_{t<T}\E\norm{\nabla f(x^t)}_*^q
\le O\bigl(T^{-q/(q+3)}\bigr)+CB^q,\qquad q\in\{1,2\}.
\label{eq:fant-neighborhood}
\end{equation}
Thus the residual is $CB^2$ for squared-error bounds and $CB$ for
normalized/LMO bounds. The finite-population sampling component vanishes
when $S=n$, while this comparison bound retains the $B^q$ term.

Apply the residual decomposition above. Under these schedules,
$(T\alpha_T)^{-1}$, $a_T^{q/2}$, and $(\alpha_T/\bar a_T)^q$ all have
order $T^{-q/(q+3)}$. Other initialization and movement terms decay at least
as quickly, and~\cref{lem:v-variation} gives $A_v=O(T^{-1})$ for $q=2$.
For $q=1$, $A_v=0$.
This upper bound does not by itself show that FA-NT ends at a different
point; Lemma~\ref{lem:fixed-point-bias}, proved next, shows where the two
fixed-point directions vanish.

\end{proof}

\subsection{Proof of Lemma~\ref{lem:fixed-point-bias}: where the fixed-point directions vanish}
\label{app:proof-fixed-point}
\begin{proof}
Fix $x$. With a constant input $u$, the recursions
$m\gets\beta_1m+(1-\beta_1)u$ and $v\gets\beta_2v+(1-\beta_2)u^{\odot2}$
have the unique fixed point $m=u$, $v=u^{\odot2}$ because $\beta_1,\beta_2<1$.

(i) Every client is refreshed in each round and evaluates its gradient at $x$, so
a refreshed FA-GC tracker stores $y_i=\nabla f_i(x)$, and
$y=n^{-1}\sum_jy_j=\nabla f(x)$. Every client's corrected input is therefore
$\nabla f_i(x)+y-y_i=\nabla f(x)$. Its moments settle at
$(\nabla f(x),\nabla f(x)^{\odot2})$, every local direction equals
$d_{\rm GC}(x)$, and so does their average. For Adam,
$d_{\rm GC}(x)=\nabla f(x)/(|\nabla f(x)|+\epsilon)$ coordinatewise, which is
zero exactly when $\nabla f(x)=0$. For normalization and exact Muon, the
direction has unit norm unless $\nabla f(x)=0$, where the convention gives
zero. For the finite-step map, every nonzero singular value of $Z_0$ lies in
$(0,1]$, and at least one exists because $M\neq0$. The polynomial
in~\eqref{eq:ns-main} satisfies $p_{\rm NS}(z)\ge z>0$ there, where
$p_{\rm NS}(z)=\frac{15}{8}z-\frac54z^3+\frac38z^5$ is the scalar polynomial
behind~\eqref{eq:ns-main} (\cref{lem:muon-inexact}), so
$P_R(M)\neq0$ whenever $M\neq0$.

(ii) FA-NT's moments receive the raw input $\nabla f_i(x)$ and settle at
$(\nabla f_i(x),\nabla f_i(x)^{\odot2})$, giving the direction
$p_i=\mathcal P(\nabla f_i(x);\nabla f_i(x)^{\odot2})$. By~\eqref{eq:adaptive-direction-refresh},
a refreshed tracker stores $y_i=p_i$, so $y=n^{-1}\sum_jp_j$, and every client
steps along $p_i+y-y_i=n^{-1}\sum_jp_j=d_{\rm NT}(x)$.

(iii) If $\mathcal P$ is linear, then
$n^{-1}\sum_i\mathcal P(\nabla f_i(x))=\mathcal P(\nabla f(x))$.
For the example in Section~\ref{sec:41}, the Adam direction of client $i$ at its
fixed point is $\varphi_\epsilon(x-x_i^*)$ with
$\varphi_\epsilon(u)=u/(|u|+\epsilon)$. This function is strictly increasing,
so $d_{\rm NT}$ is strictly increasing in $x$ and vanishes at exactly one point.
Write $F(x,\epsilon)=\sum_i\varphi_\epsilon(x-x_i^*)$. Then
$F(1,\epsilon)=-3\epsilon/\{(1+\epsilon)(4+\epsilon)\}<0$ and
$F(2,\epsilon)=\{2/(2+\epsilon)-2/(3+\epsilon)\}+\{1/(1+\epsilon)-1/(3+\epsilon)\}>0$,
so the zero $x(\epsilon)$ lies in $(1,2)$. Because $F$ is continuously
differentiable and $\partial_xF(x,\epsilon)=\sum_i\epsilon/(|x-x_i^*|+\epsilon)^2>0$, the
implicit function theorem of calculus shows that $x(\epsilon)$ is
differentiable. Since
$\partial_u\varphi_\epsilon(u)=\epsilon/(|u|+\epsilon)^2$ and
$\partial_\epsilon\varphi_\epsilon(u)=-\varphi_\epsilon(u)/(|u|+\epsilon)$,
differentiating $F(x(\epsilon),\epsilon)=0$ with $u_i=x(\epsilon)-x_i^*$ gives
\[
\frac{dx}{d\epsilon}=\frac{\sum_i\varphi_\epsilon(u_i)/(|u_i|+\epsilon)}
{\sum_i\epsilon/(|u_i|+\epsilon)^2}.
\]
At the zero, $\varphi_\epsilon(u_3)=-\varphi_\epsilon(u_1)-\varphi_\epsilon(u_2)$, and
$0<u_1,u_2<|u_3|$ because $x(\epsilon)\in(1,2)$. The numerator therefore equals
$\varphi_\epsilon(u_1)\{(u_1+\epsilon)^{-1}-(|u_3|+\epsilon)^{-1}\}
+\varphi_\epsilon(u_2)\{(u_2+\epsilon)^{-1}-(|u_3|+\epsilon)^{-1}\}>0$,
so $x(\epsilon)$ increases with $\epsilon$.
As $\epsilon\to0$, each term tends to $\operatorname{sign}(x-x_i^*)$, whose sum
vanishes at the median $x=1$; normalization gives the same limit directly. As
$\epsilon\to\infty$, $\epsilon\,d_{\rm NT}(x)\to\nabla f(x)$, whose zero is
$x^*=2$. With two clients, $\nabla f_1(x^*)=-\nabla f_2(x^*)$, and
any map with $\mathcal P(-u)=-\mathcal P(u)$ gives $d_{\rm NT}(x^*)=0$.
Under normalization in one dimension,
$d_{\rm NT}(x)=\frac12\{\operatorname{sign}(x-x_1^*)+\operatorname{sign}(x-x_2^*)\}$
vanishes for every $x$ strictly between $x_1^*$ and $x_2^*$.
\end{proof}

\begin{remark}[Numerical check]\label{rem:numerical-check}
We ran Algorithm~\ref{alg:FAGC} and Algorithm~\ref{alg:FA-NT} on the
three-client example with exact gradients, $S=\widetilde S=n=3$, $K=5$,
$\eta_g=1$, $\eta_{\ell}=0.05/\sqrt{t+1}$ in round $t$,
$(\beta_1,\beta_2)=(0.9,0.99)$, $x^0=0$, and $2\times10^4$ rounds. FA-GC ended
at $2.000$ for normalization and for Adam with $\epsilon\in\{10^{-8},1\}$.
FA-NT ended at $1.000$ for normalization and for Adam with $\epsilon=10^{-8}$,
and at $1.289$ for Adam with $\epsilon=1$, matching the zero of $d_{\rm NT}$.
The lemma concerns where the averaged direction vanishes. For normalization
and Muon, leftover momentum still produces unit-size local steps, so
convergence of the iterates also relies on decaying step sizes.
\end{remark}

\section{Proofs of the Core Lemmas and Descent Bounds}
\label{app:core-proofs}
The proofs follow the order of the statements in Appendix~\ref{app:core-lemmas}.
Each subsection links back to the result it establishes.

\subsection{Proof of Lemma~\ref{lem:client-drift}: client drift}
\label{app:proof-client-drift}
\begin{proof}
The first $k-1$ updates give
$\norm{x_i^{t,k}-x^t}_{\rm pr}\le(k-1)\eta_{\ell,T}D_1$.
Since $\sum_{k=1}^K(k-1)^2=(K-1)K(2K-1)/6$, averaging the squared bound proves~\eqref{eq:drift-bound}. The server displacement is the scaled average of $K$ local directions, so its norm is at most $\alpha_TD_1$.
\end{proof}

\subsection{Proof of Lemma~\ref{lem:expected-client-drift}: expected client drift}
\label{app:proof-expected-client-drift}
\begin{proof}
Write $x_i^{t,k}-x^t=-\eta_{\ell,T}\sum_{r<k}d_i^{t,r}$. The squared-norm inequality gives
\begin{align*}
\cD^t
&\le\frac{\eta_{\ell,T}^2}{SK}\sum_{i\in\cS^t}\sum_{k=1}^K
(k-1)\sum_{r<k}\norm{d_i^{t,r}}_{\rm pr}^2\\
&\le K^2\eta_{\ell,T}^2\frac1{SK}\sum_{i\in\cS^t,r}\norm{d_i^{t,r}}_{\rm pr}^2.
\end{align*}
Taking expectation proves the first claim. The server identity $x^{t+1}-x^t=-\alpha_T\bar d^t$ and Jensen's inequality give the second. Both estimates use the realized directions, so they hold without conditioning on the event $\mathcal H_p$.
\end{proof}

\subsection{Proof of Lemma~\ref{lem:tracking-contraction}: tracker lag}
\label{app:proof-tracking-contraction}
\begin{proof}
Let $\Delta x^t=x^{t+1}-x^t$. A refreshed client replaces $\gamma_i^{t,k}$ by $x_i^{t,k}$, so
\[
\norm{x_i^{t,k}-x^{t+1}}_{\rm pr}^2
\le2\norm{x_i^{t,k}-x^t}_{\rm pr}^2+2\norm{\Delta x^t}_{\rm pr}^2.
\]
For a stale client, Young's inequality with $\vartheta>0$ gives
\[
\norm{\gamma_i^{t,k}-x^{t+1}}_{\rm pr}^2
\le(1+\vartheta)\norm{\gamma_i^{t,k}-x^t}_{\rm pr}^2
+(1+\vartheta^{-1})\norm{\Delta x^t}_{\rm pr}^2.
\]
Average the old lag and counterfactual client displacement using Lemma~\ref{lem:mask-average}. In particular, the refreshed client displacement contributes $2\refreshprob\E\cD^t$. Server movement depends on the sampled sets, so its coefficient is obtained instead from the exact count $\sum_i\refreshmask_i^t=\widetilde S$. Thus
\[
\Xi^{t+1}\le(1-\refreshprob)(1+\vartheta)\Xi^t
+2\refreshprob\E\cD^t
+\{2\refreshprob+(1-\refreshprob)(1+\vartheta^{-1})\}\E\norm{\Delta x^t}_{\rm pr}^2.
\]
For $0<\refreshprob<1$, choose $\vartheta=\refreshprob/[2(1-\refreshprob)]$. The lag coefficient is $1-\refreshprob/2$, while the movement coefficient is $3\refreshprob-3+2/\refreshprob\le3/\refreshprob$. If $\refreshprob=1$, every client refreshes and the first displayed inequality directly gives the sharper full-refresh statement.
\end{proof}

\subsection{Proof of Lemma~\ref{lem:buffer-energy}: tracker noise}
\label{app:proof-buffer-energy}
\begin{proof}
Only one branch of the tracker replacement is active on each outcome:
\[
\norm{\omega_i^{t+1}}_2^2
=(1-\refreshmask_i^t)\norm{\omega_i^t}_2^2
+\refreshmask_i^t\norm{\bar\zeta_i^t}_2^2.
\]
For $r<k$, conditional centering of the later noise gives
$\E[\langle\zeta_i^{t,r},\zeta_i^{t,k}\rangle\mid\F^t]=0$ by (P2): average over the later draw first. Hence
\[
\E[\norm{\bar\zeta_i^t}_2^2\mid\F^t]
=K^{-2}\sum_k\E[\norm{\zeta_i^{t,k}}_2^2\mid\F^t]\le\sigma^2/K.
\]
Lemma~\ref{lem:mask-average} now gives the tracker-energy recursion. Centering over clients gives
$n^{-1}\sum_i\norm{\omega_i-\bar\omega}_2^2=n^{-1}\sum_i\norm{\omega_i}_2^2-\norm{\bar\omega}_2^2$, proving the contrast bound.

For the refresh fluctuation, put $q_i=\refreshmask_i^t(\bar\zeta_i^t-\omega_i^t)$ and let $P=I-n^{-1}\one\one^\top$ act on the client index. Since $\E[q_i\mid\F^t]=-\refreshprob\omega_i^t$ and $u=-P\omega$,
\[
\rho^{t+1}=-P\{q-\E[q\mid\F^t]\}.
\]
Thus $\E[\rho^{t+1}\mid\F^t]=0$. We use two elementary identities. For any
vectors $x_1,\ldots,x_n$ with average $\bar x$,
$n^{-1}\sum_i\norm{x_i-\bar x}^2=n^{-1}\sum_i\norm{x_i}^2-\norm{\bar x}^2\le n^{-1}\sum_i\norm{x_i}^2$,
so subtracting the client average does not increase the average energy. For
any random vector $q$,
$\E\norm{q-\E q}^2=\E\norm q^2-\norm{\E q}^2\le\E\norm q^2$. Also, given $\F^t$,
$\E\norm{q_i}^2=\refreshprob\,\E[\norm{\bar\zeta_i^t}^2+\norm{\omega_i^t}^2]$,
because the cross term vanishes by (P2). Together these give
\begin{align*}
\frac1n\sum_i\E[\norm{\rho_i^{t+1}}_2^2\mid\F^t]
&\le\frac1n\sum_i\E[\norm{q_i}_2^2\mid\F^t]\\
&\le\refreshprob\left\{\frac1n\sum_i\norm{\omega_i^t}_2^2+\frac{\sigma^2}{K}\right\}.
\end{align*}
The old-tracker/fresh-noise cross term vanishes by conditional centering. The training and refresh indicators remain dependent throughout this calculation.
\end{proof}

\subsection{Proof of Lemma~\ref{lem:innovation-influence}: momentum noise under gradient reuse}
\label{app:proof-innovation-influence}
\emph{Roadmap.} A noise sample drawn at round $t$ enters the momentum at once
and, if the client refreshes, enters the tracker as well, from where it later
returns to the momenta through the correction $u_i$. The two contributions are
therefore correlated, and the following example makes this concrete. Step~1
follows the momentum noise and the tracker noise together through one
combined state whose expected value shrinks by exactly $1-s\bar a_T$ per
round, and bounds its random part without assuming any orthogonality. Step~2
unrolls the resulting recursion. Step~3 passes from round boundaries to the
individual local steps.

\begin{remark}[Covariance under gradient reuse]
\label{check:refresh-counterexample}
For $n=2$, $S=\widetilde S=1$, and $K=1$, the noise in another client's tracker can enter
client $j$'s contrast $u_j$ in two consecutive rounds in which $j$ trains, so
their covariance need not vanish.
Lemma \ref{lem:innovation-influence} allows this: its proof bounds the direct
and tracker contributions of the shared noise term together with the three-term
inequality and never treats them as uncorrelated. Under full participation and
full refresh, a noise sample that enters the momentum and then the next
tracker similarly gives a one-round covariance, which is covered in the same way.
More explicitly, isolate one noise sample $\zeta_j^r$ that enters a refresh,
and set all other noises and old trackers to zero. Immediately after round $r$ its coefficient
in $e_j^{r+1}$ is $a$ and its tracker-contrast coefficients are
$-1/2$ for $j$ and $1/2$ for the other client. If $j$ trains again, its
coefficient one round later is $a(1/2-a)$; if the other client trains, that
client's coefficient is $a/2$. These are sums of the direct and tracker paths,
not independent contributions. For full participation/full refresh with
general $n$, the corresponding next-round coefficient is $a(1/n-a)$ for $j$
and $a/n$ for $i\ne j$.
\end{remark}
\begin{proof}
\par\smallskip\noindent\textbf{Step 1: Control the joint momentum and tracker state.}\quad
A sampled noise term enters the momentum directly and, after a refresh, again
through the stored tracker, so these two contributions are correlated. To handle
both at once, introduce the joint state
\[
z_i^t:=e_i^t+\theta_Tu_i^t.
\]
Write $\mathfrak Z^t=n^{-1}\sum_i\E\norm{z_i^t}^2$ for its average energy.
By \eqref{eq:round-noise-filter}, \cref{lem:buffer-energy}, and
$\E[I_i^t\mid\F^t]=s$,
\[
\E[z_i^{t+1}\mid\F^t]
=(1-s\bar a_T)e_i^t+
\{s\bar a_T+\theta_T(1-\refreshprob)\}u_i^t
=(1-s\bar a_T)z_i^t.
\]
After subtracting the displayed conditional mean, the centered increment is
the sum of
\[
\bar a_T(I_i^t-s)(u_i^t-e_i^t),\qquad
I_i^tZ_i^t,\qquad \theta_T\rho_i^{t+1}.
\]
These terms may be correlated. In particular, the last two reuse the same
gradient on refresh rounds. Therefore, use
$\norm{x_1+x_2+x_3}^2\le3\sum_j\norm{x_j}^2$, not orthogonality between
them. Since $\theta_T\le2s\bar a_T/\refreshprob\le1$, we have
\[
 \norm{u_i^t-e_i^t}^2
 =\norm{(1+\theta_T)u_i^t-z_i^t}^2
 \le2\norm{z_i^t}^2+8\norm{u_i^t}^2.
\]
Write $h=\bar a_T$. Conditional centering then gives
\[
 \E\norm{z_i^{t+1}}^2=(1-sh)^2\E\norm{z_i^t}^2
 +\E\norm{z_i^{t+1}-\E[z_i^{t+1}\mid\F^t]}^2.
\]
We bound the three pieces separately, after averaging over clients.
\begin{itemize}
\item Training indicator: $\E[(I_i^t-s)^2\mid\F^t]=s(1-s)\le s$ and the
display above give
$3\E\norm{\bar a_T(I_i^t-s)(u_i^t-e_i^t)}^2\le3sh^2(2\mathfrak Z^t+8\Omega^t)$.
\item Fresh noise: distinct noises are uncorrelated (P2), so
$\E\norm{Z_i^t}^2\le W_2\sigma^2\le2a_Th\sigma^2$
by~\eqref{eq:local-ema-square-weight}, and
$3\E\norm{I_i^tZ_i^t}^2\le6sa_Th\sigma^2$.
\item Refresh fluctuation: \eqref{eq:refresh-martingale-variance} gives
$3\theta_T^2\E\norm{\rho_i^{t+1}}^2\le3\theta_T^2\refreshprob(\Omega^t+\sigma^2/K)$.
\end{itemize}
Adding the three bounds to the contracted mean gives
\begin{align}
\mathfrak Z^{t+1}\le{}&\{(1-sh)^2+6sh^2\}\mathfrak Z^t\notag\\
&+\{24sh^2+3\theta_T^2\refreshprob\}\Omega^t
+\{6sa_Th+3\theta_T^2\refreshprob/K\}\sigma^2\notag\\
\le{}&(1-sh)\mathfrak Z^t
+\left(24sh^2+\frac{12s^2h^2}{\refreshprob}\right)\Omega^t
+\left(6sa_Th+\frac{12s^2h^2}{\refreshprob K}\right)\sigma^2,
\label{eq:joint-filter-recursion}
\end{align}
The second inequality uses $\theta_T\le2sh/\refreshprob$ and
$(1-sh)^2+6sh^2\le1-sh$. The latter is equivalent to $h(s+6)\le1$, which holds
because $h\le1/8$ and $s\le1$.
\par\smallskip\noindent\textbf{Step 2: Sum the geometric terms.}\quad
From \eqref{eq:buffer-energy-recursion},
$\Omega^t\le(1-\refreshprob)^t\Omega^0+\sigma^2/K$. Unrolling
\eqref{eq:joint-filter-recursion}, using
$\refreshprob-sh\ge\refreshprob/2$ for the product of the two geometric factors, and dividing
the constant per-round input by $sh$ gives the steady-state level
\[
C\left[
\frac{h}{K}+
\frac{sh}{\refreshprob K}+a_T\right]\sigma^2
\le C\left(1+\frac{s}{\refreshprob}\right)a_T\sigma^2,
\]
where $h\le Ka_T$. The geometric convolution is bounded by
$2\refreshprob^{-1}(1-sh)^t$ because $sh\le\refreshprob/2$; this gives the two displayed
transients with numerical $c,C$. Finally
$\norm{e_i^t}^2\le2\norm{z_i^t}^2+2\theta_T^2\norm{u_i^t}^2$.
\par\smallskip\noindent\textbf{Step 3: Stop the filter within a local round.}\quad
To make the within-round calculation explicit, put
$h_k=1-(1-a_T)^k\le\bar a_T$. Stopping
\eqref{eq:exact-noise-filter} after step $k$ gives
\[
 e_i^{t,k}=(1-a_T)^ke_i^t+h_ku_i^t
 +\sum_{r=1}^ka_T(1-a_T)^{k-r}\zeta_i^{t,r}.
\]
Apply $\norm{a+b+c}^2\le3(\norm a^2+\norm b^2+\norm c^2)$ to the three terms.
Uniform sampling of the old states gives the average $\cN^t$ for the first;
$h_k\le\bar a_T$ and the contrast bound
$n^{-1}\sum_i\E\norm{u_i^t}^2\le\Omega^t$ give $\bar a_T^2\Omega^t$ for the
second; and, since distinct noises are uncorrelated (P2),
$\sum_{r=1}^ka_T^2(1-a_T)^{2(k-r)}\sigma^2\le a_T\sigma^2$ for the third. This
proves~\eqref{eq:local-innovation-comparison} with the constant 3, and no step
uses $\bar a_T\le1/8$ or $s\bar a_T\le\refreshprob/2$. Averaging that comparison over time
and using $\Omega^t\le(1-\refreshprob)^t\Omega^0+\sigma^2/K$ gives
\[
 \frac1T\sum_{t<T}\cN_{\rm loc}^t
 \le C\left\{\frac{\cN^0+\Omega^0}{T c_\star}
 +\Gamma_{s,\refreshprob}a_T\sigma^2\right\}.
\]
Indeed, $\bar a_T^2/K\le a_T\bar a_T\le a_T/8$ follows from
$\bar a_T\le Ka_T$ and $\bar a_T\le1/8$. The transient
$\bar a_T^2\Omega^0/(T\refreshprob)$ is absorbed because $c_\star\le\refreshprob$.
\end{proof}

\begin{remark}[Exact one-client EMA]
For $n=S=\widetilde S=K=1$, $u^t=0$ and
$e^{t+1}=(1-a)e^t+a\zeta^t$.  If $e^0=0$ and
$\E\norm{\zeta^t}^2=\sigma^2$, orthogonal innovations give exactly
\[
 \E\norm{e^t}^2=\frac{a\sigma^2}{2-a}
 \{1-(1-a)^{2t}\},
 \qquad
 \lim_{t\to\infty}\E\norm{e^t}^2=\frac{a\sigma^2}{2-a}.
\]
Thus \cref{lem:innovation-influence} has the correct $a\sigma^2$ scaling.
\end{remark}

\subsection{Proof of Lemma~\ref{lem:fixed-gain-two-state}: fixed first-moment gain}
\label{app:proof-fixed-gain-two-state}
\begin{proof}
At a selected round, \eqref{eq:round-noise-filter} reads
$e_i^{t+1}=(1-h)e_i^t+hu_i^t+Z_i^t$; otherwise $e_i$ is unchanged.
The fresh noise sum $Z_i^t$ is orthogonal to the old pair
$(e_i^t,u_i^t)$, even though it may also enter the newly refreshed tracker.
Convexity gives
\[
\|(1-h)e_i^t+hu_i^t\|^2
\le(1-h)\|e_i^t\|^2+h\|u_i^t\|^2.
\]
Average the training mask, use \eqref{eq:contrast-energy}, and use
$\E\|Z_i^t\|^2\le2ah\sigma^2$ from
\eqref{eq:local-ema-square-weight}. This proves the first line of
\eqref{eq:fixed-two-state}; \eqref{eq:buffer-energy-recursion} proves the
second. Unrolling the second line gives
$\Omega^t\le(1-\refreshprob)^t\Omega^0+\sigma^2/K$, so
$\limsup_t\Omega^t\le\sigma^2/K$. For any $\eta>0$, eventually
$\Omega^t\le\sigma^2/K+\eta$, and then the first line gives
$\cN^{t+1}\le(1-sh)\cN^t+sh(\sigma^2/K+\eta+2a\sigma^2)$; unrolling it gives
$\limsup_t\cN^t\le(2a+1/K)\sigma^2+\eta$. Since $\eta$ is arbitrary, this proves
\eqref{eq:fixed-two-state-floor}. For the transient, unrolling the first line
gives $(1-sh)^t\cN^0$ plus
$sh\Omega^0\sum_{r<t}(1-sh)^{t-1-r}(1-\refreshprob)^r
=sh\Omega^0\frac{(1-sh)^t-(1-\refreshprob)^t}{\refreshprob-sh}$,
which is at most $\frac{sh}{\abs{\refreshprob-sh}}\Omega^0\max\{1-sh,1-\refreshprob\}^t$. The local-step statement follows from~\eqref{eq:local-innovation-comparison},
which holds for every $a\in(0,1)$: its limsup is at most
$3(2a+1/K)\sigma^2+3\sigma^2/K+3a\sigma^2=(9a+6/K)\sigma^2$. No covariance between
the two new states is discarded because the comparison treats their marginal
energies separately.
\end{proof}

\subsection{Proof of Lemma~\ref{lem:fagc-mean}: effect of the correction}
\label{app:proof-fagc-mean}
\begin{proof}
For \FA, write the difference as
\[
[\nabla f_i(x_i^{t,k})-\nabla f_i(x^t)]
+[y_{\rm pop}^t-\nabla f(x^t)]-[y_{i,\rm pop}^t-\nabla f_i(x^t)].
\]
The squared three-term inequality contributes a factor three. The first
selected client-step average is at most $L_g^2\E\cD^t$. Jensen bounds the
second and third terms by $L_g^2\Xi^t$ each. Thus the total is at most
$3L_g^2(\E\cD^t+2\Xi^t)\le6L_g^2(\E\cD^t+\Xi^t)$. Lemma~\ref{lem:mask-average} converts the selected-client average of the third term into the population average that defines $\Xi^t$.
\end{proof}

\subsection{Proof of Lemma~\ref{lem:population-momentum}: noise-free momentum error}
\label{app:proof-population-momentum}
\begin{proof}
Put $h=\bar a_T$, $w_i^t=m_{i,\rm pop}^t-g^t$, and define the normalized
population forcing
\[
\Phi_i^t:=h^{-1}\sum_{k=1}^K a_T(1-a_T)^{K-k}
(g_{i,\rm pop}^{t,k}-g^t).
\]
The counterfactual local paths are computed without looking at which clients are selected this round (P3). Lemma~\ref{lem:mask-average} therefore averages the training indicator without changing the distribution of $\Phi_i^t$. For a selected client,
\[
m_{i,\rm pop}^{t+1}-g^t=(1-h)w_i^t+h\Phi_i^t;
\]
for an unselected client it equals $w_i^t$. Hence the convexity identity gives
\begin{align*}
\frac1n\sum_i\E\|m_{i,\rm pop}^{t+1}-g^t\|_2^2
&\le(1-sh)\cP^t+sh\frac1n\sum_i\E\|\Phi_i^t\|_2^2\\
&\le(1-sh)\cP^t+6R_{K,T}shL_g^2(\E\cD^t+\Xi^t),
\end{align*}
where weighted Jensen gives
\[
 \norm{\Phi_i^t}^2\le
 \sum_{k=1}^K\frac{a_T(1-a_T)^{K-k}}{h}
 \norm{g_{i,\rm pop}^{t,k}-g^t}^2,
\]
and the largest normalized weight is at most
$a_T/h=R_{K,T}/K$. Hence the weighted client-step average is at most
$R_{K,T}$ times the uniform average in \cref{lem:fagc-mean}; this proves the
second line without treating a weighted average as a uniform one. Let
$\Delta g^t=g^{t+1}-g^t$. With the explicit Young parameter $sh/2$,
\[
\|m_{i,\rm pop}^{t+1}-g^{t+1}\|_2^2
=\|(m_{i,\rm pop}^{t+1}-g^t)-\Delta g^t\|_2^2
\le(1+sh/2)\|m_{i,\rm pop}^{t+1}-g^t\|_2^2
+(1+2/(sh))\|\Delta g^t\|_2^2.
\]
Since $(1+sh/2)(1-sh)\le1-sh/2$,
$(1+sh/2)6R_{K,T}sh\le9R_{K,T}sh$, and
$\E\|\Delta g^t\|_2^2\le L_g^2\alpha_T^2\cU^t$, this proves
\eqref{eq:population-momentum-recursion} with $c_p=1/2$ and the displayed
coefficients (all are at most $C_p=9R_{K,T}$ because $R_{K,T}\ge1$). Here the gradient bound follows by averaging the client Lipschitz inequalities: $\|\nabla f(x)-\nabla f(y)\|_2\le n^{-1}\sum_i\|\nabla f_i(x)-\nabla f_i(y)\|_2\le L_g\|x-y\|_{\rm pr}$.

For a local step, the EMA weights together with the initial weight form a
convex combination. After averaging over $k$, the total coefficient assigned
to a forcing at step $r$ is
$K^{-1}\{1-(1-a_T)^{K-r+1}\}\le K^{-1}$, while uniform client sampling maps
the initial-state contribution to $\cP^t$. Weighted Jensen and
\cref{lem:fagc-mean} therefore give
\[
\E\left[\frac1{SK}\sum_{i\in\cS^t,k}
\|m_{i,\rm pop}^{t,k}-g^t\|_2^2\right]
\le\cP^t+6L_g^2(\E\cD^t+\Xi^t).
\]
Apply $\|u+v\|^2\le2\|u\|^2+2\|v\|^2$ to
$m_i^{t,k}-g^t=(m_{i,\rm pop}^{t,k}-g^t)+e_i^{t,k}$ and average. This proves
\eqref{eq:local-moment-error}. Every norm is taken before averaging.
\end{proof}

\subsection{Proof of Lemma~\ref{lem:v-variation}: denominator residual}
\label{app:proof-v-variation}
\begin{proof}
Fix a client and list its second-moment inputs in update order as
$z^1,\ldots,z^N$, where $N=\ell_i^t+k$ includes the current update.
These inputs are the corrected gradients for \FA, raw gradients for \NT,
or their rotated coordinates for SOAP. The initialization is
$v_i^0=0$, so unrolling the exponential moving average gives
\[
v_i^{t,k}=b_T\sum_{j=1}^{N}(1-b_T)^{N-j}(z^j)^{\odot2}.
\]
Each coordinate of $v_i^{t,k}$ is therefore at most $b_T\sum_{j\le N}\norm{z^j}_2^2$.
List all positions (round, local step) at which client $i$ could have updated
its second moment up to the current one; there are at most
$N_{\max}=KT$ of them, or $KT+1$ with \NT's warm start, and the client's
actual inputs are those positions at which it was selected. Writing
$X^j=\norm{z^j}_2^2$ at position $j$ and $I^j\in\{0,1\}$ for selection there,
\[
\norm{v_i^{t,k}}_\infty\le b_T\sum_{j\le N_{\max}}I^jX^j,\qquad
\Big(\sum_{j\le N_{\max}}I^jX^j\Big)^2\le N_{\max}\sum_{j\le N_{\max}}(X^j)^2,
\]
where the second inequality is Cauchy--Schwarz. Taking expectations and using
$\E(X^j)^2\le CH_4^4$ from \cref{lem:derived-fourth} gives
\[
\big\{\E\norm{v_i^{t,k}}_\infty^2\big\}^{1/2}\le C N_{\max}b_T H_4^2.
\]
Also, for each coordinate, $\abs{(\sqrt{v_j}+\epsilon)^{-1}-\epsilon^{-1}}
=\sqrt{v_j}/(\epsilon(\sqrt{v_j}+\epsilon))\le\sqrt{v_j}/\epsilon^2$, so
\begin{equation}
\norm{D(v)-D_{\rm ref}}_\infty^2
\le\epsilon^{-4}\norm{v}_\infty.
\label{eq:denominator-holder}
\end{equation}
Thus Cauchy--Schwarz yields
\[
\E\norm{[D(v_i^{t,k})-D_{\rm ref}]m_i^{t,k}}_2^2
\le\epsilon^{-4}
\big(\E\norm{v_i^{t,k}}_\infty^2\big)^{1/2}
\big(\E\norm{m_i^{t,k}}_2^4\big)^{1/2}
\le\frac{C N_{\max}b_T H_4^4}{\epsilon^4}.
\]
These bounds hold at every local step, so averaging
proves~\eqref{eq:v-cumulative-residual}. For AMSGrad with $v_i^0=0$,
each coordinate satisfies
\[
0\le\widehat v_i^{t,k}\le b_T\sum_{j=1}^N (z^j)^{\odot2},
\]
because every partial EMA is at most the sum of all its nonnegative inputs.
The same calculation and~\eqref{eq:denominator-holder} then prove the residual
estimate for $\widehat v$. Unselected clients retain both states by
\eqref{eq:amsgrad-max}.
\end{proof}

\subsection{Proof of Lemma~\ref{lem:nt-tracker-energy}: FA-NT tracker energy}
\label{app:proof-nt-tracker-energy}
\begin{proof}
A tracker $y_i^t$ is the average of $K$ directions from client $i$'s last
refresh round, or the warm-start direction. Which round was the last refresh
depends only on the refresh indicators, which are independent of the
directions computed in that round. Summing over the possible last refresh
rounds and using $\norm{K^{-1}\sum_kx_k}^2\le K^{-1}\sum_k\norm{x_k}^2$ gives
$\E\norm{y_i^t}_{\rm pr}^2\le P_*^2$, and the same holds for $y^t$.
Apply $\norm{u-v}^2\le2\norm u^2+2\norm v^2$ for the first display, then
$\norm{p+c}^2\le2\norm p^2+2\norm c^2$ for the
second. Finally $\norm A_{\rm sp}\le\norm A_F$ handles matrix movement. This
bound precedes the smoothness calculation, so $\cU^t$ is not used circularly.
\end{proof}

\subsection{Proof of Theorem~\ref{thm:master}: one-round descent summed over rounds}
\label{app:proof-master}
\begin{proof}
By smoothness and \cref{lem:server-identities},
\[
f(x^{t+1})\le f(x^t)-\alpha_T\ip{\nabla f(x^t)}{\bar d^t}
+\frac L2\alpha_T^2\norm{\bar d^t}_{\mathrm{pr}}^2.
\]
Split $\langle\nabla f(x^t),\bar d^t\rangle$ into the inner products with $d_\star^t$ and $\bar d^t-d_\star^t$. C2--C3 then give
\begin{align*}
\alpha_T C_2\E\|\nabla f(x^t)\|_*^q
\le{}&\E f(x^t)-\E f(x^{t+1})\\
&+\alpha_T\E[C_3E_{\rm mis}^t+\nu]\\
&+\frac L2\alpha_T^2\cU^t-\alpha_T\chi_EU_{\rm loc}^t.
\end{align*}
All expectations are unconditional. Summing over $t<T$ telescopes the objective terms to $f(x^0)-\E f(x^T)\le f(x^0)-f_*$, and dividing by $T\alpha_T$ proves~\eqref{eq:master}. Finally, convexity gives $\|\bar d^t\|_{\rm pr}^2\le(SK)^{-1}\sum_{i\in\cS^t,k}\|d_i^{t,k}\|_{\rm pr}^2$, hence~\eqref{eq:expected-movement}.
\end{proof}

\subsection{Proof of Theorem~\ref{thm:lyapunov}: summed state errors}
\label{app:proof-lyapunov}
\begin{proof}
\par\smallskip\noindent\textbf{Step 1: Sum the state recursions.}\quad
For a nonnegative state satisfying $A^{t+1}\le(1-c)A^t+F^t$, summing gives $c\sum_{t<T}A^t\le A^0+\sum_{t<T}F^t$. Apply this first with $c=\refreshprob/2$ to the tracking state and then with $c=sh/2$ to the population-momentum state. Together with the drift bound this gives
\begin{align}
\overline{\cD}&\le\delta_T^2\overline U_{\rm loc},\notag\\
\overline\Xi&\le\frac{2\Xi^0}{T\refreshprob}+4\overline{\cD}
+\frac{6\alpha_T^2}{\refreshprob^2}\overline{\cU},\notag\\
\overline{\cP}&\le\frac{2\cP^0}{Tsh}
+2C_pL_g^2(\overline{\cD}+\overline\Xi)
+\frac{2C_pL_g^2\alpha_T^2}{s^2h^2}\overline{\cU}.
\label{eq:retained-state-sums}
\end{align}
Here $\overline{\cD}=T^{-1}\sum_t\E\cD^t$. At $\refreshprob=1$, weakening the
full-refresh bound gives the same displayed tracking estimate. Insert these
estimates into~\eqref{eq:local-moment-error} and use
$\overline{\cU}\le\overline U_{\rm loc}$. The coefficient of movement is
exactly~\eqref{eq:state-movement-coefficient}, and the initial-state terms
give~\eqref{eq:state-ledger-general}. This proves part~(a); it uses only
$0<a_T<1$. In particular, $\overline{\cD}+\overline\Xi\le2\Xi^0/(T\refreshprob)+5\delta_T^2\overline U_{\rm loc}+6\alpha_T^2\overline{\cU}/\refreshprob^2$. The factor 5 is the sum of the direct drift term and the four drift terms from tracking.

\par\smallskip\noindent\textbf{Step 2: Keep the reused-noise terms.}\quad
Sum~\eqref{eq:joint-filter-recursion}, with $h=\bar a_T$,
and use $\overline\Omega\le\Omega^0/(T\refreshprob)+\sigma^2/K$. Stopping the local
filter after step $k$ gives
\[
\sum_{r=0}^{k-1}a_T^2(1-a_T)^{2r}
=\frac{a_T[1-(1-a_T)^k][1+(1-a_T)^k]}{2-a_T}
\le a_Th.
\]
The three-term squared-norm inequality and
$\norm{e_i}^2\le2\norm{e_i+\theta_T u_i}^2+2\theta_T^2\norm{u_i}^2$ imply
\[
\overline{\cN}_{\rm loc}\le6\overline{\mathfrak Z}
+(6\theta_T^2+3h^2)\overline\Omega+3a_Th\sigma^2.
\]
The summed joint-state recursion is
\[
\overline{\mathfrak Z}\le\frac{\mathfrak Z^0}{Tsh}
+(24h+12sh/\refreshprob)\overline\Omega
+(6a_T+12sh/(\refreshprob K))\sigma^2.
\]
Substituting this and the tracker average into the local comparison gives
\eqref{eq:explicit-innovation-average}, which proves the first claim of
part~(b); with part~(a) it gives~\eqref{eq:state-movement-ledger}. Only this
step uses $h\le1/8$ and $sh\le\refreshprob/2$, through
\eqref{eq:joint-filter-recursion}. In particular, for fixed $K,s,\refreshprob$
and $Ka_T\le1/16$,
$J_T=O((\cN^0+\Omega^0)/(Th)+a_T\sigma^2)$.
All expectations are unconditional, and the covariance from gradient reuse
is retained in the joint-state bound.

\par\smallskip\noindent\textbf{Step 3: Apply the common descent inequality.}\quad
Use constant $\alpha_T$ and $C_2=c$ in~\eqref{eq:master}, followed by
$\cU^t\le U_{\rm loc}^t$ and~\eqref{eq:state-movement-ledger}.
Moving the resulting $C_3\mathcal Q_T\overline U_{\rm loc}$ term to the
left gives~\eqref{eq:lyap-final}; this is allowed because
$\overline U_{\rm loc}$ is finite (\cref{lem:adam-expected-movement} for the
Adam family, and at most one for normalization and Muon). With exact gradients and exact initial
trackers, $e_i^t=u_i^t=0$, so $\overline{\cN}_{\rm loc}=0$, which proves the
final assertion.
\end{proof}

\section{Proofs of the Optimizer Lemmas}
\label{sec:optimizer-instantiations}
This section proves the movement and alignment properties stated once as
Lemmas~\ref{lem:C1-main} and~\ref{lem:C23-main}. The method-specific bounds
below supply their constants and use each optimizer's update and gradient norms.

\subsection{Proof of Lemma~\ref{lem:C1-main}: bounded adaptive direction}
\label{app:proof-c1}
For the Adam family, the bound on the largest gradient coordinate and the
bound on the expected squared step are established separately below.
Normalization and the Muon maps return directions of norm at most one in their
update norms, as recorded in Appendix~\ref{app:proof-c23}. In every case, the
first inequality in~\eqref{eq:expected-movement} holds because the squared norm
of an average is at most the average of the squared norms.

\subsubsection{The event $\mathcal H_p$}
Let $d$ be the number of scalar coordinates and $N_{\rm or}$ upper-bound the
number of raw stochastic gradients drawn, including initialization; for \FA,
$N_{\rm or}=n+SKT$. Define
\begin{equation}
H_{0,p}:=G+\sigma
 \sqrt{2\log\frac{2dN_{\rm or}}p},\qquad
H_p:=3H_{0,p}.
\label{eq:Hdelta}
\end{equation}
\begin{lemma}[The event $\mathcal H_p$]\label{lem:good-event}
Let $\mathcal H_p$ be the event that every raw stochastic-gradient
coordinate has magnitude at most $H_{0,p}$. Under
\cref{assump:BG,assump:genLoss},
$\mathbb P(\mathcal H_p)\ge1-p$. On this event, every \FA\ corrected-gradient coordinate has
magnitude at most $H_p$.
\end{lemma}
\begin{proof}
\emph{Tail of one coordinate.} Let $\zeta$ be a noise vector drawn at some
step, $e_j$ the $j$th unit vector, and $z>0$. For any $\lambda>0$, Markov's
inequality applied to the nonnegative variable $e^{\lambda\ip{e_j}\zeta}$ and
\cref{assump:BG} give, given the past,
\[
\Pp(\zeta_j>z)\le e^{-\lambda z}\E[e^{\lambda\zeta_j}]
\le e^{-\lambda z+\lambda^2\sigma^2/2}.
\]
Taking $\lambda=z/\sigma^2$ gives $e^{-z^2/(2\sigma^2)}$. The same bound holds
for $-\zeta_j$, so $\Pp(\abs{\zeta_j}>z)\le2e^{-z^2/(2\sigma^2)}$, and taking
expectations removes the conditioning.

\emph{Union bound.} A run draws at most $N_{\rm or}$ gradients, each with $d$
coordinates. Which clients are selected is independent of the gradient samples,
so we may list the $dN_{\rm or}$ positions (round, client slot, local step,
coordinate) in advance. Since $\abs{\nabla f_i(x)_j}\le\norm{\nabla f_i(x)}_2\le G$,
a raw coordinate can exceed $H_{0,p}$ only if its noise exceeds
$z=\sigma\sqrt{2\log(2dN_{\rm or}/p)}$, which has probability at most
$2e^{-z^2/(2\sigma^2)}=p/(dN_{\rm or})$. The probability that this happens at
any of the $dN_{\rm or}$ positions is at most the sum of these probabilities,
which is $p$.

\emph{Corrected gradients.} A client tracker and its server average are
averages of raw stochastic gradients (or zero), so on $\mathcal H_p$ each has
coordinates of magnitude at most $H_{0,p}$. Equation~\eqref{eq:fagc-corrected}
then gives
$\abs{\widehat g_j}\le\abs{g_j}+\abs{y_j}+\abs{y_{i,j}}\le3H_{0,p}$.
\end{proof}
The event $\mathcal H_p$ supplies a lower bound on the preconditioner. The
proofs never condition on it: on its complement they use only
$\norm{\nabla f}_2\le G$, which adds an explicit term to C3, and all
expectations below are taken over every outcome.

\subsubsection{Adam-family movement bound}
For Adam,
$d_i^{t,k}=(m_i^{t,k}/\psi_i^{t,k})/(\sqrt{v_i^{t,k}}+\epsilon)$
coordinatewise; the main results use $\psi_i^{t,k}=1$. AMSGrad replaces $v$ by
its running coordinatewise maximum.
\begin{lemma}[Unconditional Adam-family movement]\label{lem:adam-expected-movement}
For persistent momentum without bias correction,
\begin{equation}
\sup_t U_{\rm loc}^t\le
\frac{C(G^2+\sigma^2)}{\epsilon^2}
\label{eq:adam-expected-movement}
\end{equation}
for Adam/AMSGrad and round-frozen SOAP, with dimension and initialization
absorbed into $C$.  For \NT, the additive tracking correction changes only
the constant.  The bias-corrected form has the same bound after the burn-in time $t_{\rm b}$ of \cref{cor:bias-corrected-momentum}
and keeps $1/(\psi_i^{t,k})^2$ explicitly before it. Thus
\cref{lem:C1-main} holds with $U_1:=\sup_t U_{\rm loc}^t$.
\end{lemma}
\begin{proof}
A raw gradient satisfies
$\E\norm{g_i^{t,k}}_2^2=\E\norm{\nabla f_i(x_i^{t,k})}_2^2+\E\norm{\zeta_i^{t,k}}_2^2\le G^2+\sigma^2$,
because the noise has mean zero given the past. A tracker $y_i^t$ is the
average of client $i$'s gradients in its last refresh round, or a fresh
initial gradient, or zero. Which round was the last refresh depends only on
the refresh indicators, which are drawn independently of the gradients
computed in that round. Hence, summing over the possible last refresh rounds
and using $\norm{K^{-1}\sum_kx_k}^2\le K^{-1}\sum_k\norm{x_k}^2$, we get
$\E\norm{y_i^t}_2^2\le G^2+\sigma^2$, and the same bound holds for the average
$y^t$. From $\norm{x_1+x_2+x_3}^2\le3(\norm{x_1}^2+\norm{x_2}^2+\norm{x_3}^2)$,
\[
\E\norm{\widehat g_i^{t,k}}_2^2\le9(G^2+\sigma^2).
\]
Without bias correction, the momentum is a weighted sum
$m_i^{t,k}=\sum_jw_j\widehat g_j$ of its past inputs with weights $w_j\ge0$ and
$\sum_jw_j\le1$, so Cauchy--Schwarz gives
$\norm{m_i^{t,k}}_2^2\le\sum_jw_j\norm{\widehat g_j}_2^2$. Which inputs
enter, and with which weights, depends only on the selection indicators, which
are independent of the counterfactual inputs (P3). Hence
$\E[w_j\norm{\widehat g_j}_2^2]=\E[w_j]\,\E\norm{\widehat g_j}_2^2$ and
$\E\norm{m_i^{t,k}}_2^2\le9(G^2+\sigma^2)$. Every coordinate of the scaling is
at most $\epsilon^{-1}$, so $\norm{d_i^{t,k}}_2\le\epsilon^{-1}\norm{m_i^{t,k}}_2$,
which gives~\eqref{eq:adam-expected-movement} with $C=9$. For SOAP, the
rotations do not change Frobenius norms.
\end{proof}

\subsection{Proof of Lemma~\ref{lem:C23-main}: alignment and mismatch}
\label{app:proof-c23}
We first record the comparison directions $d_\star^t$, then prove pointwise
inequalities for each optimizer, and finally average them over the selected
client steps. These averages are the C2 and C3 inequalities of
Lemma~\ref{lem:C23-main}. There, $E_{\rm mis}^t$ is a number that bounds the
expected momentum error, and $\nu$ is the small term from the event that
$\mathcal H_p$ fails. The constant $C_2$ need not be at least one; for Adam
it contains the denominator scale.
The comparison directions are fixed before round $t$'s samples are drawn:
\begin{equation}
\begin{array}{c|c}
\text{optimizer class}&d_\star^t\\ \hline
\text{Adam, AMSGrad, SOAP (positive preconditioner)}&(\lambda_p/2)\nabla f(x^t)\\
\text{reference-denominator comparison}&(2\epsilon)^{-1}\nabla f(x^t)\\
\text{gradient normalization}&N(\nabla f(x^t))\\
\text{exact Muon}&P_{\rm M}(\nabla f(x^t))\\
\text{finite-step Muon}&r^{-1/2}P_{\rm M}(\nabla f(x^t)).
\end{array}
\label{eq:predictable-surrogates}
\end{equation}
All C3 statements concern the average over selected client steps; the
average of the errors is bounded, not the error of each step.

\subsubsection{Adam and AMSGrad alignment}
\begin{lemma}[Positive-preconditioner alignment]\label{lem:adam-c23}
Let $D$ be diagonal with every entry in $(0,\epsilon^{-1}]$, and consider the
direction $Dm$. For all vectors $g$ and $m$,
\begin{equation}
\ip g{Dm}\ge\frac12\ip g{Dg}
+\frac\epsilon2\norm{Dm}^2-\frac1{2\epsilon}\norm{m-g}^2.
\label{eq:adam-c3}
\end{equation}
For FA-GC Adam/AMSGrad with the initialization of
Theorem~\ref{thm:FA-GC}, every entry of $D$ is at least $\lambda_p$ on
$\mathcal H_p$. Consequently, for $g=g^t$, on every outcome,
\begin{equation}
\ip g{Dm}\ge\frac{\lambda_p}{2}\norm g^2
+\frac\epsilon2\norm{Dm}^2-\frac1{2\epsilon}\norm{m-g}^2
-\frac{\lambda_pG^2}{2}\mathbf1_{\mathcal H_p^c}.
\label{eq:positive-preconditioner-tail}
\end{equation}
No independence between $D,m,g$ is required.
\end{lemma}
\begin{proof}
Work one coordinate at a time, with entries $D_j\in(0,\epsilon^{-1}]$.
Expanding the square gives
\[
2g_jD_jm_j=D_jg_j^2+D_jm_j^2-D_j(m_j-g_j)^2.
\]
Since $D_j\le\epsilon^{-1}$, we have $D_j\ge\epsilon D_j^2$, so
$D_jm_j^2\ge\epsilon(D_jm_j)^2$, and
$D_j(m_j-g_j)^2\le\epsilon^{-1}(m_j-g_j)^2$. Summing over $j$ and dividing by
two proves~\eqref{eq:adam-c3}.
On $\mathcal H_p$, corrected coordinates are bounded by $3H_{0,p}$
(\cref{lem:good-event}). A zero-initialized second moment is a weighted sum of
squared inputs whose weights sum to at most one, so each coordinate of $v$ is at
most $(3H_{0,p})^2$; the AMSGrad maximum keeps this bound. Adam's
second-moment bias correction divides $v$ by $1-(1-b_T)^{\ell_i^t+k}$, where
$\ell_i^t+k$ is the client's number of updates so far; this makes the weights sum to exactly one,
so the same bound holds. Hence every entry of
$D$ is at least $(3H_{0,p}+\epsilon)^{-1}=\lambda_p$ there, and
$\ip g{Dg}\ge\lambda_p\norm g^2$. On the complement, $\ip g{Dg}\ge0$ and
$\norm g\le G$, so $\ip g{Dg}\ge\lambda_p\norm g^2-\lambda_pG^2$. This proves
the global inequality~\eqref{eq:positive-preconditioner-tail}.
\end{proof}

\begin{lemma}[Reference-denominator comparison]\label{lem:adam-reference}
Let $D_{\rm ref}=\epsilon^{-1}I$, $\mu=m/\psi$, and
$r_v=[D(v)-D_{\rm ref}]\mu$. Then
\begin{equation}
\ip g{D(v)\mu}\ge\frac{1}{2\epsilon}\norm g^2
-\frac1\epsilon\norm{\mu-g}^2-\epsilon\norm{r_v}^2.
\label{eq:adam-reference-c3}
\end{equation}
This gives an alternative C3 instance with $\chi_E=0$ and comparison direction
$(2\epsilon)^{-1}g$, used for the FA-NT and fixed-gain comparisons.
\end{lemma}
\begin{proof}
Write $D(v)\mu=\epsilon^{-1}g+\epsilon^{-1}(\mu-g)+r_v$, so
$\ip g{D(v)\mu}=\epsilon^{-1}\norm g^2+\epsilon^{-1}\ip g{\mu-g}+\ip g{r_v}$.
By $\abs{\ip xy}\le\frac c2\norm x^2+\frac1{2c}\norm y^2$ with suitable $c$,
\[
\epsilon^{-1}\abs{\ip g{\mu-g}}\le\frac{\norm g^2}{4\epsilon}+\frac{\norm{\mu-g}^2}{\epsilon},\qquad
\abs{\ip g{r_v}}\le\frac{\norm g^2}{4\epsilon}+\epsilon\norm{r_v}^2.
\]
Subtracting both bounds proves~\eqref{eq:adam-reference-c3}.
\end{proof}

\subsubsection{Round-frozen SOAP}
\begin{proposition}[Round-frozen SOAP]\label{prop:soap}
Let orthogonal $Q_L^t,Q_R^t$ be chosen using only information available before round $t$ and fixed during round
$t$. Store $m_i^t$ in ambient coordinates and, at a selected step, form
\begin{align}
 \widehat g_{i,\rm rot}^{t,k}&=(Q_L^t)^{\top}\widehat g_i^{t,k}Q_R^t,\\
 m_i^{t,k}&=(1-a_T)m_i^{t,k-1}+a_T\widehat g_i^{t,k},\\
 v_i^{t,k}&=(1-b_T)v_i^{t,k-1}+b_T(\widehat g_{i,\rm rot}^{t,k})^{\odot 2},\\
 m_{i,\mathrm{rot}}^{t,k}&=(Q_L^t)^{\top} m_i^{t,k}Q_R^t,\\
 d_i^{t,k}&=Q_L^t
 \frac{m_{i,\mathrm{rot}}^{t,k}}{\sqrt{v_i^{t,k}}+\epsilon}(Q_R^t)^{\top} .
 \label{eq:soap-update}
\end{align}
The numerical diagonal $v_i^t$ persists and is reinterpreted in the new
predictable basis at the next round. Unselected clients retain ambient $m_i^t$
and numerical $v_i^t$. Then \cref{lem:adam-c23} holds in Frobenius geometry
with $c_{\rm rot}=\sqrt d$ in~\eqref{eq:main-preconditioner-envelope}.
The reference comparison~\cref{lem:adam-reference} also holds.
\end{proposition}
\begin{proof}
Write $\mathcal O(X)=(Q_L^t)^{\top} XQ_R^t$ for the rotation. It preserves
Frobenius inner products and norms, $\ip{\mathcal O(X)}{\mathcal O(Y)}=\ip XY$,
and its inverse is $X\mapsto Q_L^tX(Q_R^t)^{\top}$. With
$D=\diag((\sqrt{v_i^{t,k}}+\epsilon)^{-1})$ acting entrywise on rotated
matrices,
\[
\ip g{d_i^{t,k}}=\ip{\mathcal O(g)}{D\,\mathcal O(m_i^{t,k})},\qquad
\norm{d_i^{t,k}}_F=\norm{D\,\mathcal O(m_i^{t,k})}_F .
\]
Apply \cref{lem:adam-c23} to the rotated matrices $\mathcal O(g)$ and
$\mathcal O(m_i^{t,k})$, then undo the rotation in every norm. On
$\mathcal H_p$, every corrected matrix has Frobenius norm at most
$3\sqrt dH_{0,p}$, so every rotated coordinate has magnitude at most
$3\sqrt dH_{0,p}$, whatever the basis. Hence each entry of $v$ is at most
$(3\sqrt dH_{0,p})^2$ and each entry of $D$ is at least $\lambda_p$ with
$c_{\rm rot}=\sqrt d$. The ambient first-moment recursion is unchanged, so the
common population and noise lemmas apply. Since $D_{\rm ref}=\epsilon^{-1}I$ is
unchanged by rotation, the same argument proves~\cref{lem:adam-reference}.
\end{proof}

\paragraph{How the Adam family enters Theorem~\ref{thm:FA-GC}}
Adam and AMSGrad use~\cref{lem:adam-c23}, while the stated round-frozen SOAP
update uses~\cref{prop:soap}. Their averaged C3 bound in
\cref{prop:optimizer-bridge} supplies the movement gain $\chi_E=\epsilon/2$
used in the proof of Theorem~\ref{thm:FA-GC} (Appendix~\ref{app:proof-fagc}).
The rate in~\eqref{eq:main-explicit-residual} concerns persistent
moments without bias correction. The optional bias-correction factors $\psi_i^{t,k}$ are treated separately
in Appendix~\ref{app:extensions}.

\subsubsection{Gradient normalization}
Let $N(u)=u/\norm u_2$ for $u\ne0$ and $N(0)=0$.
\begin{lemma}[Gradient-normalization alignment]\label{lem:normalization}
For $g\ne0$ and any $m$,
\begin{equation}
\ip g{N(m)}\ge\norm g_2-2\norm{m-g}_2.
\label{eq:norm-alignment}
\end{equation}
Thus $C_1=1$, $q=1$, $C_2=1$, and $C_3=2$.
The comparison direction is $N(\nabla f(x^t))$ as in
\eqref{eq:predictable-surrogates}.
For $g=0$ the same bound is trivial.
\end{lemma}
\begin{proof}
First, for $v\ne0$ and any $u$, $\norm{N(u)-N(v)}_2\le2\norm{u-v}_2/\norm v_2$.
If $u=0$, the left side is $1=\norm v_2/\norm v_2$. If $u\ne0$,
\[
\Big\|\frac u{\norm u}-\frac v{\norm v}\Big\|
\le\Big\|\frac u{\norm u}-\frac u{\norm v}\Big\|+\Big\|\frac{u-v}{\norm v}\Big\|
=\frac{\abs{\norm v-\norm u}}{\norm v}+\frac{\norm{u-v}}{\norm v}
\le\frac{2\norm{u-v}}{\norm v}.
\]
Then $\ip g{N(m)}=\ip g{N(g)}+\ip g{N(m)-N(g)}\ge\norm g_2-\norm g_2\cdot2\norm{m-g}_2/\norm g_2$.
\end{proof}
With $m=m_i^{t,k}$, averaging this bound gives the C3 term
$2\sqrt{\cM_{\rm GC}^t}$ (\cref{prop:optimizer-bridge}).

\paragraph{How normalization enters Theorem~\ref{thm:FA-GC}}
\Cref{lem:normalization} supplies C2--C3 with $q=1$ and no denominator
residual. With the $q=1$ schedule~\eqref{eq:constructive-schedule}, the proof
in Appendix~\ref{app:proof-fagc} then gives
\begin{equation}
\frac1T\sum_{t<T}\E\norm{\nabla f(x^t)}_2=O(T^{-1/4}).
\label{eq:normalization-rate}
\end{equation}

\subsubsection{Exact and finite-step Muon}
\emph{Matrix facts.} Every $d_1\times d_2$ matrix has a singular value
decomposition $M=\sum_{j=1}^{r}\sigma_j u_jv_j^\top=U\Sigma V^\top$ with
$r=\min(d_1,d_2)$, singular values $\sigma_j\ge0$, and orthonormal columns
$u_j$ and $v_j$; the compact form keeps only $\sigma_j>0$. The spectral norm is
$\norm M_{\rm sp}=\max_j\sigma_j$, the nuclear norm is
$\norm M_{\rm tr}=\sum_j\sigma_j$, and $\norm M_F^2=\sum_j\sigma_j^2$. We use
four facts.
\begin{enumerate}[label=(F\arabic*)]
\item $\ip AB\le\norm A_{\rm tr}\norm B_{\rm sp}$: write
$\ip AB=\sum_j\sigma_j(A)\,u_j^\top Bv_j$ and use
$\abs{u^\top Bv}\le\norm B_{\rm sp}$ for unit vectors $u,v$. This holds
because, writing $B=\sum_j\sigma_j(B)a_jb_j^\top$ with orthonormal $\{a_j\}$
and $\{b_j\}$, $\norm{Bv}_2^2=\sum_j\sigma_j(B)^2(b_j^\top v)^2
\le\norm B_{\rm sp}^2\sum_j(b_j^\top v)^2\le\norm B_{\rm sp}^2$, since the
$b_j$ are orthonormal; then $\abs{u^\top Bv}\le\norm u_2\norm{Bv}_2$ by
Cauchy--Schwarz.
\item $\ip M{UV^\top}=\sum_j\sigma_j=\norm M_{\rm tr}$.
\item $\norm M_{\rm tr}\le\sqrt r\norm M_F$, by Cauchy--Schwarz on
$(\sigma_1,\ldots,\sigma_r)$.
\item If $Z=U\diag(\varsigma)V^\top$, then $ZZ^\top Z=U\diag(\varsigma^3)V^\top$, so a matrix
polynomial with odd powers acts on the singular values and keeps the singular
vectors.
\end{enumerate}
Fact (F1) is what makes the nuclear norm the gradient norm that matches the
spectral update norm.

Define $P_{\rm M}(M)=UV^\top$ using the compact
SVD and set $P_{\rm M}(0)=0$. Under the usual LMO,
$\operatorname{lmo}(M)=-UV^\top$, so $x^+=x+\eta\operatorname{lmo}(M)$ is
identical to $x^+=x-\eta P_{\rm M}(M)$.
\begin{lemma}[Exact Muon alignment]\label{lem:muon-exact}
For matrices $g,M$ of the same size and any $c>0$,
\[
\ip g{P_{\rm M}(M)}
\ge\norm g_{\mathrm{tr}}-\frac{2}{c}\norm{M-cg}_{\mathrm{tr}},\qquad
\norm{P_{\rm M}(M)}_{\mathrm{sp}}\le1.
\]
Thus C1 is one in spectral geometry, $q=1$, and stationarity is nuclear norm.
The comparison direction is $P_{\rm M}(\nabla f(x^t))$.
\end{lemma}
\begin{proof}
All singular values of $UV^\top$ equal one, so $\norm{P_{\rm M}(M)}_{\rm sp}\le1$.
Write $g=M/c+(g-M/c)$. By (F2) and (F1),
\[
\ip g{P_{\rm M}(M)}\ge\frac{\norm M_{\rm tr}}c-\frac{\norm{M-cg}_{\rm tr}}c .
\]
The triangle inequality gives $\norm M_{\rm tr}\ge c\norm g_{\rm tr}-\norm{M-cg}_{\rm tr}$,
which proves the claim.
\end{proof}
\begin{lemma}[Finite-step Newton--Schulz nuclear alignment]
\label{lem:muon-inexact}
Let $r=\min(d_1,d_2)$, set $P_R(0)=0$, and for $M\ne0$ initialize
$Z_0=M/\norm M_F$. Apply $R$
singular-value iterations
\begin{equation}
p_{\rm NS}(z)=\frac{15}{8}z-\frac54z^3+\frac38z^5,\qquad
Z_{j+1}=\frac{15}{8}Z_j-\frac54Z_jZ_j^\top Z_j
+\frac38Z_jZ_j^\top Z_jZ_j^\top Z_j.
\label{eq:ns-iteration}
\end{equation}
Set $P_R(M)=Z_R$. Then
\begin{align}
\norm{P_R(M)}_{\rm sp}&\le1,\\
\ip M{P_R(M)}
&=\sum_j\sigma_j(M)p_{\rm NS}^{\circ R}
 (\sigma_j(M)/\norm M_F)
\ge\norm M_F
\ge r^{-1/2}\norm M_{\rm tr}.
\label{eq:ns-rank-alignment}
\end{align}
Here $p_{\rm NS}^{\circ R}$ denotes $p_{\rm NS}$ applied $R$ times.
Consequently
\begin{equation}
\ip g{P_R(M)}\ge r^{-1/2}\norm g_{\rm tr}
-\frac{1+r^{-1/2}}{c}\norm{M-cg}_{\rm tr}
\label{eq:ns-mismatch}
\end{equation}
for any $c>0$. Thus the comparison direction may be taken as
$r^{-1/2}P_{\rm M}(g)$, $q=1$, with nuclear-gradient stationarity and a
rank-dependent constant $C_3$. The lower bound is uniform in the iteration count,
so it does not quantify the benefit of additional Newton--Schulz steps.
\end{lemma}
\begin{proof}
For $z\in[0,1]$,
\[
p_{\rm NS}(z)-z=\frac z8(1-z^2)(7-3z^2)\ge0,\qquad
p_{\rm NS}'(z)=\frac{15}8(1-z^2)^2\ge0,\qquad p_{\rm NS}(1)=1,
\]
so $z\le p_{\rm NS}(z)\le1$. The singular values of $Z_0=M/\norm M_F$ are
$\sigma_j(M)/\norm M_F\in[0,1]$. By (F4) and induction, $Z_j$ has the singular
vectors of $M$ and singular values $p_{\rm NS}^{\circ j}(\sigma_j(M)/\norm M_F)$,
which stay in $[0,1]$ and never decrease. Hence $\norm{P_R(M)}_{\rm sp}\le1$, and
\[
\ip M{P_R(M)}=\sum_j\sigma_j(M)\,p_{\rm NS}^{\circ R}\Big(\frac{\sigma_j(M)}{\norm M_F}\Big)
\ge\sum_j\frac{\sigma_j(M)^2}{\norm M_F}=\norm M_F,
\]
which proves \eqref{eq:ns-rank-alignment} together with (F3). Finally write
$g=M/c+(g-M/c)$. By (F1),
$\ip g{P_R(M)}\ge c^{-1}\norm M_F-c^{-1}\norm{M-cg}_{\rm tr}
\ge c^{-1}r^{-1/2}\norm M_{\rm tr}-c^{-1}\norm{M-cg}_{\rm tr}$, and
$\norm M_{\rm tr}\ge c\norm g_{\rm tr}-\norm{M-cg}_{\rm tr}$ gives~\eqref{eq:ns-mismatch}.
\end{proof}

\paragraph{How Muon enters Theorem~\ref{thm:FA-GC}}
\Cref{lem:muon-exact,lem:muon-inexact} supply C2--C3 with $q=1$, the nuclear
norm as gradient norm, and no denominator residual. With spectral smoothness
and the $q=1$ schedule~\eqref{eq:constructive-schedule}, the proof in
Appendix~\ref{app:proof-fagc} gives the rate~\eqref{eq:main-first-power-rate}
for exact Muon, and the same rate with a rank-dependent constant for
finite-step Newton--Schulz. If the momentum error is measured in the Frobenius
norm, (F3) converts it to the nuclear norm. If only Frobenius smoothness with
constant $L_F$ is known, $\norm \Delta_F^2\le r\norm \Delta_{\rm sp}^2$ for every update $\Delta$ gives spectral
smoothness with $L=rL_F$.

\subsubsection{Selected client-step averages for FA-GC}
\begin{proposition}[Optimizer-to-master bridge for \FA]
\label{prop:optimizer-bridge}
Use the primary momentum without bias correction, $\psi_i^{t,k}=1$, and let
$\cM_{\rm GC}^t$ be the selected client--step error defined after
\eqref{eq:population-moment-state}. For that average:
\begin{enumerate}[label=(\roman*)]
\item Adam, AMSGrad, and the specified round-frozen SOAP update satisfy
$q=2$, $C_2=\lambda_p/2$, and
\begin{equation}
 -\E\ip{g^t}{\bar d^t-d_\star^t}
 \le \frac{\cM_{\rm GC}^t}{2\epsilon}
 -\frac\epsilon2 U_{\rm loc}^t+\frac{\lambda_pG^2p}{2}.
 \label{eq:adam-averaged-c3}
\end{equation}
Thus $\chi_E=\epsilon/2$.
\item Gradient normalization satisfies $q=1$, $C_2=1$, and
\begin{equation}
 -\E\ip{g^t}{\bar d^t-d_\star^t}
 \le2\sqrt{\cM_{\rm GC}^t}.
 \label{eq:norm-averaged-c3}
\end{equation}
\item Exact Muon satisfies $q=1$, $C_2=1$, while finite-step Muon satisfies
$q=1$, $C_2=r^{-1/2}$. Their respective C3 terms are at most
\begin{equation}
 2\sqrt{r\cM_{\rm GC}^t},\qquad
 (1+r^{-1/2})\sqrt{r\cM_{\rm GC}^t}.
 \label{eq:muon-averaged-c3}
\end{equation}
\end{enumerate}
For normalization and Muon, take $\chi_E=0$ and $\nu=0$. The
unconditional local direction moment is uniformly bounded for every listed
class. These are the C1--C3 inputs to the common descent theorem.
\end{proposition}
\begin{proof}
Average \eqref{eq:positive-preconditioner-tail} and use $\Pp(\mathcal H_p^c)\le p$ to obtain
\eqref{eq:adam-averaged-c3}. Average \eqref{eq:norm-alignment}, then use
Cauchy--Schwarz on the client--step average, to obtain
\eqref{eq:norm-averaged-c3}.  Apply the same argument to
\cref{lem:muon-exact,lem:muon-inexact} and use
$\norm A_{\rm tr}\le\sqrt r\norm A_F$ to obtain
\eqref{eq:muon-averaged-c3}. Adam-family moments are bounded by
\cref{lem:adam-expected-movement}; normalization and both Muon maps have
movement at most one in the chosen norm for every realization. Every norm is taken before averaging.
\end{proof}

\subsection{Optimizer comparison for FA-NT}
\label{app:fant-optimizer}
\begin{proposition}[\NT\ optimizer-to-master bridge]\label{prop:fant-bridge}
Let $\cM_{\rm NT}^t$ be the local-reference error in
\eqref{eq:nt-local-error}.  After averaging selected client--steps, and before
adding the zero-mean correction:
\begin{align}
\text{Adam family:}\quad
 \E\left[\frac1{SK}\sum_{i\in\cS^t,k}\norm{m_i^{t,k}-g^t}_2^2\right]
 &\le2\cM_{\rm NT}^t+2B^2,
 \label{eq:nt-adam-bridge}\\
\text{normalization:}\quad
 \E\left[\frac1{SK}\sum_{i\in\cS^t,k}\norm{m_i^{t,k}-g^t}_2\right]
 &\le\sqrt{\cM_{\rm NT}^t}+B,
 \label{eq:nt-norm-bridge}\\
\text{Muon:}\quad
 \E\left[\frac1{SK}\sum_{i\in\cS^t,k}\norm{m_i^{t,k}-g^t}_{\rm tr}\right]
 &\le\sqrt r\{\sqrt{\cM_{\rm NT}^t}+B\}.
 \label{eq:nt-muon-bridge}
\end{align}
Thus \cref{lem:adam-reference,lem:normalization,lem:muon-exact,lem:muon-inexact}
give respectively $CB^2$ for $q=2$ and $CB$ for $q=1$, with the same
denominator residual for Adam/AMSGrad/SOAP and the rank constants displayed in
\eqref{eq:muon-averaged-c3}.  The additive correction has zero first-order
conditional mean. Moreover \cref{lem:nt-tracker-energy} supplies the uniform
direction moment used by smoothness; $P_*<\infty$ follows from
\cref{lem:adam-expected-movement} for the Adam family and unit movement for every realization
for normalization/Muon.
\end{proposition}
\begin{proof}
For every client--step,
\[
 \norm{m_i-g^t}_2^2\le2\norm{m_i-\nabla f_i(x^t)}_2^2
 +2\norm{\nabla f_i(x^t)-g^t}_2^2.
\]
Average and use \cref{assump:BDH} to prove
\eqref{eq:nt-adam-bridge}.  The unsquared triangle inequality, Cauchy--Schwarz
on the selected average, and the same assumption prove
\eqref{eq:nt-norm-bridge}; combine this with
$\norm A_{\rm tr}\le\sqrt r\norm A_F$ for
\eqref{eq:nt-muon-bridge}. Equations \eqref{eq:nt-zero-mean} and
\eqref{eq:nt-finite-pop} in Appendix~\ref{app:proof-fant} then isolate the correction exactly.
\end{proof}

\section{Supporting Technical Lemmas}
\label{app:supporting}
These auxiliary calculations support the core and optimizer proofs. Their
statements and short derivations are kept together for reference.

\subsection{Server identities and sampling}
\begin{lemma}[Server-average invariants]\label{lem:server-identities}
For both algorithms,
\begin{align}
y^t&=\frac1n\sum_{i=1}^n y_i^t,\label{eq:y-invariant}\\
x^{t+1}&=x^t-\alpha_T\bar d^t,\qquad
\bar d^t:=\frac1{SK}\sum_{i\in\cS^t}\sum_{k=1}^K d_i^{t,k}.
\label{eq:server-direction}
\end{align}
\end{lemma}
\begin{proof}
The first identity follows by induction from \eqref{eq:server-tracker}.
Substituting the local recursion into \eqref{eq:server-model} gives
$x_i^{t,K+1}-x^t=-\eta_{\ell,T}\sum_kd_i^{t,k}$, proving the second identity.
\end{proof}
\begin{lemma}[Averaging the training and refresh indicators]
\label{lem:mask-average}
Conditional on $\F^t$, let $(H_i)_{i=1}^n$ be quantities with finite means that are computed without looking at the current training/refresh set pair (P3). Then
\begin{align}
\E\!\left[\frac1S\sum_i I_i^tH_i\,\middle|\,\F^t\right]
&=\frac1n\sum_i\E[H_i\mid\F^t],\notag\\
\E\!\left[\frac1n\sum_i\refreshmask_i^tH_i\,\middle|\,\F^t\right]
&=\frac{\refreshprob}{n}\sum_i\E[H_i\mid\F^t].
\label{eq:mask-average}
\end{align}
The identities apply to old client states and to counterfactual local paths generated from the round-start states and the current round's minibatch draws. A common quantity $V^t$ that depends on the selected sets instead satisfies the pathwise identities
\[
\frac1n\sum_i\refreshmask_i^tV^t=\refreshprob V^t,
\qquad \frac1n\sum_i(1-\refreshmask_i^t)V^t=(1-\refreshprob)V^t.
\]
\end{lemma}
\begin{proof}
Condition on $\F^t$ and the family $(H_i)_i$. Independence leaves each training indicator with mean $s=S/n$ and each refresh indicator with mean $\refreshprob=\widetilde S/n$. Linearity of expectation gives~\eqref{eq:mask-average}. For the last identities use $\sum_i\refreshmask_i^t=\widetilde S$ on every outcome. No independence between the two nested sets is needed.
\end{proof}

\begin{lemma}[Finite-population variance]\label{lem:finite-pop}
Let $z_1,\ldots,z_n$ be vectors (or matrices, with the Frobenius inner product) and $\bar z=n^{-1}\sum_i z_i$.
For $n\ge2$ and a uniform subset $\cS$ of size $S$,
\begin{equation}
\E\norm{\frac1S\sum_{i\in\cS}z_i-\bar z}_2^2
=\frac{n-S}{S(n-1)}\frac1n\sum_{i=1}^n\norm{z_i-\bar z}_2^2.
\label{eq:finite-pop}
\end{equation}
For $n=1$ the left side is identically zero. In particular, the sampling term
is zero whenever $S=n$. For matrices, a spectral-norm sampling estimate follows
from this Frobenius identity and $\norm A_{\rm sp}\le\norm A_F$.
\end{lemma}
\begin{proof}
Let $I_i$ indicate $i\in\cS$ and $w_i=z_i-\bar z$, so $\sum_iw_i=0$ and
$S^{-1}\sum_{i\in\cS}z_i-\bar z=S^{-1}\sum_iI_iw_i$. A uniform subset gives
$\E I_i=S/n$ and $\E I_iI_j=S(S-1)/(n(n-1))$ for $i\ne j$. Expanding the square,
\[
\E\Big\|\frac1S\sum_iI_iw_i\Big\|^2
=\frac1{S^2}\Big[\frac Sn\sum_i\norm{w_i}^2
+\frac{S(S-1)}{n(n-1)}\sum_{i\ne j}\ip{w_i}{w_j}\Big].
\]
Since $0=\norm{\sum_iw_i}^2=\sum_i\norm{w_i}^2+\sum_{i\ne j}\ip{w_i}{w_j}$, the
cross sum equals $-\sum_i\norm{w_i}^2$, and the bracket becomes
$\frac Sn\big(1-\frac{S-1}{n-1}\big)\sum_i\norm{w_i}^2$. This
is~\eqref{eq:finite-pop}.
\end{proof}
\subsection{Fourth-moment estimates}
\label{app:fourth-moments}
\begin{lemma}[Derived fourth moments]\label{lem:derived-fourth}
Let $H_4^4:=G^4+d^2\sigma^4$. Under \cref{assump:BG,assump:genLoss}, uniformly
over clients, rounds, and local steps,
\begin{equation}
 \E\norm{\widehat g_i^{t,k}}_2^4+\E\norm{m_i^{t,k}}_2^4
 \le 2^{16} H_4^4 .
 \label{eq:derived-fourth}
\end{equation}
The same statement holds for the raw-gradient \NT\ moments and for the
round-frozen rotated SOAP inputs.
\end{lemma}
\begin{proof}
For real $x$ and $\lambda>0$, the exponential series gives
$(\lambda\abs x)^4/24\le e^{\lambda\abs x}\le e^{\lambda x}+e^{-\lambda x}$.
Apply this to the coordinate $\zeta_j=\ip{e_j}\zeta$ and take expectations.
Assumption~\ref{assump:BG} with the unit vectors $\pm e_j$ gives
$\E e^{\pm\lambda\zeta_j}\le e^{\lambda^2\sigma^2/2}$, so
$\E\zeta_j^4\le48\lambda^{-4}e^{\lambda^2\sigma^2/2}$. Choosing $\lambda=2/\sigma$
and then using Cauchy--Schwarz on $\norm\zeta_2^2=\sum_j\zeta_j^2$,
\[
\E\zeta_j^4\le3e^2\sigma^4<23\sigma^4,
\qquad
\E\|\zeta\|_2^4\le d\sum_j\E\zeta_j^4\le23d^2\sigma^4.
\]
Population gradients and the population
parts of the trackers have norm at most $G$. At initialization, the tracker
noise is either sub-Gaussian, so the same calculation bounds its fourth moment
by $23d^2\sigma^4$, or equals $-\nabla f_i(x^0)$ for zero trackers, with fourth
power at most $G^4$. After the first
refresh, the noise in a tracker is the average of the $K$ noises drawn in its
last refresh round. Which round that is depends only on the refresh
indicators, which are independent of the noises drawn in that round, and
$\norm{K^{-1}\sum_kx_k}^4\le K^{-1}\sum_k\norm{x_k}^4$. Hence every tracker noise has
$\E\norm{\omega_i^t}^4\le23H_4^4$. By $\norm{a+b}^4\le8(\norm a^4+\norm b^4)$
and convexity for $\bar\omega^t$, the contrast satisfies
$\E\norm{u_i^t}^4\le16\cdot23H_4^4=368H_4^4$.
The decomposition
$\widehat g_i^{t,k}=\nabla f_i(x_i^{t,k})+y_{\rm pop}^t-y_{i,\rm pop}^t
+\zeta_i^{t,k}+u_i^t$ and
$\norm{x_1+\cdots+x_4}^4\le4^3\sum_j\norm{x_j}^4$ then give
$\E\norm{\widehat g_i^{t,k}}^4\le64(G^4+16G^4+23d^2\sigma^4+368H_4^4)
\le26112H_4^4$.
A persistent EMA without bias correction is a weighted sum of its input
history with nonnegative weights summing to at most one. The weights depend
only on the selection indicators, which are independent of the counterfactual
inputs (P3), so the expectation of each weighted term factors.
Thus convexity of $\norm{\cdot}_2^4$ gives the same bound for the momentum,
and the sum is at most $52224H_4^4<2^{16}H_4^4$. Orthogonal SOAP
rotations preserve the Frobenius norm. The argument does not assume that
stochastic gradients are bounded.
\end{proof}

\subsection{Geometric and square-root convolutions}
\begin{lemma}[Geometric and square-root convolution]\label{lem:convolution}
If $A^{t+1}\le(1-c)A^t+F^t$ with $A^t,F^t\ge0$ and $c\in(0,1]$, then
\begin{align}
A^t&\le(1-c)^tA^0+\sum_{r<t}(1-c)^{t-1-r}F^r,
\label{eq:pointwise-convolution}\\
\frac1T\sum_{t<T}A^t
&\le\frac{A^0}{Tc}+\frac1{Tc}\sum_{r<T}F^r,
\label{eq:mean-convolution}\\
\frac1T\sum_{t<T}\sqrt{A^t}
&\le\frac{2\sqrt{A^0}}{Tc}
+\sqrt{\frac1{Tc}\sum_{r<T}F^r}.
\label{eq:sqrt-convolution}
\end{align}
\end{lemma}
\begin{proof}
Unroll the recursion for \eqref{eq:pointwise-convolution} and sum the
geometric series for \eqref{eq:mean-convolution}.  For the last display,
write $B^t=\sum_{r<t}(1-c)^{t-1-r}F^r$ and separate the homogeneous term:
\[
\frac1T\sum_{t<T}\sqrt{A^t}
\le\frac{\sqrt{A^0}}{T(1-\sqrt{1-c})}
+\frac1T\sum_{t<T}\sqrt{B^t}
\le\frac{2\sqrt{A^0}}{Tc}+\sqrt{\frac1T\sum_{t<T}B^t}.
\]
The second step uses $\sqrt{1-c}\le1-c/2$ and Jensen's inequality. Since $\sum_tB^t\le c^{-1}\sum_rF^r$, this proves~\eqref{eq:sqrt-convolution} without putting the initial transient inside the square root.
\end{proof}

\subsection{Feasibility of the horizon schedules}
\label{app:schedule-feasibility}
\begin{lemma}[Schedules for the shared descent bound]\label{lem:absorption-schedule}
Let $c_a=\min\{1/16,\refreshprob/(8s)\}$ and
$a_T=c_a/[K\sqrt{T+1}]$. Then
\[
\tfrac78Ka_T\le h:=1-(1-a_T)^K\le Ka_T\le1/16,
\qquad sh\le\refreshprob/8,\qquad C_p=9Ka_T/h<11.
\]
For Theorem~\ref{thm:FA-GC}(i) ($q=2$) choose $\alpha_T=\delta_T=c_\eta h$ with
$c_\eta$ from~\eqref{eq:main-absorption-schedule}. Then
\[
L\alpha_T\le\epsilon/2,\qquad
\mathcal Q_T\le\epsilon^2/4,
\]
and the movement coefficient in~\eqref{eq:lyap-final} is at least
$\epsilon/8$ for $C_3=(2\epsilon)^{-1}$ and $\chi_E=\epsilon/2$.
The second-moment gain is any fixed $b\in(0,1)$.
For Theorem~\ref{thm:FA-GC}(ii) ($q=1$) take $\alpha_T=\delta_T=(T+1)^{-3/4}$.
Both choices are realized by $\eta_{g,T}=1$ and
$\eta_{\ell,T}=\alpha_T/K$.
\end{lemma}
\begin{proof}
Since $(1-a_T)^j\ge1-ja_T$, the geometric-sum identity gives
$h=a_T\sum_{j<K}(1-a_T)^j\ge Ka_T(1-(K-1)a_T)\ge\frac{15}{16}Ka_T\ge\frac78Ka_T$,
because $(K-1)a_T\le Ka_T\le1/16$. The upper bound $h\le Ka_T$ follows from
$(1-a_T)^j\le1$. The choice of $c_a$ then gives $sh\le sKa_T\le\refreshprob/8$
and $C_p=9Ka_T/h\le72/7<11$. Hence $4C_p+12\le372/7<56$ and
$4C_p\le288/7<44$. With $\delta_T=\alpha_T$,~\eqref{eq:state-movement-coefficient} gives
\[
\mathcal Q_T\le L_g^2\alpha_T^2\Big[56\Big(5+\frac6{\refreshprob^2}\Big)+\frac{44}{s^2h^2}\Big]
\le L_g^2\Big(\frac{\alpha_T}h\Big)^2
\Big(280+\frac{336}{\refreshprob^2}+\frac{44}{s^2}\Big),
\]
where the last step uses $h\le1$. With $\alpha_T=c_\eta h$ and
$c_\eta\le\epsilon/(2L_g\sqrt{C_{s,\refreshprob}})$, the right side is at most
$\epsilon^2/4$. Also $L\alpha_T\le Lc_\eta\le\epsilon/2$. Thus
$\epsilon/2-L\alpha_T/2-\mathcal Q_T/(2\epsilon)\ge\epsilon/2-\epsilon/4-\epsilon/8=\epsilon/8$.
Part~(ii) needs no absorption restriction because smoothness
movement remains an explicit vanishing term.
\end{proof}

\paragraph{Schedules for the FA-NT bound}
The FA-NT rate uses the following schedule with $\chi_E=0$. It is separate
from the schedule of Theorem~\ref{thm:FA-GC}(i) above. For $q\in\{1,2\}$ define

\begin{equation}
u_q:=\frac{2}{q+3},\qquad v_q:=\frac{3}{q+3}.
\label{eq:q-exponents}
\end{equation}

\begin{lemma}[Feasible FA-NT schedules]
\label{lem:schedule-feasibility}
Let $T\ge1$, $K\ge1$, and $q\in\{1,2\}$. Define
\begin{equation}
\begin{aligned}
c_a&:=\min\left\{\frac1{16},\frac{\refreshprob}{8s}\right\},\\
c_b&:=\min\left\{\frac1{16},\frac{c_a}{4}\right\},\\
c_\alpha&:=1,
\end{aligned}
\label{eq:schedule-prefactors}
\end{equation}
Set
\begin{equation}
a_T=\frac{c_a}{K}(T+1)^{-u_q},\qquad
\alpha_T=c_\alpha(T+1)^{-v_q}.
\label{eq:constructive-schedule}
\end{equation}
For an Adam-family method also set
\begin{equation}
b_T=\frac{c_b}{K(T+1)^2}.
\label{eq:constructive-beta2}
\end{equation}
Then all of the following hold simultaneously:
\begin{align}
&0<a_T\le\frac12,\quad 0<b_T<1,
\quad \bar a_T\le\frac18,\quad s\bar a_T\le\frac\refreshprob2,
\label{eq:feasible-unit}\\
&\frac78Ka_T\le\bar a_T=1-(1-a_T)^K\le Ka_T,\label{eq:bar-a-equivalence}\\
&b_T<a_T,\notag\\
&KTb_T=c_b\frac{T}{(T+1)^2}=O(T^{-1}).
\label{eq:feasible-state}
\end{align}
Moreover, $\eta_{g,T}=1$ and $\eta_{\ell,T}=\alpha_T/K$ realize
$\alpha_T=K\eta_{g,T}\eta_{\ell,T}$ and
$K\eta_{\ell,T}=\alpha_T$. Thus the schedule can be realized for every
$T\ge1$.
\end{lemma}
\begin{proof}
Because $u_q\in\{2/5,1/2\}$,
$0<Ka_T\le c_a\le1/16$. Bernoulli's inequality and the geometric-sum identity give
\[
\bar a_T=a_T\sum_{j=0}^{K-1}(1-a_T)^j
\ge Ka_T(1-(K-1)a_T)\ge\frac78Ka_T,
\]
while $\bar a_T\le Ka_T$ follows from the union bound/Bernoulli inequality.
Also
\[
\frac{b_T}{a_T}=\frac{c_b}{c_a}(T+1)^{u_q-2}<1.
\]
Also
$s\bar a_T\le sKa_T\le sc_a\le\refreshprob/8$, which verifies the second small-gain
condition of Lemma~\ref{lem:innovation-influence}. Equation \eqref{eq:feasible-state} and the
stepsize identity follow by direct substitution. Consequently
$\beta_{1,T}=1-a_T$ and, when present, $\beta_{2,T}=1-b_T$ both lie in
$(0,1)$.
\end{proof}

The reference comparison keeps smoothness movement as an explicit residual,
so $c_\alpha=1$ is feasible there. Its restriction on $b_T$ is not a
requirement of Theorem~\ref{thm:FA-GC}(i).

\section{Extensions and Scope}
\label{app:extensions}
The material here concerns optional bias correction, parameter limits, and
alternative momentum conventions. It is separate from the proof path for the
primary algorithm with persistent moments and no moment bias correction.

\subsection{Client-update-count bias correction}
PyTorch's Adam, used in our experiments, divides both moments by
bias-correction factors computed from each client's own update count. The
second-moment factor needs no new argument. The Adam-family proof of
Theorem~\ref{thm:FA-GC} uses the second moment only through
\cref{lem:adam-c23}, which needs every entry of $D$ to lie in
$[\lambda_p,\epsilon^{-1}]$ on $\mathcal H_p$ and in $(0,\epsilon^{-1}]$
otherwise; the proof of that lemma shows that this still holds after the
correction. Theorem~\ref{thm:FA-GC}(i) therefore holds unchanged with the
second-moment correction. (The Adam-family part of Theorem~\ref{thm:FA-NT}
keeps the second moment close to zero and does not cover it.) The
first-moment correction changes the momentum itself, and the next corollary
bounds its effect.
\begin{corollary}[Client-update-count bias-correction estimates]
\label{cor:bias-corrected-momentum}
Assume bounded client population gradients and the fourth-moment conclusion of
\cref{lem:derived-fourth}. Write $a=a_T\in(0,1)$, $h=1-(1-a)^K$, and
let $\kappa_i^t=\sum_{t'<t}I_i^{t'}\sim{\rm Binomial}(t,s)$ be the number of
rounds in which client $i$ trained before round $t$, so that
$\ell_i^t=K\kappa_i^t$ in \eqref{eq:client-age-correction}.
For $k\in\{1,\ldots,K\}$, $\psi_i^{t,k}=1-(1-a)^{K\kappa_i^t+k}$. For each integer $r\ge1$ and $t>0$,
\begin{equation}
\E[(\psi_i^{t,k})^{-r}]
\le(1-e^{-Kast/2})^{-r}+a^{-r}e^{-st/8}.
\label{eq:bias-inverse-explicit}
\end{equation}
Also, for every $t\ge0$,
\begin{equation}
\E|1-\psi_i^{t,k}|^r\le e^{-sht}.
\label{eq:bias-factor-decay}
\end{equation}
For example, set $\Lambda_T=\log(n(T+1))$ and choose the \emph{burn-in time}
\[
t_{\rm b}=\left\lceil\max\left\{\frac{16\Lambda_T}{sh},
\frac8s\bigl(8\log(1/a)+4\Lambda_T\bigr)\right\}\right\rceil.
\]
For $t\ge t_{\rm b}$ the eighth inverse moment is at most $2^8+1$.
If $t\ge t_{\rm b}$ and $M_4^4\ge\sup_{i,t,k}\E\norm{m_i^{t,k}}_2^4$, then
\begin{equation}
\E\norm{(\psi_i^{-1}-1)m_i}^2
\le M_4^2(2^8+1)^{1/4}e^{-sht/4}.
\label{eq:bias-perturbation-explicit}
\end{equation}
For Adam/AMSGrad and the stated SOAP map, correcting only the first moment while keeping the denominator state unchanged gives a squared direction perturbation at most $\epsilon^{-2}$ times this bound.

Under the horizon schedules in this paper, with $K$ and their positive
prefactors fixed, $t_{\rm b}=O(\log(n(T+1))/(sh))$. Let
$D_*^2$ bound the second moment of the aggregated direction when using
$m_i/\psi_i$, let $t_0=\min\{T,t_{\rm b}\}$, and suppose the descent sum
after round $t_0$ uses a constant alignment coefficient $c>0$ and step
$\alpha_T$. Restart that descent sum at round $t_0$. The first $t_0$
stationarity terms and the change in objective up to round $t_0$ contribute
at most
\begin{equation}
\frac{t_0}{T}\left\{G^2+
\frac{GD_*+L\alpha_TD_*^2/2}{c}\right\}.
\label{eq:bias-burnin-transient}
\end{equation}
When $c$ depends on the horizon, its inverse remains in this expression.
The state errors at round $t_0$ are the initial conditions for the
remaining state recursions.
\end{corollary}
\begin{proof}
For any $\lambda>0$, Markov's inequality applied to $e^{-\lambda \kappa_i^t}$ and
$1+x\le e^x$ give
\[
\Pp(\kappa_i^t\le st/2)\le e^{\lambda st/2}\E e^{-\lambda \kappa_i^t}
=e^{\lambda st/2}(1-s+se^{-\lambda})^t
\le e^{st(\lambda/2+e^{-\lambda}-1)}.
\]
With $\lambda=\log2$, the exponent is $-0.153\,st\le-st/8$, so
$\Pp(\kappa_i^t<st/2)\le e^{-st/8}$. On the complement,
$\psi_i^{t,k}\ge1-(1-a)^{K\kappa_i^t}\ge1-e^{-aK\kappa_i^t}\ge1-e^{-Kast/2}$, while on
every outcome $\psi_i^{t,k}\ge a$.
Splitting the expectation proves~\eqref{eq:bias-inverse-explicit}. The chosen
$t_{\rm b}$ makes the first lower bound at least $1/2$ because $h\le Ka$.
It also gives $a^{-8}e^{-st/8}\le1$. These are explicit bounds in $a$, so
the inverse-moment estimate does not require a hidden uniformity assumption.

Since $\kappa_i^t$ is a sum of $t$ independent indicators that equal one with
probability $s$, $\E x^{\kappa_i^t}=(1-s+sx)^t$ for $x\in[0,1]$, which yields
\[
\E(1-a)^{rK\kappa_i^t}
=\{1-s+s(1-a)^{rK}\}^t
\le e^{-st\{1-(1-a)^{rK}\}}
\le e^{-sht}.
\]
The additional factor $(1-a)^{rk}$ is at most one, which proves
\eqref{eq:bias-factor-decay}. Two applications of the Cauchy--Schwarz
inequality $\E[XY]\le(\E X^2)^{1/2}(\E Y^2)^{1/2}$ give
\begin{align*}
\E[(1-\psi_i)^2\psi_i^{-2}\norm{m_i}^2]
&\le(\E\norm{m_i}^4)^{1/2}
\{\E[(1-\psi_i)^4\psi_i^{-4}]\}^{1/2}\\
&\le M_4^2\{\E(1-\psi_i)^8\}^{1/4}
\{\E\psi_i^{-8}\}^{1/4}.
\end{align*}
This proves~\eqref{eq:bias-perturbation-explicit} without requiring momentum
and the client-update count to be independent. The positive scaling operator
has norm at most $\epsilon^{-1}$, giving the direction estimate.

Before $t_{\rm b}$, $m_i/\psi_i$ is a convex average of the input history.
The same input-moment bounds used in Lemma~\ref{lem:adam-expected-movement}
therefore give a finite $D_*$. Since $\norm{\nabla f(x)}\le G$, the first
$t_0$ stationarity terms contribute at most $G^2t_0/T$. Smoothness and
$\E\norm{\bar d^t}\le D_*$ also give
\[
\E f(x^{t_0})-f(x^0)
\le t_0\{\alpha_TGD_*+(L/2)\alpha_T^2D_*^2\}.
\]
Dividing this objective increment by $T\alpha_Tc$ proves
\eqref{eq:bias-burnin-transient}. Thus the descent sum restarted at round $t_0$ must retain
both the earlier stationarity terms and the objective increment, together
with the state errors at round $t_0$.
\end{proof}

\subsection{Single-step and zero-momentum limits}
\begin{remark}[$K=1$ and $\beta_1=0$ algebra]
For $K=1$, $\bar a_T=a_T$ and every round-boundary formula reduces to its
single-step version. If $\beta_1=0$ (so $a_T=1$), then
$m_i^{t,k}=\widehat g_i^{t,k}$ and the bias-correction factor equals one. This is an
algebra check only: $a_T=1$ intentionally lies outside the small-gain
conditions of Theorem~\ref{thm:FA-GC}.
\end{remark}

\subsection{Why client moments persist across rounds}
\label{app:why-persist}
Algorithm~\ref{alg:FAGC} keeps each client's moments across rounds instead of resetting them when a
new round starts. This subsection explains why this choice matters for normalization and Muon. After a
reset, the first local steps follow only a few fresh gradients. For normalized directions, unbiased
noise then no longer guarantees alignment with the gradient (Proposition~\ref{prop:normalized-counterexample}),
and the resulting bound keeps a noise term that does not shrink with $T$
(Corollary~\ref{cor:normalization-residual}). With the scheduled first-moment
coefficient of Theorem~\ref{thm:FA-GC}, persistent moments instead make this
noise term vanish. With a fixed $\beta_1$, both versions keep a noise term of
similar size (compare Corollary~\ref{cor:fixed-gain-main}).
\begin{proposition}[Unbiased noise does not imply normalized alignment]
\label{prop:normalized-counterexample}
In one dimension let $g=0.1$ and let a stochastic gradient equal $-1$ with
probability $0.9$ and $10$ with probability $0.1$. Then
\[
\E[\widehat g]=0.1,\qquad
\E[\operatorname{sign}(\widehat g)]=-0.8,\qquad
g\,\E[\operatorname{sign}(\widehat g)]<0.
\]
The noise has mean zero and bounded variance, as Assumption~\ref{assump:genLoss}
requires. Thus unbiased noise with bounded variance alone cannot give the C2 bound when reset momentum exposes a nonvanishing
fresh-noise fraction. One-dimensional exact Muon has the same obstruction.
\end{proposition}
\begin{corollary}[Reset-momentum neighborhood]
\label{cor:normalization-residual}
Use the \FA\ gradient correction and tracker refresh, but reset
$m_i^{t,0}=0$ every round. For a first-moment gain $a\in(0,1)$, define
\[
h_k:=1-(1-a)^k,\qquad
w_{k,r}:=\frac{a(1-a)^{k-r}}{h_k},\quad 1\le r\le k,
\qquad \vartheta_K:=\frac1K\sum_{k=1}^K\sum_{r=1}^k w_{k,r}^2.
\]
The weights sum to one. For normalization and exact or finite-step Muon,
the per-round alignment error entering~\cref{thm:master} is bounded by
\[
O(\sigma\sqrt{\vartheta_K})
+O\!\left(\sqrt{\Omega^0}(1-\refreshprob)^{t/2}\right)
+O(L_g\sqrt{\E\cD^t+\Xi^t}).
\]
Here $\Omega^0$ is the initial tracker-error energy
in~\eqref{eq:buffer-initialization}, and constants may depend on fixed $K$
and the optimizer's rank factors. Averaging over $T$ rounds bounds the
initialization contribution by $O(\sqrt{\Omega^0}/(T\refreshprob))$.
Consequently, inserting this bound into the descent argument of
Theorem~\ref{thm:FA-GC}(ii), with any steps satisfying $\alpha_T\to0$,
$\delta_T\to0$, and $T\alpha_T\to\infty$, gives
\begin{equation}
\limsup_{T\to\infty}\mathcal G_{1,T}\le C\sigma\sqrt{\vartheta_K},
\label{eq:reset-floor}
\end{equation}
where $C$ depends only on the optimizer's C2--C3 constants. This floor does
not vanish when the first-moment gain is made small: as $a\to0$,
$w_{k,r}\to1/k$, so $\vartheta_K\to K^{-1}\sum_{k=1}^Kk^{-1}\approx(\log K)/K$.
The main theorem instead combines persistent momentum with a scheduled
first-moment gain, which makes the filtered-noise contribution vanish.
\end{corollary}
\begin{proof}
Resetting the momentum gives
$m_i^{t,k}/h_k=\sum_{r=1}^k w_{k,r}\widehat g_i^{t,r}$.
Normalization and both Muon maps are unchanged by multiplication of their
input by a positive scalar, so this normalized average gives the same
direction as $m_i^{t,k}$. Using the population/noise decomposition
in~\eqref{eq:stored-noise}, write its error as
\[
\frac{m_i^{t,k}}{h_k}-g^t
=\sum_{r=1}^k w_{k,r}(g_{i,\rm pop}^{t,r}-g^t)
+u_i^t+\sum_{r=1}^k w_{k,r}\zeta_i^{t,r}.
\]
The first term contains the local movement and tracker lag. Convexity and
Lemma~\ref{lem:fagc-mean} bound its selected local-step average by
$6K L_g^2(\E\cD^t+\Xi^t)$: each input can occur in at most $K$ of these
weighted averages. Distinct fresh noises have zero cross terms by
conditional centering, so the selected local-step average of the last term's
second moment is at most $\sigma^2\vartheta_K$. The middle term retains the stored tracker noise.
Applying the three-term squared-norm inequality and averaging gives
\[
\E\!\left[\frac1{SK}\sum_{i\in\cS^t,k}
\norm{m_i^{t,k}/h_k-g^t}_2^2\right]
\le18K L_g^2(\E\cD^t+\Xi^t)+3\Omega^t+3\sigma^2\vartheta_K.
\]
Lemma~\ref{lem:buffer-energy} gives
$\Omega^t\le(1-\refreshprob)^t\Omega^0+\sigma^2/K$.
Also $w_{1,1}=1$, so the $k=1$ term gives $\vartheta_K\ge1/K$, and the
stationary tracker term $\sigma^2/K$ is covered by the displayed noise term. The normalization
and Muon alignment bounds now give the claimed residual after taking
square roots. Finally,
\[
\frac1T\sum_{t<T}(1-\refreshprob)^{t/2}
\le\frac1{T(1-\sqrt{1-\refreshprob})}
\le\frac2{T\refreshprob},
\]
which proves the averaged initialization bound. For~\eqref{eq:reset-floor},
the C3 bound for normalization and Muon is a constant times the square root of
the displayed error, and
$\sqrt{x+y+z}\le\sqrt x+\sqrt y+\sqrt z$ separates its three parts. Under the
stated steps, the drift and tracker-lag terms tend to zero
(\cref{lem:expected-client-drift,lem:tracking-contraction}), the
initialization term averages to zero, and the objective and smoothness terms of
\cref{thm:master} are $O(1/(T\alpha_T)+\alpha_T)$, which also tends to zero.
\end{proof}

\subsection{Noiseless fixed-gain convergence}
\begin{corollary}[Fixed gains with exact gradients]\label{cor:deterministic-fixed}
Assume the Adam-family setting of Theorem~\ref{thm:FA-GC}(i),
exact gradients, and exact initial trackers. Fix any $a,b\in(0,1)$ and
let $h=1-(1-a)^K$. For a sufficiently small constant
$\alpha=\delta>0$ satisfying $L\alpha\le\epsilon/2$ and
$\mathcal Q_T\le\epsilon^2/4$, the Adam/AMSGrad and stated SOAP iterates obey
$\mathcal G_{2,T}=O(T^{-1})$. Partial participation and selective tracking
remain allowed.
\end{corollary}
\begin{proof}
Exact initialization and gradients give $e_i^t=u_i^t=0$, hence
$\overline{\cN}_{\rm loc}=0$ in Theorem~\ref{thm:lyapunov}(a), which holds for
every $a\in(0,1)$, with finite $C_p=9Ka/h$. Corrected population gradients have norm at most $3G$.
This bounds every second-moment coordinate, including after SOAP rotation,
and gives, deterministically, that every diagonal entry of $D$ is at least $(3G+\epsilon)^{-1}$.
Insert~\eqref{eq:state-ledger-general} into~\eqref{eq:master} exactly as in
the proof of~\eqref{eq:lyap-final}, with this lower bound and zero tail
residual.
The coefficient $\mathcal Q_T$ is proportional to $\alpha^2$ at fixed gains,
so the restrictions admit a positive constant step. The remaining initial terms $4\cP^0/(Tsh)$ and
$2(4C_p+12)L_g^2\Xi^0/(T\refreshprob)$ are $O(T^{-1})$, proving the claim.
\end{proof}

\subsection{Scheduled and practical parameter regimes}
\begin{remark}[First- and second-moment gains]
Theorem~\ref{thm:FA-GC}(i) fixes $\beta_2$ and schedules
$1-\beta_{1,T}$ and the step with the horizon. Thus the second moment may
continue adapting throughout the run. Keeping both gains fixed under noise
instead gives the conservative neighborhoods in~\eqref{eq:fixed-q2}--\eqref{eq:fixed-q1},
including movement floors when local or global steps do not decay.
The noiseless corollary above permits both gains to remain fixed.
Its rate concerns moments without bias correction. The optional
bias-correction estimates in this appendix quantify an additional burn-in
transient and are not included in that rate statement.
\end{remark}

\clearpage

\section{Experimental and Implementation Details}
\label{appen:experimental_part}

\subsection{Algorithmic Details for Federated Adaptive Naive Tracking ({\tt FA-NT})}
\label{appen:naive_tracking_alg_and_rate}

This section presents the analyzed Federated Adaptive Naive Tracking ({\tt FA-NT}) post-estimation correction
variant.  Its local moments persist exactly as in {\tt FA-GC}, but are formed
from the raw stochastic gradients.  The one-time warm start below makes the
initial direction-valued trackers well defined; the paragraph after the
algorithm states which other initializations the analysis allows.
Table~\ref{tab:information-flow-main} compares FA-NT with FA-GC.

\begin{algorithm}[H]
\caption{{\tt FA-NT}: Federated Adaptive Naive Tracking}
\label{alg:FA-NT}
\DontPrintSemicolon
\SetAlgoLined
\KwIn{Rounds $T$, local steps $K$, stepsizes $\eta_\ell,\eta_g$, decays $\beta_1,\beta_2$, initial model $x^0$, adaptive operator $\mathcal P_i^{t,k}(m;v)$}
\tcp{For normalization/Muon, omit $v$ and its update}

\ForEach{$i\in[n]$}{
    \tcp{one-time warm start}
    Draw $g_{i,\rm init}=\nabla\ell_i(x^0;\xi_{i,\rm init})$\;
    $m_i^0=(1-\beta_1)g_{i,\rm init}$, \quad
    $v_i^0=(1-\beta_2)(g_{i,\rm init})^{\odot 2}$\;
    $y_i^0=\mathcal P_i^0(m_i^0;v_i^0)$\;
}
$y^0=n^{-1}\sum_i y_i^0$\;

\For{$t = 0$ \KwTo $T-1$}{
    Randomly sample clients $\mathcal{S}^t \subseteq \{1, \ldots, n\}$\;
    Randomly sample tracking clients $\widetilde{\mathcal{S}}^t \subseteq \mathcal{S}^t$\;
    Server broadcasts $(x^t,y^t)$ to all $i\in\mathcal{S}^t$\;

    \ForEach{$i \in \mathcal{S}^t$}{
        $x_i^{t,1}\gets x^t$, $m_i^{t,0}\gets m_i^t$, $v_i^{t,0}\gets v_i^t$\;

        \For{$k = 1$ \KwTo $K$}{
            $g_i^{t,k}=\nabla\ell_i(x_i^{t,k};\xi_i^{t,k})$\;
            $m_i^{t,k}\gets\beta_1m_i^{t,k-1}+(1-\beta_1)g_i^{t,k}$\;
            $v_i^{t,k}\gets\beta_2v_i^{t,k-1}+(1-\beta_2)(g_i^{t,k})^{\odot 2}$\;
            $p_i^{t,k}\gets\mathcal{P}_i^{t,k}(m_i^{t,k};v_i^{t,k})$\;
            $x_i^{t,k+1} \gets x_i^{t,k} - \eta_\ell\left(p_i^{t,k} + y^t - y_i^t\right)$\;
        }

        $m_i^{t+1}\gets m_i^{t,K}$, \quad $v_i^{t+1}\gets v_i^{t,K}$\;

        \eIf{$i \in \widetilde{\mathcal{S}}^t$}{
            $y_i^{t+1}\gets K^{-1}\sum_{k=1}^K p_i^{t,k}$\;
        }{
            $y_i^{t+1}\gets y_i^t$\;
        }
    }

    Unselected clients retain $(m_i^t,v_i^t,y_i^t)$\;

    Server aggregates $\{x_i^{t,K+1}-x^t\}_{i\in\mathcal{S}^t}$ and $\{y_i^{t+1}-y_i^t\}_{i\in\widetilde{\mathcal{S}}^t}$\;
    $x^{t+1}\gets x^t+\eta_g|\mathcal{S}^t|^{-1}\sum_{i\in\mathcal{S}^t}(x_i^{t,K+1}-x^t)$\;
    $y^{t+1}\gets y^t+n^{-1}\sum_{i\in\widetilde{\mathcal{S}}^t}(y_i^{t+1}-y_i^t)$\;
}
\end{algorithm}

\begin{center}
\begin{minipage}{0.98\linewidth}
\captionof{table}{FA-GC and FA-NT side by side. Both keep each client's moments across rounds and refresh trackers from the gradients computed during local training.}
\label{tab:information-flow-main}
\centering
\begin{tabularx}{\linewidth}{
>{\raggedright\arraybackslash}p{0.20\linewidth}
>{\raggedright\arraybackslash}X
>{\raggedright\arraybackslash}X}
\toprule
 & FA-GC & FA-NT\\
\midrule
Input to the moments & corrected gradient $g_i+y-y_i$ & raw gradient $g_i$\\
Local step direction & $\mathcal P_i(m_i;v_i)$ & $\mathcal P_i(m_i;v_i)+y-y_i$\\
Refreshed tracker & average raw gradient $K^{-1}\sum_k g_i^{t,k}$ & average direction $K^{-1}\sum_k\mathcal P_i(m_i^{t,k};v_i^{t,k})$\\
Initialization & one stochastic gradient per tracker; $m_i=v_i=0$ & one adaptive direction per tracker, computed from one stochastic gradient\\
Client not selected & keeps $m_i,v_i,y_i$ & keeps $m_i,v_i,y_i$\\
Convergence bound & no separate heterogeneity term (Theorem~\ref{thm:FA-GC}) & extra $B^2$ ($q=2$) or $B$ ($q=1$) term (Theorem~\ref{thm:FA-NT})\\
\bottomrule
\end{tabularx}
\end{minipage}
\end{center}

\paragraph{Initialization accounting.}
The displayed all-client FA-GC gradient initialization and FA-NT adaptive warm
start make every tracker immediately available. The proof also permits zero initialization, or initializing a tracker only
when its client is first selected: the resulting finite tracker and moment
discrepancies enter only the stated initialization transients.

Only in the special case $S=\widetilde S=n$ and $\eta_g=1$, {\tt FA-NT} admits the following simplified communication identity. Defining $c_i^t:=y^t-y_i^t$ yields
\begin{align}
    x_i^{t,k+1} &= x_i^{t,k} - \eta_\ell \bigl(p_i^{t,k} + c_i^t\bigr),\\
    c_i^{t+1} &= c_i^t + \frac{1}{K\eta_\ell}\bigl(x_i^{t,K+1} - x^{t+1}\bigr).
\end{align}
Outside this full-training/full-refresh, unit-server-step case, the identity does not hold and the tracker update must be communicated explicitly as in Algorithm~\ref{alg:FA-NT}.

\subsection{SOAP and Muon on QNLI}
\label{appen:qnli-soap-muon}

Figure~\ref{fig:qnli-soap-muon} uses the GPT-2 QNLI setup from the main language benchmark with $n=100$ clients, $S=10$ selected clients, one local epoch per round, 500 communication rounds, and four shared seeds. The client partitions, initializations, and sampling streams are matched across Local, FA-NT, and FA-GC within each optimizer and heterogeneity setting. All optimizer and tracker states persist locally when a client is not selected. Each client fine-tunes rank-8 LoRA adapters with batch size 32, sequence length 128, gradient clipping at 10, and weight decay 0.01.

For SOAP, the client stores its first moment in ambient coordinates, freezes the orthogonal basis throughout a local round, and updates the basis only at round boundaries. The diagonal second moment is maintained in the rotated coordinates and, when the basis changes, re-expressed in the new coordinates. SOAP takes its basis from the eigenvectors of running averages of $GG^\top$ and $G^\top G$ for the matrix gradient $G$; both moment decays and the decay of these averages are 0.95, with $\epsilon=10^{-8}$. FA-GC rotates the corrected gradient before updating this state, whereas FA-NT forms the raw-gradient SOAP direction before adding the tracker correction. The displayed SOAP experiment uses Dirichlet $\alpha=0.1$.

The Muon implementation uses momentum decay 0.95 and five Newton--Schulz steps with coefficients $(3.4445,-4.7750,2.0315)$~\cite{jordan2024muon}. It applies the map to the Nesterov combination $(1-\beta_1)\widehat g+\beta_1m$, where $m$ is the updated persistent momentum. A nonzero matrix input is divided by its Frobenius norm, while a zero input returns zero. The resulting direction is multiplied by $\sqrt{\max\{1,d_1/d_2\}}$ for a $d_1\times d_2$ parameter matrix. Non-matrix parameters use an auxiliary Adam update with decays $(0.9,0.95)$ and $\epsilon=10^{-10}$, and decoupled weight decay is applied to both groups. The finite-step guarantee in~\eqref{eq:ns-iteration} analyzes the polynomial with coefficients $(15/8,-5/4,3/8)$. FA-GC corrects the gradient before persistent momentum and the Nesterov combination, whereas FA-NT adds the tracker after the resulting direction. We report $\alpha\in\{0.1,0.3,1.0\}$. The Local/FA-NT/FA-GC rounds-401--500 means are 72.77/65.61/80.09, 83.87/85.54/86.38, and 85.88/88.01/87.87, respectively.

Every Local, FA-NT, and FA-GC variant receives the same local/server learning-rate search around its optimizer configuration. The selected setting maximizes the area under the validation-accuracy curve (AUC) over rounds 1--100 and is then held fixed for the four reported seeds. Table~\ref{tab:qnli-settings-main} lists the selected pairs. The plotted summaries are held-out test-accuracy means, with one-standard-deviation bands; the values printed below the panels average rounds 401--500, and the insets enlarge the late-round Muon comparisons. These runs use fixed practical gains and are not empirical implementations of the horizon-dependent schedules in Theorem~\ref{thm:FA-GC}.

\begin{table}[htbp]
\centering
\caption{Selected $(\eta_\ell,\eta_g)$ for the QNLI optimizer study.}
\label{tab:qnli-settings-main}
\setlength{\tabcolsep}{3pt}
\begin{tabular}{lccc}
\toprule
Optimizer / $\alpha$ & Local & FA-NT & FA-GC\\
\midrule
SOAP / 0.1 & $(.003,.3)$ & $(.003,.3)$ & $(.001,.3)$\\
Muon / 0.1 & $(.003,.3)$ & $(.001,.9)$ & $(.003,.1)$\\
Muon / 0.3 & $(.003,.3)$ & $(.001,.9)$ & $(.003,.3)$\\
Muon / 1.0 & $(.003,.9)$ & $(.003,.9)$ & $(.003,.3)$\\
\bottomrule
\end{tabular}
\end{table}

\subsection{Additional Muon Results on SST-2}
\label{appen:sst2-muon}

We further evaluate the Muon instances on federated SST-2 with GPT-2 and LoRA. Figure~\ref{fig:sst2-muon} reports the held-out accuracy for Dirichlet parameters $\alpha=0.1$ and $\alpha=0.3$. Under the more heterogeneous $\alpha=0.1$ partition, the rounds-401--500 means are $82.20\%$, $80.65\%$, and $87.70\%$ for Local, FA-NT, and FA-GC, respectively. At $\alpha=0.3$, the corresponding means are $89.09\%$, $89.62\%$, and $90.74\%$. The separation is therefore clearest under stronger heterogeneity, while the three methods become closer as the client distributions become more homogeneous. As in the QNLI study, these are fixed-gain practical runs and do not implement the horizon-dependent schedules in Theorem~\ref{thm:FA-GC}.

\begin{figure}[t]
    \centering
    \includegraphics[width=\textwidth]{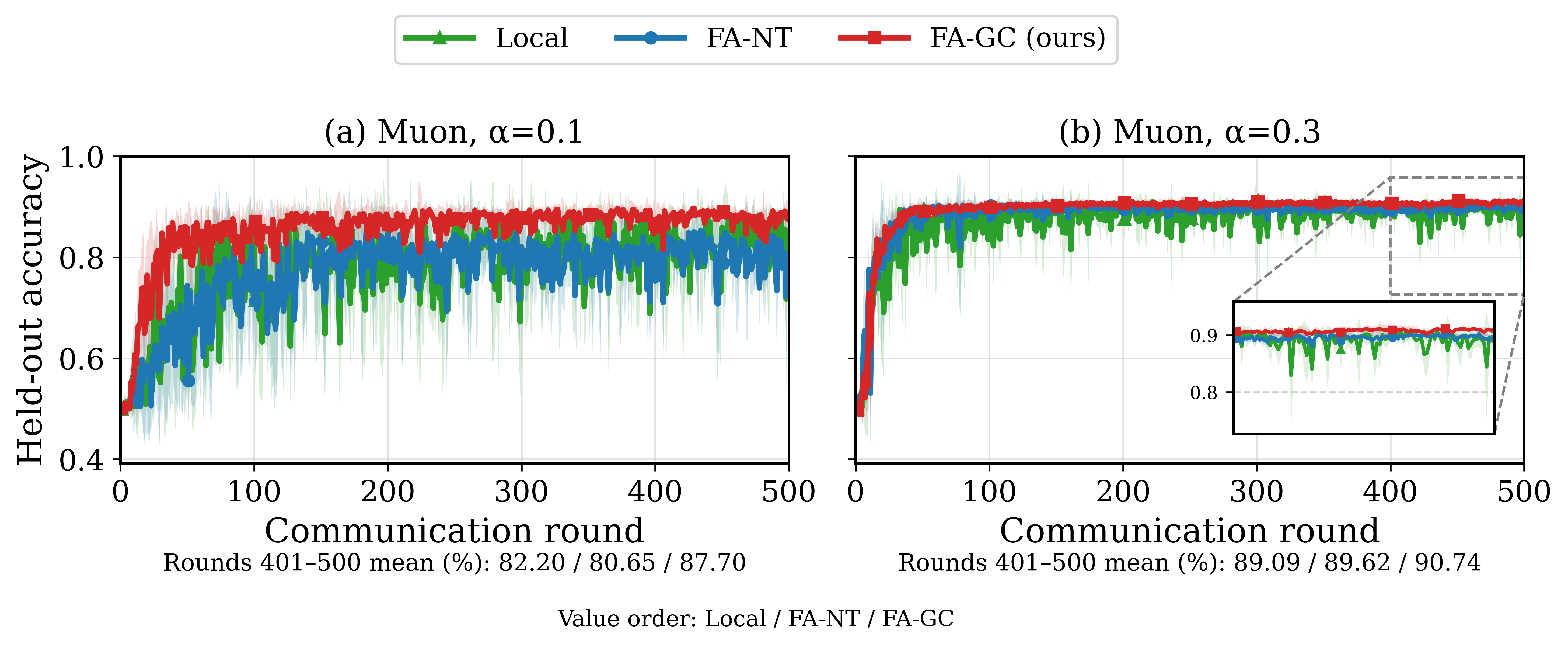}
    \caption{Five-step Newton--Schulz Muon with fixed $\beta_1$ on federated SST-2 with GPT-2 and LoRA at Dirichlet $\alpha\in\{0.1,0.3\}$. All runs use $n=100$, $S=10$, one complete local epoch with batch size 32 and sequence length 128, 500 communication rounds, and four shared seeds. Curves show mean held-out accuracy, and bands show one standard deviation. Values below the panels are rounds-401--500 means in Local/FA-NT/FA-GC order.}
    \label{fig:sst2-muon}
\end{figure}

Table~\ref{table:sst2-muon-lr} gives the selected learning rates. Each method receives the same local/server multiplier grid $\{1/3,1,3\}$, with selection based on validation-accuracy AUC over rounds 1--100. The common configuration uses LoRA rank 8, scaling parameter 16, and dropout 0.05; Muon momentum 0.95; five Newton--Schulz iterations; weight decay 0.01; gradient-clipping norm 10; and auxiliary Adam moments $\beta=(0.9,0.95)$ with $\epsilon=10^{-10}$ for non-matrix parameters. We use a 5\% validation split and the shared seeds 123456, 234567, 345678, and 456789.

\begin{table}[t]
    \centering
    \caption{Selected Learning Rates for the SST-2 Muon Runs}
    \label{table:sst2-muon-lr}
    \footnotesize
    \setlength{\tabcolsep}{8pt}
    \begin{tabular}{c l c c}
        \toprule
        Dirichlet $\alpha$ & Method & Local LR & Server LR \\
        \midrule
        0.1 & Local & $3\times10^{-3}$ & 0.9 \\
        0.1 & FA-NT & $3\times10^{-3}$ & 0.9 \\
        0.1 & FA-GC & $3\times10^{-3}$ & 0.3 \\
        \midrule
        0.3 & Local & $3\times10^{-3}$ & 0.9 \\
        0.3 & FA-NT & $3\times10^{-3}$ & 0.9 \\
        0.3 & FA-GC & $3\times10^{-3}$ & 0.3 \\
        \bottomrule
    \end{tabular}
\end{table}

\subsection{Additional Benchmark Results on CIFAR Datasets}
\label{appen:additional_experiments}

This section provides additional trajectory-level results for the Adam instantiations of {\tt FA-GC} and {\tt FA-NT} under different client sampling rates and local update budgets. These experiments complement the summary metrics reported in the main paper by showing the full convergence behavior across training.

With $K=60$ local steps (Figs.~\ref{fig:addendix_cifar100_k=60} and~\ref{fig:addendix_cifar10_k=60}), {\tt FA-GC} reaches the lowest loss at every sampling rate on both datasets. On CIFAR-100, its accuracy lead over {\tt FA-NT} is largest at the lowest sampling rate (about 0.43 versus 0.37 after 100 rounds at 10\%). With $K=10$ (Figs.~\ref{fig:addendix_cifar100_k=10} and~\ref{fig:addendix_cifar10_k=10}), the two placements are closer. {\tt FA-GC} is slightly ahead at 30\% and 50\% sampling, while at 10\% sampling {\tt FA-NT} reaches lower loss and higher accuracy on both datasets. A likely reason is that fewer local steps cause less client drift per round, so the placement of the correction matters less.

In addition, Fig.~\ref{fig:mean_var} reports the mean and variance of the training curves over four random trials. The reduced variance of {\tt FA-GC} across different heterogeneity levels further highlights the stability induced by pre-estimation correction.

\begin{figure}[h]
    \centering
    \includegraphics[width=0.98\linewidth]{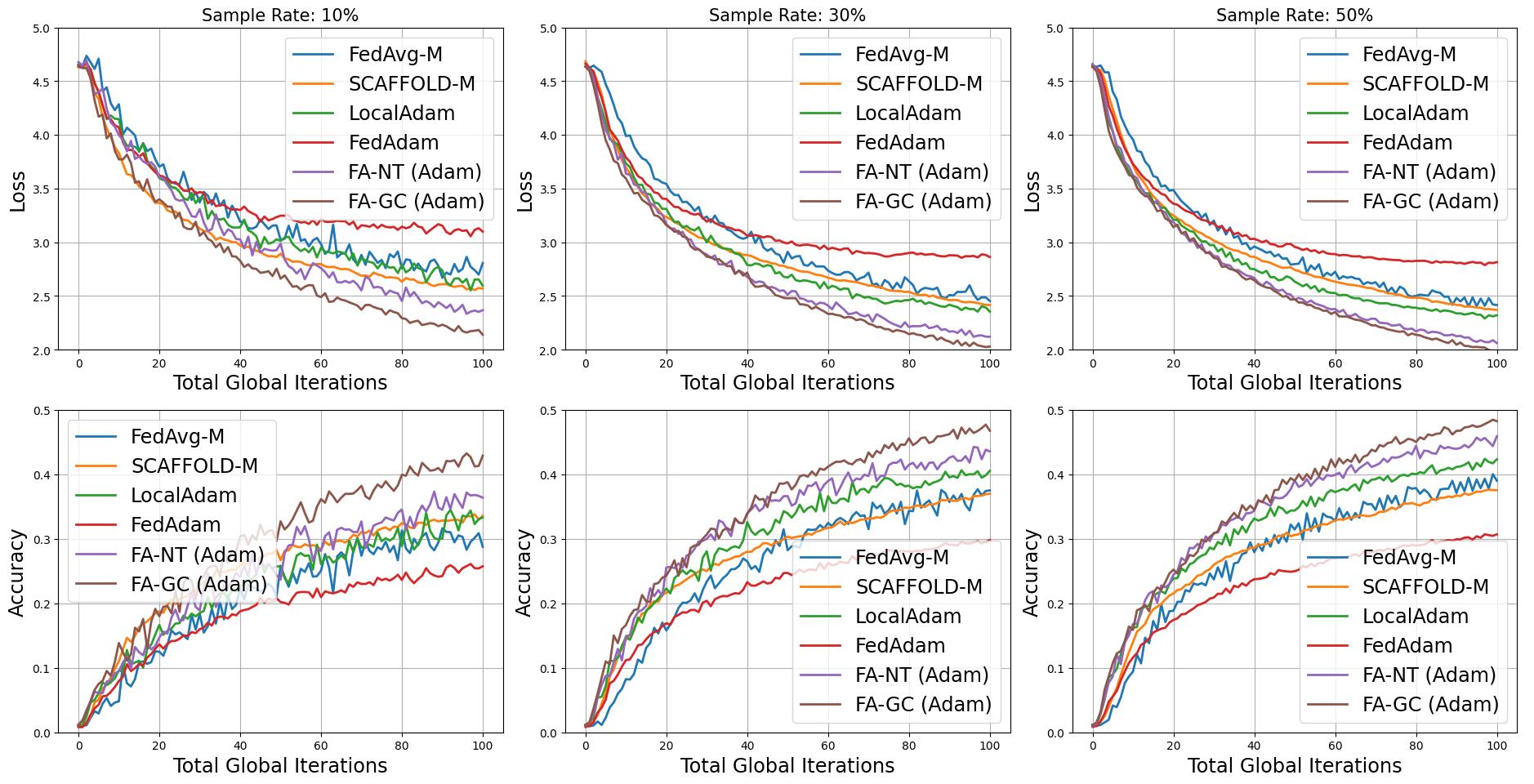}
    \caption{Additional CIFAR-100 results under varying client sampling rates with $K=60$, Dirichlet $\alpha=0.1$, and $\widetilde S=S$ (every participating client refreshes). Each curve is a single run.}
    \label{fig:addendix_cifar100_k=60}
\end{figure}

\begin{figure}[h]
    \centering
    \includegraphics[width=0.98\linewidth]{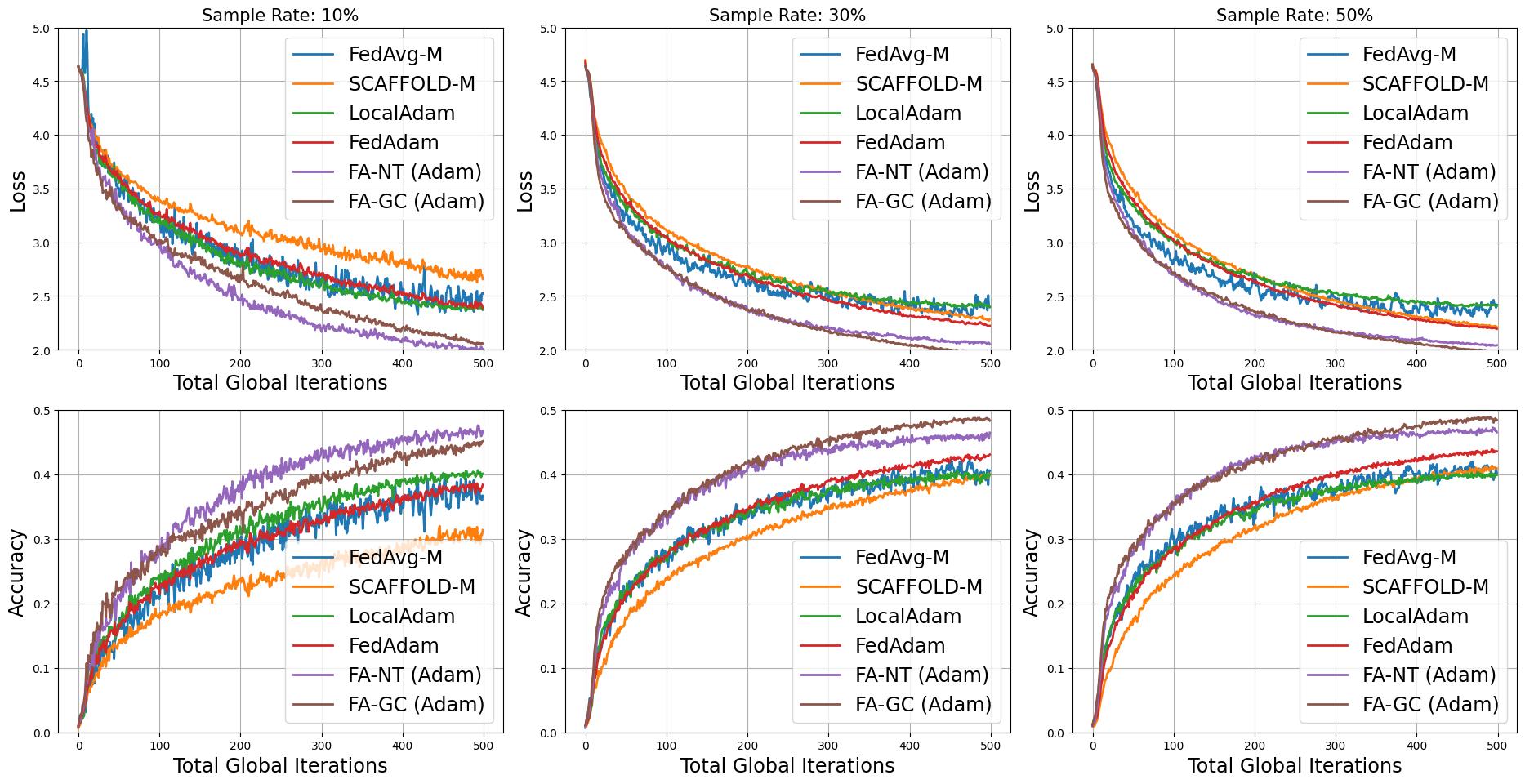}
    \caption{Additional CIFAR-100 results under varying client sampling rates with $K=10$, Dirichlet $\alpha=0.1$, and $\widetilde S=S$ (every participating client refreshes). Each curve is a single run.}
    \label{fig:addendix_cifar100_k=10}
\end{figure}

\begin{figure}[h]
    \centering
    \includegraphics[width=0.98\linewidth]{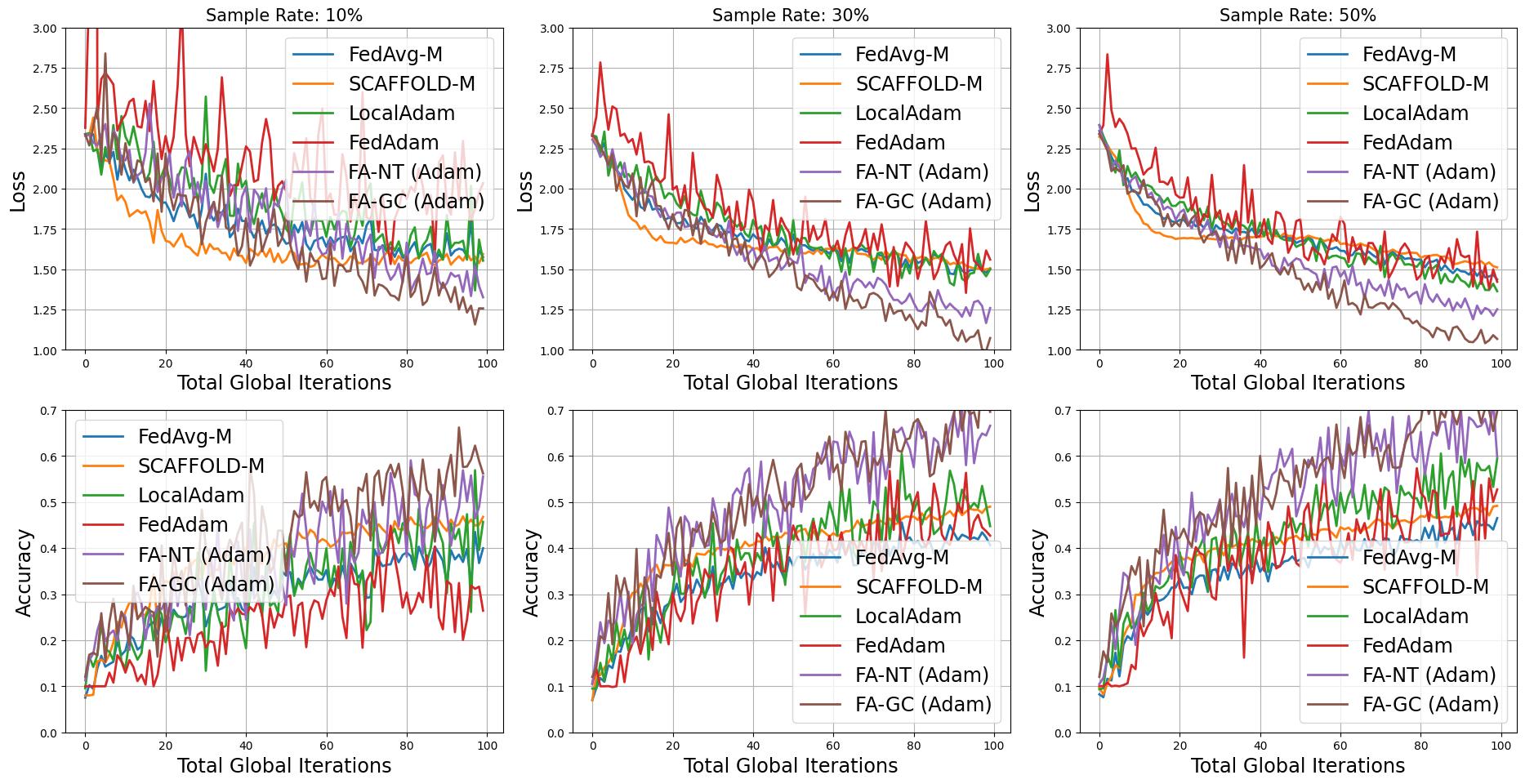}
    \caption{Additional CIFAR-10 results under varying client sampling rates with $K=60$, Dirichlet $\alpha=0.1$, and $\widetilde S=S$ (every participating client refreshes). Each curve is a single run.}
    \label{fig:addendix_cifar10_k=60}
\end{figure}

\begin{figure}[h]
    \centering
    \includegraphics[width=0.98\linewidth]{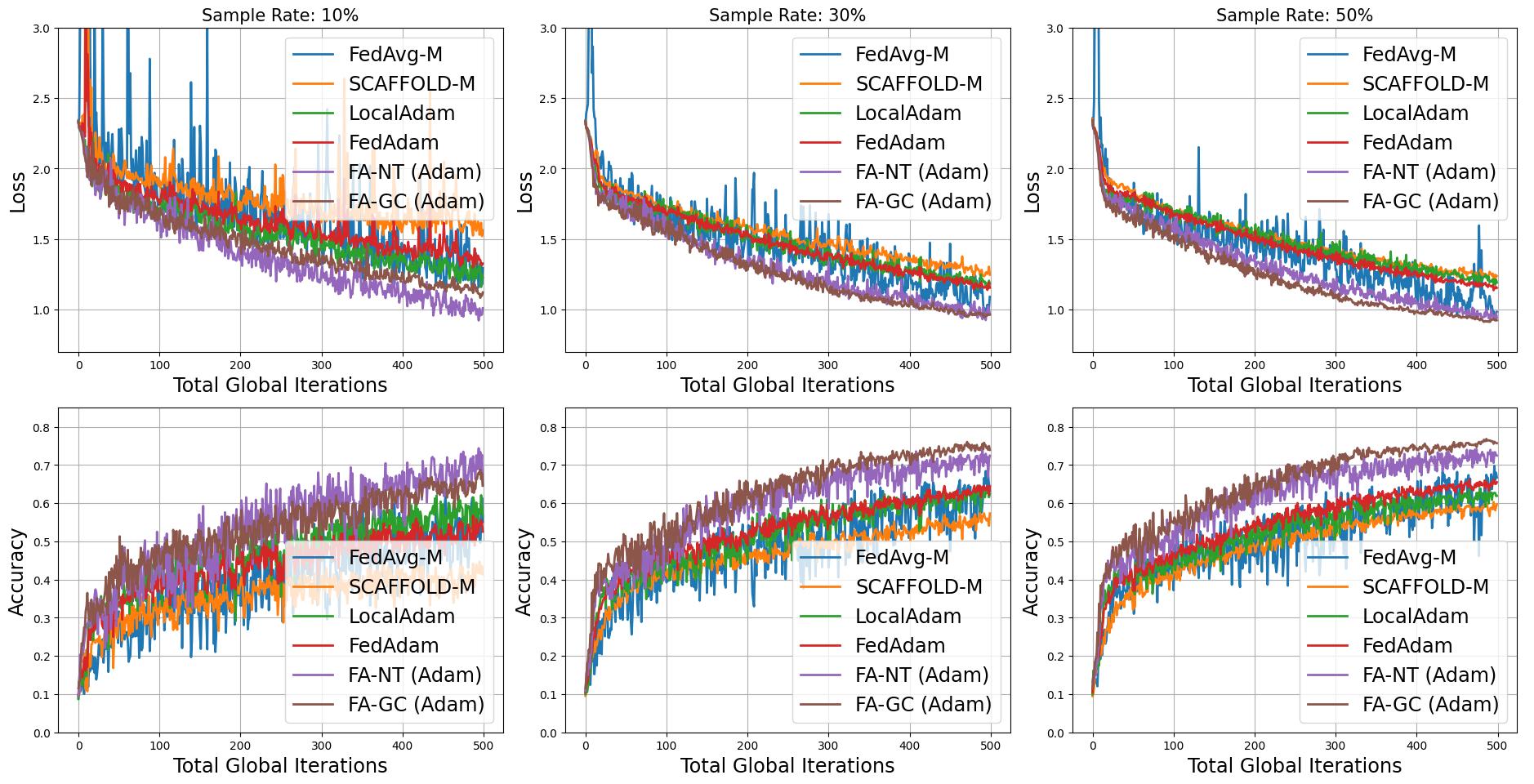}
    \caption{Additional CIFAR-10 results under varying client sampling rates with $K=10$, Dirichlet $\alpha=0.1$, and $\widetilde S=S$ (every participating client refreshes). Each curve is a single run.}
    \label{fig:addendix_cifar10_k=10}
\end{figure}

\begin{figure}[t]
    \centering
    \includegraphics[width=0.8\linewidth]{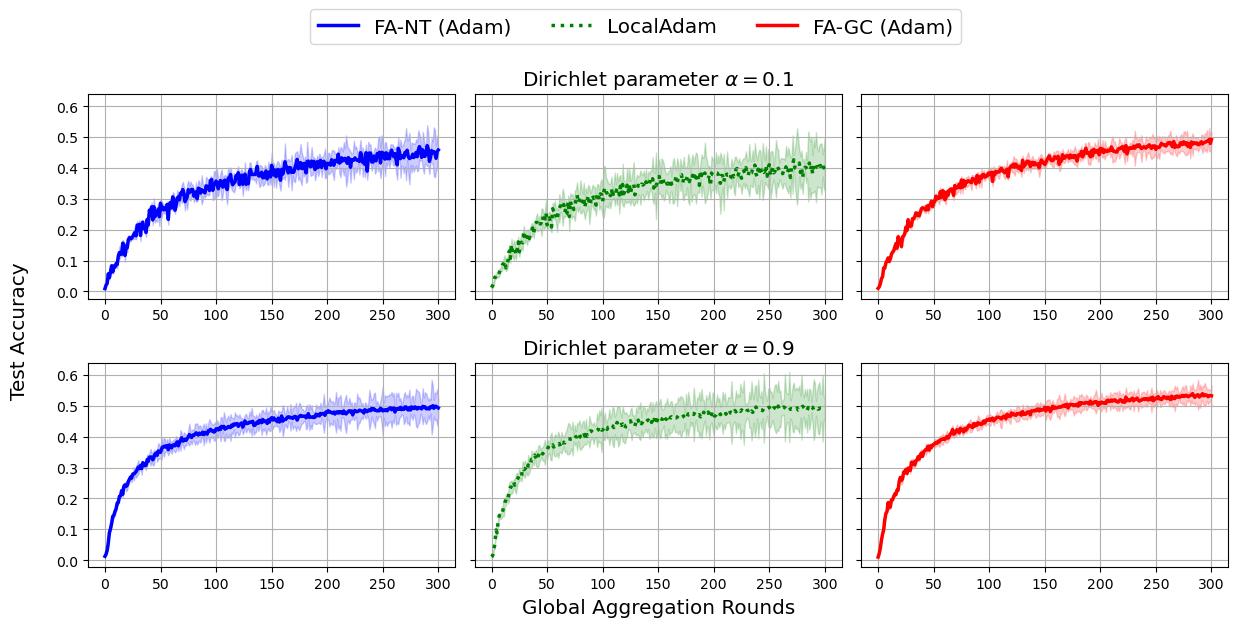}
    \caption{Mean and variance across four random trials under different levels of data heterogeneity.}
    \label{fig:mean_var}
\end{figure}

\subsection{Effect of Second-Moment Estimation}
\label{appen:beta2_comparison}

Although our adaptive-operator view includes both second-moment-based adaptive methods and simpler normalized updates, it is still useful to isolate the empirical role of second-moment estimation within the Adam instantiation. Table~\ref{table:beta_2_comparison} compares {\tt FA-GC} and {\tt FA-NT} under different choices of $\beta_2$.

The results show that larger $\beta_2$ values improve communication-round efficiency across the three image tasks.

\begin{table*}[h!]
\caption{Rounds to target for Adam instantiations under different $\beta_2$ choices (mean and standard deviation over four trials).}
\label{table:beta_2_comparison}
\begin{center}
\resizebox{0.98\textwidth}{!}{
\setlength{\extrarowheight}{3pt}
\begin{small}
\begin{tabular}{cccccccc|>{\columncolor[rgb]{0.85,0.9,0.9}}c>{\columncolor[rgb]{0.85,0.9,0.9}}c}
\toprule
Settings & Dataset & LocalAdam & FA-NT ($\beta_2 = 0$)& FA-GC ($\beta_2 = 0$) & FA-GC ($\beta_2 = 0.3$) & FA-GC ($\beta_2 = 0.7$) & FA-NT & FA-GC\\
\hline
\multirow{3}{*}{\makecell{Communication \\ Rounds}}
& CIFAR-10  & 589.5\com{74.0}& 401.5\com{33.9}& 358.8\com{21.4}& 334.3\com{20.2} & 318.8\com{15.3} & 394.8\com{31.3} & \textbf{310.0}\com{16.8}\\
\hhline{~--------}
& CIFAR-100  & 678.3\com{40.6}& 867.5\com{25.3}& 527.0\com{16.5}& 413.3\com{15.5} & 370.5\com{14.7} & 530.3\com{17.6} & \textbf{323.8}\com{16.3}\\
\hhline{~--------}
& TinyImageNet & 177.3\com{8.3} & 164.3\com{6.7}& 85.5\com{5.7} & 75.3\com{5.7}& 74.0\com{6.6} &157.0\com{6.4} & \textbf{66.3}\com{4.4} \\
\bottomrule
\end{tabular}
\end{small}
}
\end{center}
\end{table*}

\subsection{Chosen Hyperparameters}
\label{appen:lr_and_target_acc}

Tables~\ref{table:lr}-\ref{table:lr_peft} record the selected learning rates
for the main image and language benchmarks. Image tasks use batch size 128
without data augmentation. Learning rates are chosen by grid search over
$\{10^{-5},3\times10^{-5},10^{-4},3\times10^{-4},\dots,1\}$ for each method
and dataset; for 20NewsGroups, the grid also includes $5\times10^{-3}$. A run
ends when it reaches the target test accuracy; the limit of 3000 rounds was
never reached. The common Adam-family setting is
$(\beta_1,\beta_2)=(0.9,0.99)$ with $\epsilon=10^{-8}$ and task-specific
learning rates. FedAvg-M and SCAFFOLD-M both use a momentum coefficient of $0.9$.
The separate tuning protocols for the QNLI SOAP/Muon study
and the SST-2 Muon study are described in Appendices~\ref{appen:qnli-soap-muon}
and~\ref{appen:sst2-muon}, respectively.

\begin{table}[htbp]
\caption{The hyperparameters for image classification tasks. }
\label{table:lr}
\begin{center}
\resizebox{0.98\textwidth}{!}{
\setlength{\extrarowheight}{3pt} 
\begin{small}
\begin{tabular}{cccccccccc}
\toprule
Learning Rate & Dataset & FedAvg-M & SCAFFOLD-M & FedAdam & FedAMS & LocalAdam & PAdaMFed & FA-NT (Adam) & FA-GC (Adam)\\
\hline
 
\multirow{3}{*}{\makecell{$\eta_g$}}

& CIFAR-10 & 1& 1& $1\times10^{-3}$& $1\times10^{-3}$&  1& 1 & 1 & 1\\
\hhline{~---------}
& CIFAR-100 &1 &1 &$1\times10^{-3}$ &$1\times10^{-3}$ & 1 & 1 &  1& 1\\
\hhline{~---------}
& TinyImageNet & $3\times10^{-1}$& $3\times10^{-1}$&$1\times10^{-3}$ & $1\times10^{-3}$& $1\times10^{-1}$ & $1\times10^{-1}$ & $3\times10^{-1}$ & $3\times10^{-1}$\\
\hline

\multirow{3}{*}{\makecell{$\eta_\ell$}}

& CIFAR-10 & $1\times10^{-2}$& $1\times10^{-2}$& $3\times10^{-2}$& $3\times10^{-2}$&  $1\times10^{-3}$& $1\times10^{-3}$ &  $1\times10^{-3}$&$1\times10^{-3}$\\
\hhline{~---------}
& CIFAR-100 & $1\times10^{-2}$& $1\times10^{-2}$& $3\times10^{-2}$& $3\times10^{-2}$&  $1\times10^{-3}$ & $1\times10^{-3}$ & $1\times10^{-3}$ &$1\times10^{-3}$\\
\hhline{~---------}
& TinyImageNet & $1\times10^{-2}$& $1\times10^{-2}$& $1\times10^{-2}$&$1\times10^{-2}$ &$1\times10^{-3}$  & $1\times10^{-3}$ & $1\times10^{-3}$ &$1\times10^{-3}$ \\

\bottomrule
\end{tabular}

\end{small}
} \\
\end{center}
\end{table}

\begin{table}[htbp]
\caption{The hyperparameters for language tasks. }
\label{table:lr_peft}
\begin{center}
\resizebox{0.98\textwidth}{!}{
\setlength{\extrarowheight}{3pt} 
\begin{small}
\begin{tabular}{cccccccccc}
\toprule
Learning Rate & Dataset & FedAvg-M & SCAFFOLD-M & FedAdam & FedAMS & LocalAdam & PAdaMFed & FA-NT (Adam) & FA-GC (Adam)\\
\hline
 
\multirow{4}{*}{\makecell{$\eta_g$}}

& 20NEWSGROUPS & $1\times 10^{-1}$& $1\times 10^{-1}$& $1\times 10^{-2}$&$1\times 10^{-2}$ &$1\times 10^{-1}$ & $1\times10^{-1}$ &$1\times 10^{-1}$ &$1\times 10^{-1}$\\
\hhline{~---------}
& QQP & $1\times 10^{-1}$&$1\times 10^{-1}$ & $1\times 10^{-3}$&$1\times 10^{-3}$ & $3\times 10^{-1}$ & $1\times10^{-1}$ & $3\times 10^{-1}$&$3\times 10^{-1}$\\
\hhline{~---------}
& QNLI & $1\times 10^{-1}$& $1\times 10^{-1}$&$1\times 10^{-3}$ &$1\times 10^{-3}$ & $3\times 10^{-1}$& $1\times10^{-1}$ & $3\times 10^{-1}$&$3\times 10^{-1}$\\
\hhline{~---------}
& SST-2 & $1\times 10^{-1}$& $1\times 10^{-1}$& $1\times 10^{-3}$& $1\times 10^{-3}$& $3\times 10^{-1}$& $1\times10^{-1}$ & $3\times 10^{-1}$&$3\times 10^{-1}$\\
\hline

\multirow{4}{*}{\makecell{$\eta_\ell$}}

& 20NEWSGROUPS & $1\times 10^{-2}$& $1\times 10^{-2}$& $1\times 10^{-3}$&$1\times 10^{-3}$ &$5\times 10^{-3}$ & $5\times 10^{-3}$ &$5\times 10^{-3}$ & $5\times 10^{-3}$\\
\hhline{~---------}
& QQP & $1\times 10^{-2}$&$1\times 10^{-2}$ & $3\times 10^{-4}$& $3\times 10^{-4}$& $1\times 10^{-4}$& $1\times 10^{-4}$ &$1\times 10^{-4}$ &$1\times 10^{-4}$\\
\hhline{~---------}
& QNLI & $1\times 10^{-2}$&$1\times 10^{-2}$ & $3\times 10^{-4}$& $3\times 10^{-4}$& $3\times 10^{-5}$& $3\times 10^{-5}$ & $3\times 10^{-5}$&$3\times 10^{-5}$\\
\hhline{~---------}
& SST-2 & $1\times 10^{-2}$&$1\times 10^{-2}$ &$3\times 10^{-4}$ &$3\times 10^{-4}$ & $1\times 10^{-3}$& $1\times 10^{-3}$ & $1\times 10^{-3}$&$1\times 10^{-3}$\\

\bottomrule
\end{tabular}

\end{small}
} \\
\end{center}
\end{table}
\begin{table}[htbp]
\caption{Model, gradient, optimizer, and tracker storage per client, excluding activations and temporary workspace. Counts are scalars, with approximate FP32 memory in parentheses at 4 bytes per scalar.}
\label{tab:client_storage_overhead}
\begin{center}
\resizebox{0.98\textwidth}{!}{
\setlength{\extrarowheight}{3pt} 

\begin{tabular}{lcccccc}
\toprule
Task & FedAvg-M & SCAFFOLD-M & FedAdam & FedAMS & LocalAdam & FA-GC (Adam) \\
\midrule
Image task (ResNet-18)
& 35.1M (\,$\sim$0.14\,GB)  
& 46.8M (\,$\sim$0.18\,GB)  
& 23.4M (\,$\sim$0.09\,GB)  
& 23.4M (\,$\sim$0.09\,GB)  
& 46.8M (\,$\sim$0.18\,GB)  
& 58.5M (\,$\sim$0.23\,GB)  
\\
Language task
& 125.48M (\,$\sim$0.50\,GB)
& 126.22M (\,$\sim$0.50\,GB)
& 124.74M (\,$\sim$0.49\,GB)
& 124.74M (\,$\sim$0.49\,GB)
& 126.22M (\,$\sim$0.50\,GB)
& 126.96M (\,$\sim$0.51\,GB)
\\
\bottomrule
\end{tabular}
} \\
\end{center}
\end{table}

\begin{table}[htbp]
\centering
\setlength{\tabcolsep}{6pt}
\renewcommand{\arraystretch}{1.2}
\caption{Aggregate training-round communication to reach the target accuracy, in decimal GB. Mean and standard deviation are obtained from the round statistics in Table~\ref{table:baseline} using a common 46.8 MB ResNet-18 payload and the message counts in Appendix~\ref{appen:storage}.}
\label{tab:comm_vol_resnet18}
\resizebox{0.98\textwidth}{!}{

\begin{tabular}{lcccccccc}
\toprule
Dataset & FedAvg-M & SCAFFOLD-M & FedAdam & FedAMS & LocalAdam & PAdaMFed & FA-NT (Adam) & FA-GC (Adam) \\
Vectors/round & 30 & 40 & 20 & 20 & 20 & 60 & 35 & 35 \\
\midrule
CIFAR-10
& 1424.36 \(\pm\) 351.42
& 1018.93 \(\pm\) 112.13
& 2370.42 \(\pm\) 321.33
& 2235.92 \(\pm\) 268.26
& 551.77 \(\pm\) 69.26
& 1011.72 \(\pm\) 58.41
& 646.68 \(\pm\) 51.27
& 507.78 \(\pm\) 27.52
\\

CIFAR-100
& 1401.89 \(\pm\) 124.25
& 1164.01 \(\pm\) 103.71
& 1735.34 \(\pm\) 173.44
& 1548.14 \(\pm\) 146.39
& 634.89 \(\pm\) 38.00
& 1434.89 \(\pm\) 42.68
& 868.63 \(\pm\) 28.83
& 530.38 \(\pm\) 26.70
\\

TinyImageNet
& 302.56 \(\pm\) 14.60
& 453.40 \(\pm\) 28.83
& 508.44 \(\pm\) 42.31
& 433.84 \(\pm\) 33.42
& 165.95 \(\pm\) 7.77
& 245.70 \(\pm\) 17.97
& 257.17 \(\pm\) 10.48
& 108.60 \(\pm\) 7.21
\\
\bottomrule
\end{tabular}
}
\end{table}

\subsection{Communication and Storage Analysis}
\label{appen:storage}

\subsubsection{Storage overhead on clients}

For the Adam instantiation of {\tt FA-GC}, each client stores the local model, first- and second-moment states, and one additional tracking vector. The global tracking term is broadcast by the server and need not be stored persistently on each client. Relative to {\tt LocalAdam}, this introduces only one extra trainable-parameter-sized vector. For GPT-2, the storage totals include the frozen backbone, while only the trainable LoRA adapters are transmitted and require a tracker. Thus the additional tracker occupies about 3~MB even though the total stored state is much larger. Table~\ref{tab:client_storage_overhead} summarizes both workloads, excluding activations and temporary workspace.

\subsubsection{Communication volume metric and results}

We quantify communication cost using the communication-volume metric
\[
\mathrm{CommVol}
=
\left(\sum_{t<T_{\rm target}} \bigl(N_{\mathrm{up}}^t + N_{\mathrm{down}}^t\bigr)\right)\times(\text{bytes per vector}),
\]
where $N_{\mathrm{up}}^t$ and $N_{\mathrm{down}}^t$ count model-sized tensors across all participating client links, and $T_{\rm target}$ is the number of rounds to reach the target accuracy. A broadcast to $S$ clients therefore counts as $S$ downlink transmissions. We use a common FP32 ResNet-18 payload of $46.8\times10^6$ bytes and report decimal GB, with $1\,\mathrm{GB}=10^9$ bytes. The metric counts training-round payloads and excludes one-time initialization, protocol headers, and local computation.

For FedAdam, FedAMS, and LocalAdam, each selected client downloads the model and uploads its model update, giving $2S$ vectors per round. FedAvg-M additionally receives the global momentum, giving $3S$ vectors. For SCAFFOLD-M, we count the global control and momentum contribution as one combined downlink vector, alongside the model, and count separate model and control updates on the uplink, giving $4S$ vectors~\cite{cheng2023momentum}. Under the explicit tracker-update protocol in Algorithms~\ref{alg:FAGC} and~\ref{alg:FA-NT}, FA-GC and FA-NT each download two vectors per training client, upload one model update per training client, and upload one tracker update per refreshed client. Their total is therefore $3S+\widetilde S$ vectors per round. For PAdaMFed~\cite{yan2025problem}, each training client downloads the model, the global control variate, and the global momentum, and uploads updates of all three, giving $6S$ vectors per round in our implementation. If the control variate and the momentum enter the client update only through one combination, as for SCAFFOLD-M, they can be sent as one downlink vector, which gives $5S$. Unlike SCAFFOLD-M, PAdaMFed cannot also drop the momentum upload: its local steps are normalized, so the model update does not determine the clients' momenta. Table~\ref{tab:comm_vol_resnet18} uses $6S$. Under $5S$, PAdaMFed would need $843.10\pm48.67$, $1195.74\pm35.57$, and $204.75\pm14.98$~GB on CIFAR-10, CIFAR-100, and TinyImageNet, and even at the $4S$ of SCAFFOLD-M it would need $674.48\pm38.94$, $956.59\pm28.45$, and $163.80\pm11.98$~GB. FA-GC uses less in every case.

Thus, with $S=10$ and $\widetilde S=5$ for the image benchmark, the respective counts are 20, 30, 40, and 35 vectors per round, and 60 for PAdaMFed. For a method with a constant per-round count $V$, multiplying its reported round mean and standard deviation by $V\,(46.8\times10^6)/10^9$ gives the corresponding communication mean and standard deviation in Table~\ref{tab:comm_vol_resnet18}.

Table~\ref{tab:comm_vol_resnet18} reports the transmitted volume required to
reach the target accuracy. FA-GC has the lowest mean communication volume in
all three image-task comparisons.

\end{document}